\documentclass[twoside,11pt]{article}

\usepackage[T1]{fontenc}
\usepackage[utf8]{inputenc}
\usepackage{amsmath,amssymb,amsthm}
\let\amsproof\proof
\let\endamsproof\endproof
\usepackage[abbrvbib,preprint]{jmlr2e}
\let\proof\amsproof
\let\endproof\endamsproof
\let\cite\citep
\usepackage[showonlyrefs]{mathtools}
\usepackage{enumitem}
\usepackage{booktabs}
\usepackage{float}
\usepackage{placeins}
\usepackage{microtype}
\usepackage{xcolor}
\usepackage{aliascnt}
\usepackage[nameinlink,capitalize]{cleveref}
\usepackage{lastpage}

\makeatletter
\@addtoreset{theorem}{section}
\makeatother

\newaliascnt{proposition}{theorem}
\newtheorem{proposition}[proposition]{Proposition}
\aliascntresetthe{proposition}

\newaliascnt{lemma}{theorem}
\newtheorem{lemma}[lemma]{Lemma}
\aliascntresetthe{lemma}

\newaliascnt{corollary}{theorem}
\newtheorem{corollary}[corollary]{Corollary}
\aliascntresetthe{corollary}
\newaliascnt{assumption}{theorem}

\aliascntresetthe{assumption}

\theoremstyle{definition}
\newtheorem{definition}{Definition}[section]
\newtheorem{remark}{Remark}[section]
\newcommand{\E}{\mathbb E}
\newcommand{\Pp}{\mathbb P}
\newcommand{\R}{\mathbb R}

\newcommand{\one}{\mathbf 1}

\newcommand{\op}{\mathrm{op}}
\newcommand{\HS}{\mathrm{HS}}
\newcommand{\tr}{\operatorname{tr}}

\newcommand{\diag}{\operatorname{diag}}
\newcommand{\Var}{\operatorname{Var}}

\newcommand{\range}{\operatorname{range}}

\newcommand{\argmin}{\operatorname*{arg\,min}}
\newcommand{\argmax}{\operatorname*{arg\,max}}
\newcommand{\Sd}{\mathbb S^{d-1}}

\newcommand{\cC}{\mathcal C}

\newcommand{\cH}{\mathcal H}

\newcommand{\cS}{\mathcal S}

\newcommand{\proofstep}[1]{\par\medskip\noindent\textbf{#1}\par\smallskip}

\jmlrheading{1}{2026}{1--\pageref{LastPage}}{7/26}{}{Manuscript}{Manuel
Fernandez V and Yizhe Zhu}
\ShortHeadings{Spectral Concentration and Recovery}{Fernandez V and Zhu}
\firstpageno{1}

\title{Spectral Concentration and Recovery in Sparse High-Dimensional Random Geometric Graphs}
\author{\name Manuel Fernandez V \email manuelf7@usc.edu\\
\addr Department of Mathematics\\
University of Southern California\\
3620 S. Vermont Ave., KAP 104\\
Los Angeles, CA 90089-2532, USA
\AND
\name Yizhe Zhu \email yizhezhu@usc.edu\\
\addr Department of Mathematics\\
University of Southern California\\
3620 S. Vermont Ave., KAP 104\\
Los Angeles, CA 90089-2532, USA\\
Corresponding author}
\hypersetup{
  hypertexnames=false,
  pdftitle={Spectral Concentration and Recovery in Sparse High-Dimensional Random Geometric Graphs},
  pdfauthor={Manuel Fernandez V and Yizhe Zhu},
  pdfkeywords={random geometric graphs, spectral concentration, latent-space recovery, decoupling, community detection}
}

\begin{document}
\maketitle

\begin{abstract}
We study sparse threshold random geometric graphs generated by
high-dimensional spherical or Gaussian latent vectors.  Although each edge
has marginal probability $p$, shared latent variables make the adjacency
entries dependent.  At the connectivity scale $np=\Omega(\log n)$, the
spherical adjacency matrix satisfies, with high probability,
        $\|A-\E A\|_{\op}
        =O\!\left(\sqrt{np\log n}+np\tau\right)$,
where $\tau$ is the cap threshold; an analogous estimate holds for Gaussian
vectors after controlling radial fluctuations.  This sharpens the spectral
bound in Liu, Mohanty, Schramm, and Yang
\citeyearpar{LiuMohantySchrammYangExpansion} under weaker assumptions.  These
estimates also strengthen the global-synchronization guarantee of
\citet{abdalla2024guarantees} for the homogeneous Kuramoto model.

The leading eigenspace also estimates the latent geometry.  When
$np\gg\log n$, the vector and relative Gram-matrix errors vanish for
$\log(1/p)\ll d\ll np\log(1/p)/\log n$ in the spherical model and for
$\log^2(1/p)\log n\ll d\ll np\log(1/p)/\log n$ in the Gaussian model,
improving the recovery conditions in \citet{LiSchramm}.  For the Gaussian
mixture block model, a polynomial-time semidefinite program gives, to our
knowledge, the first exact-recovery guarantee at the connectivity scale in a
moderate-separation regime.  At much larger separation, fixed edge density
creates isolated vertices and makes exact recovery impossible.  Our reusable
decoupling and matrix concentration framework avoids the trace-moment methods
used in previous work and applies broadly to random graph models with latent
vectors.
\end{abstract}


\section{Introduction}\label{sec:introduction}

Random geometric graphs model networks whose vertices carry latent spatial
or geometric features \cite{Penrose,LiuRaczLatent,DallChristensen}.  They
arise in wireless communication \cite{MaoAnderson}, biological networks
\cite{HighamRasajskiPrzulj}, and latent-space models for social graphs
\cite{HoffRafteryHandcock}.  In each case, edges encode proximity between
unobserved vertex features.  The resulting edge dependence distinguishes
geometric graphs from Erd\H{o}s--R\'enyi graphs and is also the source of the
geometric information one hopes to recover.

We study sparse high-dimensional threshold inner-product graphs.  In the
spherical model, $U_1,\ldots,U_n$ are independent and uniform on
$\mathbb S^{d-1}$, and vertices $i$ and $j$ are adjacent when
$\langle U_i,U_j\rangle\ge\tau$, where
\[
        \Pp\{\langle U_i,U_j\rangle\ge\tau\}=p.
\]
Thus the graph has the same edge density as $G(n,p)$ but not the same joint
law: distinct spherical edge indicators are pairwise independent, while
triangles and longer cycles remain dependent.  We also consider Gaussian
latent vectors, for which fluctuations of a shared endpoint's norm create
additional dependence even among incident edges.

This dependence plays two complementary roles.  It creates clustering through
triangle closing \cite{DallChristensen}, and it allows the graph to retain
information about the latent vectors, as required in statistical network
models based on high-dimensional similarity \cite{RaczBubeck}.  We ask how
large the nontrivial adjacency spectrum can be and when that spectrum contains
enough information to recover the latent geometry.  We work down to the
connectivity scale $np\asymp\log n$, treat both spherical and Gaussian latent
vectors, and then study exact label recovery in the Gaussian mixture block
model.

A further goal is methodological: we seek a proof
strategy that can be reused across random graph models with latent vectors.
Our framework separates model-specific estimates for the population kernel
from a model-independent sampling argument based on Hoeffding decomposition,
decoupling, and matrix concentration.  Once the population operator and
kernel sections are controlled, the same argument yields nonasymptotic
spectral bounds that can then be combined with perturbation methods for
recovery.

\paragraph{Testing and recovering high-dimensional geometry}

The statistical questions begin with detection: can one decide from an
unlabeled graph whether latent geometry is present?  Devroye et
al.~\citeyearpar{DevroyeGyorgyLugosiUdina} proved asymptotic indistinguishability
from $G(n,p)$ when $n$ is fixed and $d\to\infty$.  Bubeck et
al.~\citeyearpar{BubeckDingEldanRacz} identified the sharp dense transition and
initiated the sparse problem, while Liu et
al.~\citeyearpar{LiuMohantySchrammYangTesting} proved a complementary sparse
indistinguishability theorem.  More recently, Du et
al.~\citeyearpar{DuMaoSunWuXu} identified the conjectured threshold
$d\asymp(nh(p))^3$ in the regime $d>(1+\varepsilon)n$.  Related work treats
anisotropic geometries, low-degree tests, sandwiching couplings, and other
geometric models
\cite{BrennanBreslerHuang,BaguleyGobelPappikSchiller,
BangachevBreslerSandwiching,BangachevBresler}.

Recovery asks a finer question: given one graph, can one estimate the latent
Gram matrix?  The rate-distortion lower bound in \citet{MaoZhang} rules out
recovery beyond $d\asymp nh(p)$, where $h(p)$ is the binary entropy and
$nh(p)\asymp np(1+\log(1/p))$ uniformly for $0<p\le p_0$.  For Gaussian
vectors, Li and Schramm~\citeyearpar[Theorem~1.4]{LiSchramm} obtained relative
inner-product recovery averaged over all vertex pairs using a trace-moment
argument.

These considerations lead to the two questions that organize the paper:
\begin{center}
 \textit{How large can the nontrivial adjacency spectrum be, and when does it
 retain enough information to recover the latent geometry?}
\end{center}

The recovery theorems in
\cref{thm:spherical-inner-product,thm:raw-gaussian-inner-product} control
the vector approximation error and the Gram-matrix approximation error.
These errors vanish when
\[
\begin{aligned}
\text{spherical vectors:}\quad
&np\gg\log n,
&&\log(1/p)\ll d\ll \frac{np\log(1/p)}{\log n},\\
\text{Gaussian vectors:}\quad
&np\gg\log n,
&&\log^2(1/p)\log n\ll d\ll
  \frac{np\log(1/p)}{\log n}.
\end{aligned}
\]
The spherical window is nonempty as soon as $np\gg\log n$.  In the Gaussian
model, the same conclusions hold in a wider sparsity and dimension range than
in \citet[Theorem~1.4]{LiSchramm}; the precise comparison appears after
\cref{thm:raw-gaussian-inner-product}.

\paragraph{Spectral concentration and expansion}

Recovery begins with control of the nontrivial adjacency spectrum.  For an
Erd\H{o}s--R\'enyi graph with $np\gtrsim\log n$, the centered adjacency matrix
has operator norm of order $\sqrt{np}$ under standard conditions
\cite{FeigeOfek,LeLevinaVershynin}.  The geometric
analogue must account for dependent cycles as well as the nonconstant
eigenvalues of the population kernel.

We recall the part of the Hoeffding decomposition \cite{Hoeffding} used below.
If $k(x,y)$ is
a centered symmetric kernel and $Y$ is an independent latent vector, its
\emph{first Hoeffding projection} is
\[
        a(x)=\E_Y k(x,Y).
\]
Thus
\[
        k(x,y)=a(x)+a(y)+r(x,y),
        \qquad \E_Y r(x,Y)=0,
\]
and $a(x)$ is precisely the part of the kernel fluctuation explained by one
endpoint alone.  For the spherical threshold kernel, rotational invariance
gives $a\equiv0$.  For the centered Gaussian threshold kernel
$k(x,y)=\one_{\{\langle x,y\rangle/\sqrt d\ge u\}}-p$, where $u$ is chosen to
give marginal edge probability $p$, the conditional law depends on $x$ only
through $\|x\|$.  Thus $a(x)$ is radial and generally nonzero.

For the spherical threshold graph, \cref{thm:main-concentration}
gives, with polynomially high probability,
\[
        \|A-\E A\|_{\op}
        =
        O\left(\sqrt{np\log n}+np\tau\right)
        \qquad\text{when }np\ge C\log n.
\]
The first term is the sparse graph fluctuation; the second is the leading
nonconstant eigenvalue scale of the spherical cap operator.  In the Gaussian
model, double-centering removes the radial first Hoeffding projection and
yields the analogous estimate in \cref{thm:raw-gaussian-concentration}.  Both
bounds hold down to the connectivity scale and impose no additional
$\sqrt{\log n}$ loss on the leading geometric term.

Up to constants, the natural conjectural scale for the nontrivial adjacency
spectrum is
\[
        \max\left\{\sqrt{np},\,n\Lambda\right\},
\]
where $\Lambda$ is the largest absolute value of a nonconstant eigenvalue of
the population kernel.  For spherical caps, $n\Lambda\asymp np\tau$.  Our
bound matches the geometric contribution and loses only a factor
$\sqrt{\log n}$ in the sparse term.  We conjecture that this remaining loss
can be removed.  Standard Kahn--Szemer\'edi arguments \cite{FeigeOfek} and
nonbacktracking matrix methods \cite{benaych2020spectral} rely on independent
edge exposure and do not readily accommodate the geometric dependence around
cycles.

The closest prior sparse spectral result is the expansion theorem of
\citet{LiuMohantySchrammYangExpansion}, whose proof combines the
trace method, random restriction, and random-walk mixing.  On the
unnormalized adjacency scale,
its sparse term is $\sqrt{np}\log^4 n$, and the result requires the additional
hypotheses of that method.  Our theorem instead controls the centered
adjacency matrix directly under $np\ge C\log n$, which is the form needed for
the embedding analysis.  A direct comparison of the hypotheses and bounds is
given after \cref{thm:main-concentration}.

\paragraph{Sparse kernel random matrices}

From a random-matrix viewpoint, our adjacency matrices are sparse kernel
matrices of the form
\[
        A_{ij}=f(\langle X_i,X_j\rangle)\one_{\{i\ne j\}}.
\]
Dense inner-product kernel matrices have been studied through linearization
and limiting spectral laws
\cite{BordenaveEuclidean,ElKaroui,ChengSinger,DoVu,FanMontanari}, as well as
through concentration and proportional-dimensional asymptotics
\cite{LouartLiaoCouillet,AminiRazaee,MeiMisiakiewiczMontanari,LuYau,
DubovaLuMcKennaYau,PanditWangZhu,WangZhu}.  Recent work on smooth geometric
kernels obtains sharp thresholds for detection and latent-vector estimation
\cite{MaoWuXuSmooth}.  Sparse threshold graphs fall outside that setting: the
kernel is discontinuous, its threshold varies with $n,d,p$, and recovery
requires an extremal high-probability estimate rather than a limiting law.

Results for the empirical eigenvalue distribution provide a complementary
bulk perspective.  A semicircle law for sparse high-dimensional spherical
graphs is proved in \citet{CaoZhu}, while fixed-dimensional geometric graph
spectra are studied in
\cite{AdhikariAdlerBobrowskiRosenthal,HamidoucheCottatellucciAvrachenkov}.
Our objective is different: we need a nonasymptotic estimate for the extremal
nontrivial spectrum that is strong enough for recovery.

We obtain this estimate using the decoupling strategy of
\citet{KaushikRombergMuthukumar}, followed by matrix Chernoff applied to a
positive semidefinite Gram matrix.  In contrast with a direct matrix Bernstein
argument, this preserves the geometric covariance scale without an additional
$\sqrt{\log n}$ factor.  The model-independent concentration
step is developed in \cref{sec:matrix-tools}.
Sharp universality and matrix-chaos inequalities provide powerful alternatives
\cite{BrailovskayaVanHandelSharp,BandeiraLuccaNizicNikolacVanHandelChaos}, but
their direct application to discontinuous sparse kernels still requires
additional approximation.

The same distinction between random fluctuation and geometric signal becomes
especially consequential when the latent vectors also encode communities.

\paragraph{Geometric community detection}

In community detection, latent geometry changes both the signal and the
noise.  Related models include Euclidean random graphs with labels, geometric
block models, and geometric perturbations of stochastic block models
\cite{AbbeBaccelliSankararaman,GalhotraMazumdarPalSaha,PechePerchet,
AvrachenkovKumarLeskela}; exact recovery and information flow in geometric
community models are studied in
\cite{EldanMikulincerPieters,GaudioNiuWei,GaudioGuanNiuWei}.  Unlike in the
classical stochastic block model \cite{AbbeSBMSurvey}, edges remain dependent
even after conditioning on all labels.

We consider the Gaussian mixture block model introduced by Li and
Schramm~\citeyearpar{LiSchramm}.  Each vertex has a hidden label
$\xi_i\in\{\pm1\}$ and latent vector
\[
        Z_i=\frac{G_i}{\sqrt d}+\mu\xi_i e_1.
\]
Edges are obtained by thresholding inner products to have marginal density
$p$.  Li and Schramm~\citeyearpar[Theorem~1.6]{LiSchramm} established
almost exact recovery, meaning that the fraction of mislabeled vertices
vanishes, in a moderate-separation regime.  We use a semidefinite program and
a signed-Laplacian dual certificate, following the classical SDP strategy for
the stochastic block model \cite{HajekWuXuSDP}, to recover every label at
$np\ge C\log n$.  This gives, to our knowledge, the first exact-recovery
guarantee for the model.  At a larger separation scale, preserving the
marginal edge density instead creates isolated vertices in both label classes
and makes exact recovery information-theoretically impossible.  The resulting
nonmonotone effect of separation is made precise in
\cref{thm:gmbm-exact,thm:gmbm-large-separation}.

\paragraph{Spectral gap and synchronization}
Beyond recovery, spectral expansion also governs synchronization
on networks.  Abdalla et
al.~\citeyearpar[Theorem~4]{abdalla2024guarantees} use an adjacency spectral bound to
prove global synchronization of the homogeneous Kuramoto model on spherical
threshold graphs.  A subsequent deterministic criterion
\citet[Theorem~1.10]{AbdallaEtAlExpander} reduces global synchronization to
centered adjacency and Laplacian estimates.  Our matrix and degree
concentration bounds verify this criterion at the connectivity scale,
substantially weaken the dimension requirement in the spherical model, and
establish global synchronization in every sufficiently large fixed dimension
when $p$ is fixed.  This partially answers the fixed-dimensional question in
\citet[Section~1.1]{abdalla2024guarantees}.  We also give the corresponding
result for the Gaussian vector graph; see \cref{thm:global-synchronization}.

\section{Main results}\label{sec:main-results}

We state the main results for spherical threshold and Gaussian
vector graphs, followed by synchronization in both models and the Gaussian
mixture block model.
Throughout this section, let $J=\one\one^\top$ and
$\Pi_n=I_n-n^{-1}J$.

\subsection{Spherical threshold graphs}\label{sec:main-spherical}

For an integer $d\ge3$, let
$\sigma=\sigma_d$ be the uniform probability measure on $\Sd$
and let
$U_1,\ldots,U_n$ be independent samples from $\sigma$.  For
$p=p_n\in(0,1/2)$, let $\tau=\tau_{d,p}$ be determined by
\begin{equation}
        \Pp\{\langle U_1,U_2\rangle\ge\tau\}=p.
\label{eq:p-tau-def}
\end{equation}
The spherical threshold random geometric graph has adjacency matrix
\begin{equation}
        A_{ij}=\one_{\{\langle U_i,U_j\rangle\ge\tau\}},
        \qquad i\ne j,
        \qquad A_{ii}=0.
\label{eq:A-def}
\end{equation}
We write $B=A-\E A=A-p(J-I)$ for the centered adjacency matrix.

Our first theorem controls the fluctuation of the adjacency matrix around its
mean.  The geometric contribution enters only through the cap threshold
$\tau$.

\begin{theorem}[Spherical matrix concentration]\label{thm:main-concentration}
For every $D>0$ there are constants $C_D>0$ and $p_0>0$ such that, if
\[
        0<p\le p_0,
        \qquad
        np\ge C_D\log n,
\]
then, for every $d\ge3$, with probability at least $1-n^{-D}$,
\begin{equation}
        \|A-\E A\|_{\op}
        \le
        C_D\left(\sqrt{np\log n}+np\tau\right).
\label{eq:main-ptau}
\end{equation}
\end{theorem}

\paragraph{Comparison and eigenvalue consequence.}
Liu et al.~\citeyearpar[Theorem~1.7]{LiuMohantySchrammYangExpansion} obtain the
unnormalized nontrivial-eigenvalue bound
$\max\{np\tau,\sqrt{np}\log^4 n\}$ for the same spherical graph, assuming
$\tau\ge (np)^{-1/2}$ and
$np=\omega\!\left(d^3\log^4 n\right)$.  In comparison,
\cref{thm:main-concentration} improves the sparse term to
$\sqrt{np\log n}$ and applies for every $d\ge3$.

Let $\lambda_1(A)\ge\cdots\ge\lambda_n(A)$ be the eigenvalues of $A$.
Since $\E A=p(J-I)$ has eigenvalue $-p$ on $\one^\perp$, standard
Courant--Fischer and Weyl bounds \cite[Chapter~III]{BhatiaMatrixAnalysis}
together with \cref{thm:main-concentration} imply, with probability at least
$1-n^{-D}$,
\[
        \max_{2\le k\le n}|\lambda_k(A)|
        \le
        C_D\left(\sqrt{np\log n}+np\tau\right)
\]
under the assumptions of \cref{thm:main-concentration}.

We next define the spectral estimator for the latent vectors and their Gram
matrix.  \cref{thm:spherical-inner-product} controls the
vector approximation error and the Gram-matrix approximation error.

Let $X\in\R^{n\times d}$ be the latent direction matrix whose $i$th row is
$U_i^\top$.
Let $\widehat Y\in\R^{n\times d}$ have orthonormal columns spanning an
eigenspace associated with the $d$ largest eigenvalues of $B=A-\E A$, and set
\[
        \widehat U=\sqrt{n/d}\,\widehat Y,
        \qquad
        \widehat K=\widehat U\widehat U^\top.
\]

\begin{theorem}[Spherical embedding recovery]\label{thm:spherical-inner-product}
For every $D>0$ there are constants $c_D,C_D>0$ and $p_0>0$ such that, if
\[
        0<p\le p_0,\qquad
        np\ge C_D\log n,\qquad
        C_D\log(1/p)\le d\le
        c_D\frac{np\log(1/p)}{\log n},
\]
then, with probability at least $1-n^{-D}$,
\begin{equation}
 \frac1{\sqrt n}\min_{O\in O(d)}\|\widehat UO-X\|_{\mathrm F}
 \le C_D\left(
        \sqrt{\frac{d\log n}{np\log(1/p)}}
        +\sqrt{\frac{\log(1/p)}d}
        \right),
\label{eq:spherical-vector-recovery}
\end{equation}
and
\begin{equation}
 \frac{\sqrt d}{n}\|\widehat K-XX^\top\|_{\mathrm F}
 \le C_D\left(
        \sqrt{\frac{d\log n}{np\log(1/p)}}
        +\sqrt{\frac{\log(1/p)}d}
        \right).
\label{eq:spherical-gram-recovery}
\end{equation}
\end{theorem}

The vector and Gram-matrix approximation errors are both $o(1)$ whenever
\[
        np\gg\log n,\qquad
        \log(1/p)\ll d
        \ll\frac{np\log(1/p)}{\log n}.
\]
This dimension window is nonempty as soon as $np\gg\log n$.  The factor
$\sqrt d/n$ makes the second bound relative to the natural Frobenius scale of
the latent Gram matrix, since
$\|XX^\top\|_{\mathrm F}\asymp n/\sqrt d$ with high probability.

\subsection{Gaussian vector graphs}\label{sec:main-gaussian-vector}

We next consider isotropic Gaussian vectors.  Let
$G_1,\ldots,G_n\stackrel{\mathrm{i.i.d.}}{\sim}N(0,I_d)$ and choose
$u=u_{d,p}$ so that
\[
        \Pp\{\langle G_1,G_2\rangle/\sqrt d\ge u\}=p.
\]
The Gaussian vector inner-product graph has adjacency matrix
\begin{equation}
        A^{\rm G}_{ij}
        =
        \one_{\{\langle G_i,G_j\rangle/\sqrt d\ge u\}},
        \qquad i\ne j,
        \qquad A^{\rm G}_{ii}=0.
\label{eq:raw-gaussian-A-main}
\end{equation}
The next theorem gives a full-matrix estimate and a sharper bound obtained by
double-centering.

\begin{theorem}[Gaussian vector matrix concentration]\label{thm:raw-gaussian-concentration}
For every $D>0$ there are constants $C_D>0$ and $p_0>0$ such that, if
\[
        0<p\le p_0,\qquad
        np\ge C_D\log n,\qquad
        d\ge C_D\log^2(1/p)\log n,
\]
then, with probability at least $1-n^{-D}$,
\begin{align}
        \|\Pi_n(A^{\rm G}-\E A^{\rm G})\Pi_n\|_{\op}
        &\le
        C_D\left(
        \sqrt{np\log n}
        +np\sqrt{\frac{\log(1/p)}d}
        \right),                                      \label{eq:raw-centered-main}\\
        \|A^{\rm G}-\E A^{\rm G}\|_{\op}
        &\le
        C_D\left(
        \sqrt{np\log n}
        +np\frac{\log(1/p)}{\sqrt d}
        \right).                                      \label{eq:raw-full-main}
\end{align}
\end{theorem}

\begin{remark}[Why double-centering appears]
Let $\overline\Phi(t)=\Pp\{N(0,1)\ge t\}$ denote the upper-tail probability of
a standard Gaussian random variable.  For fixed
$x\in\R^d\setminus\{0\}$ and $Y\sim N(0,I_d)$, the conditional edge
probability is
\[
        q(x)
        :=\Pp\{\langle x,Y\rangle/\sqrt d\ge u\}
        =\overline\Phi\left(\frac{u\sqrt d}{\|x\|}\right).
\]
Thus larger Gaussian norms produce larger expected degrees, an effect absent
in the spherical model.  Writing $a(x)=q(x)-p$, the corresponding
one-endpoint term in the sampled matrix is
\[
        a_G\one^\top+\one a_G^\top-2\diag(a_G),
        \qquad a_G=(a(G_1),\ldots,a(G_n))^\top.
\]
Because $\Pi_n\one=0$, double-centering removes the rank-two part and leaves
only a lower-order diagonal correction, controlled in
\cref{lem:raw-first-projection}.  This explains both the sharper estimate
\eqref{eq:raw-centered-main} and the use of the double-centered matrix for
recovery.
\end{remark}

Let $Z\in\R^{n\times d}$ have rows $Z_i^\top=G_i^\top/\sqrt d$.
Let $\widehat Y^{\rm G}\in\R^{n\times d}$ have orthonormal columns spanning
an eigenspace associated with the $d$ largest eigenvalues of
$\Pi_n(A^{\rm G}-\E A^{\rm G})\Pi_n$, and set
\[
        \widehat Z=\sqrt{n/d}\,\widehat Y^{\rm G},
        \qquad
        \widehat K^{\rm G}=\widehat Z\widehat Z^\top.
\]
Although the estimator is formed from a double-centered matrix, the theorem
compares it with the original, uncentered latent matrix $Z$.  The proof
controls the additional centering error through concentration of the Gaussian
sample mean.

\begin{theorem}[Gaussian vector embedding recovery]\label{thm:raw-gaussian-inner-product}
For every $D>0$ there are constants $c_D,C_D>0$ and $p_0>0$ such that, if
\[
        0<p\le p_0,\qquad
        np\ge C_D\log n,
        \qquad
        C_D\log^2(1/p)\log n
        \le d\le
        c_D\frac{np\log(1/p)}{\log n},
\]
then, with probability at least $1-n^{-D}$,
\begin{equation}
 \frac1{\sqrt n}\min_{O\in O(d)}\|\widehat ZO-Z\|_{\mathrm F}
 \le C_D\left(
        \sqrt{\frac{d\log n}{np\log(1/p)}}
        +\frac{\log^{3/2}(1/p)}{\sqrt d}
        \right),
\label{eq:gaussian-vector-recovery}
\end{equation}
and
\begin{equation}
 \frac{\sqrt d}{n}\|\widehat K^{\rm G}-ZZ^\top\|_{\mathrm F}
 \le C_D\left(
        \sqrt{\frac{d\log n}{np\log(1/p)}}
        +\frac{\log^{3/2}(1/p)}{\sqrt d}
        \right).
\label{eq:gaussian-gram-recovery}
\end{equation}
\end{theorem}

The vector and Gram-matrix approximation errors are both $o(1)$ whenever
\[
        np\gg\log n,\qquad
        \log^2(1/p)\log n\ll d
        \ll \frac{np\log(1/p)}{\log n}.
\]
This dimension window is nonempty when
$np\gg\log(1/p)(\log n)^2$.  As in the spherical model, the factor
$\sqrt d/n$ gives the Gram-matrix error relative to its natural Frobenius
scale.

\paragraph{Information limit.}
Mao and Zhang~\citeyearpar[Corollary~2.3]{MaoZhang} show that the expected loss of
every graph-based estimator is bounded away from zero when $d\gtrsim nh(p)$,
where $h(p)$ is the binary entropy.  Since
$h(p)\asymp p\log(1/p)$ for $0<p\le p_0$, the information-theoretic threshold
has order $np\log(1/p)$.  The upper end of our vanishing-error window is
\[
        d\ll\frac{np\log(1/p)}{\log n},
\]
which is within a factor of $\log n$ of this threshold.

\paragraph{Comparison with Li--Schramm.}
For $\mu=0$, the bound in \citet[Theorem~1.4]{LiSchramm} vanishes only when
\[
        \log(1/p)\log^{18}n\ll d
        \ll\frac{np\log(1/p)}{\log^{18}n}.
\]
This window is nonempty only if $np\gg\log^{36}n$.  By comparison,
\cref{thm:raw-gaussian-inner-product} gives a nonempty window when
$np\gg\log(1/p)(\log n)^2$, lowers the dimension requirement to
$d\gg\log^2(1/p)\log n$, and improves the upper range to within a factor
$\log n$ of the information limit.

\paragraph{Uniform recovery.}
It remains open whether either model admits uniform embedding recovery
throughout the same vanishing-error windows, namely whether
\[
 \min_{O\in O(d)}\|\widehat UO-X\|_{2,\infty}=o(1),
 \qquad
 \min_{O\in O(d)}\|\widehat ZO-Z\|_{2,\infty}=o(1),
\]
where $\|M\|_{2,\infty}=\max_i\|M_{i\cdot}\|_2$.  Existing two-to-infinity
bounds \cite{CapeTangPriebe,LeiTwoInfinity} do not reach this scale from
operator-norm control alone, while a direct leave-one-out argument
\cite{AbbeFanWangZhong} encounters dependence between the deleted row and the
corresponding minor eigenspace.

\subsection{Synchronization}\label{sec:main-synchronization}

Synchronization describes the emergence of coherent behavior among coupled
oscillators.  In the homogeneous Kuramoto model
\cite{KuramotoBook,DorflerBulloSurvey}, vertex $i$ carries a phase
$\theta_i\in\R/2\pi\mathbb Z$, and $M_{ij}$ records its coupling to vertex
$j$.  A co-rotating frame removes the common intrinsic frequency.  For an
unweighted graph, $M$ is the adjacency matrix and the flow is
\begin{equation}
        \dot\theta_i
        =\sum_{j=1}^n M_{ij}\sin(\theta_j-\theta_i),
        \qquad i\in[n].
\label{eq:kuramoto-flow}
\end{equation}
The dynamics in \eqref{eq:kuramoto-flow} are the negative gradient flow of
the interaction energy
\begin{align*}
 \mathcal E_M(\theta)
 &=\sum_{1\le i<j\le n}M_{ij}
   \bigl(1-\cos(\theta_i-\theta_j)\bigr).
\end{align*}
Following \citet{abdalla2024guarantees,AbdallaEtAlExpander}, a graph is
\emph{globally synchronizing} if, for Lebesgue-almost every initial condition
on $(\R/2\pi\mathbb Z)^n$, the solution of \eqref{eq:kuramoto-flow} converges
to a synchronized state.  On a connected graph the global minimizers of
$\mathcal E_M$ are synchronized, but connectivity alone does not exclude
locally stable nonsynchronized states, such as twisted states on cycles
\cite{DorflerBulloSurvey}.  The centered-adjacency and Laplacian criterion of
\citet[Theorem~1.10]{AbdallaEtAlExpander} supplies the required global control.
Our matrix and degree concentration bounds verify that criterion and give the
following result.
\begin{theorem}[Global synchronization]\label{thm:global-synchronization}
For every $D>0$ there are constants $C_D>0$ and $p_0>0$ such that, if
\[
        0<p\le p_0,
        \qquad
        np\ge C_D\log n,
\]
then each of the following conclusions holds with probability at least
$1-n^{-D}$:
\begin{enumerate}[label=\textup{(\roman*)},leftmargin=2.2em]
\item the spherical threshold graph is globally synchronizing whenever
      \[
              d\ge C_D\log(1/p);
      \]
\item the Gaussian vector graph is globally synchronizing whenever
      \[
              d\ge C_D\log^2(1/p)\log n.
      \]
\end{enumerate}
\end{theorem}

For comparison, the spherical result in
\citet[Theorem~4]{abdalla2024guarantees} assumes
\[
 np\ge C(\log n)^2,\qquad
 d\ge C\bigl(n^2p^2+(\log n)^4\bigr)(\log n)^4.
\]
Thus part~\textup{(i)} attains the connectivity scale while substantially
relaxing the required lower bound on the ambient dimension.  It also captures,
at the level of the sufficient conditions, the effect of increasing the edge
density anticipated in \citet[Section~1.1]{abdalla2024guarantees}: within the
range $0<p\le p_0$, both $np\ge C_D\log n$ and
$d\ge C_D\log(1/p)$ become easier to satisfy as $p$ increases.  In particular,
for every fixed $p\in(0,p_0]$, part~\textup{(i)} establishes global
synchronization in every sufficiently large fixed dimension.  This partially
answers the fixed-dimensional question posed by
\citet[Section~1.1]{abdalla2024guarantees}.  A complementary result of
\citet{de2026phase} proves synchronization on $\mathbb S^2$ for initial
conditions converging to a smooth profile; as those authors note, that result
does not establish global synchronization.  Part~\textup{(ii)} gives the
corresponding new guarantee for Gaussian latent vectors.

\subsection{Gaussian mixture block model}\label{sec:main-gmbm}

We conclude the main results with an application to community recovery.  The
Gaussian mixture block model of Li and Schramm~\citeyearpar{LiSchramm} is a geometric
analogue of the stochastic block model: each vertex has a hidden binary label,
but the edges are obtained by thresholding latent inner products and remain
dependent even after conditioning on all labels.  The problem therefore
combines a rank-one community signal with dependent geometric fluctuations.

\begin{definition}[Gaussian mixture block model]\label{def:gmbm}
Fix integers $n,d\ge1$, an edge density $p\in(0,1)$, and a separation
parameter $\mu\ge0$.  Let
$\xi_1,\ldots,\xi_n$ be independent Rademacher labels and let
$G_1,\ldots,G_n$ be independent $N(0,I_d)$ vectors, independent of the labels.
Write $\xi=(\xi_1,\ldots,\xi_n)^\top$.
For $1\le i\le n$, set
\[
        Z_i=\frac{G_i}{\sqrt d}+\mu\xi_i e_1\in\R^d.
\]
Let $\tau=\tau_{\mu,d,p}$ be the threshold chosen so that
\[
        \Pp\{\sqrt d\,\langle Z_1,Z_2\rangle\ge \tau\}=p.
\]
The Gaussian mixture block model is the graph with adjacency matrix
\begin{equation}
        A^\mu_{ij}
        =
        \one_{\{\sqrt d\,\langle Z_i,Z_j\rangle\ge \tau\}},
        \qquad i\ne j,
        \qquad
        A^\mu_{ii}=0.
\label{eq:gmbm-A-def}
\end{equation}
\end{definition}

For exact recovery, define
\[
        B_\mu=\Pi_n(A^\mu-p(J-I))\Pi_n.
\]
Consider the semidefinite program
\begin{equation}
 \widehat X\in\argmax\left\{
        \langle B_\mu,X\rangle:
        X\succeq0,\ X_{ii}=1\ \text{for every }i
 \right\}.
\label{eq:gmbm-sdp}
\end{equation}
This is the standard elliptope relaxation and can be solved in polynomial
time to any prescribed accuracy \cite{VandenbergheBoydSDP}.

\begin{theorem}[Gaussian mixture exact recovery]\label{thm:gmbm-exact}
For every $D>0$ there are constants $c_D,C_D>0$ and $p_0>0$ such that, if
\[
        0<p\le p_0,\qquad np\ge C_D\log n,
\]
and
\begin{equation}
        \mu^2d\ge C_D\log n,\qquad
        C_D\sqrt{\frac{\log n}{np}}
        \le \mu^2\sqrt{d\log(1/p)}\le c_D,
\label{eq:GMM}
\end{equation}
then, with probability at least $1-n^{-D}$, the optimizer in
\eqref{eq:gmbm-sdp} is unique and equals
\[
        \widehat X=\xi\xi^\top.
\]
Consequently, a rank-one factor of $\widehat X$ recovers all labels up to a
global sign.
\end{theorem}

Here $np\ge C_D\log n$ is the connectivity-scale condition, while
$\mu^2d\ge C_D\log n$ is the exact-recovery scale even when the Gaussian
mixture vectors are observed directly \cite{NdaoudGMM}.  The upper bound in
\eqref{eq:GMM} restricts the theorem to moderate separation.  The next result
shows why some such restriction is intrinsic when the edge density is fixed.

\begin{theorem}[Large-separation impossibility]
\label{thm:gmbm-large-separation}
There are absolute constants $C>0$ and $p_0>0$ such that, along any parameter
sequence with $n\to\infty$, if
\[
        0<p\le p_0,\qquad np\ge C\log n,\qquad
        \mu^2d\ge C\log n,
\]
and
\[
        \mu\sqrt{\log(1/p)}
        \ge C\left(\sqrt{\log n}+\frac{\log n}{\sqrt d}\right),
\]
then
\[
 \Pp\{A^\mu\text{ has an isolated vertex in each label class}\}
 \longrightarrow1.
\]
Consequently, the supremum over all possibly randomized estimators
$\widetilde\xi=\widetilde\xi(A^\mu)\in\{\pm1\}^n$ satisfies
\[
 \sup_{\widetilde\xi}
 \Pp\{\widetilde\xi\in\{\xi,-\xi\}\}
 \le \frac12+o(1).
\]
\end{theorem}

\paragraph{Interpretation.}
The dependence on $\mu$ is nonmonotone because the threshold is recalibrated
to keep the marginal edge density equal to $p$.  At the scale
\[
        \mu\sqrt{\log(1/p)}
        \gtrsim \sqrt{\log n}+\frac{\log n}{\sqrt d},
\]
extreme first-coordinate fluctuations create isolated vertices in both label
classes.  When $d\ge\log n$, the condition simplifies to
$\mu\gtrsim\sqrt{\log n/\log(1/p)}$.  Thus increasing separation first reveals
the labels and eventually destroys the information carried by extreme
vertices.

\paragraph{Comparison with Li--Schramm.}
Li and Schramm~\citeyearpar[Theorem~1.6]{LiSchramm} assume
\[
        np\gg1,\qquad
        \log^{16}n\ll d<n,\qquad
        d^{-1/2}\ll\mu\le d^{-1/4}\log^{-1/2}n,
\]
and prove that a spectral
algorithm misclassifies a vanishing fraction of vertices.  Their weaker
almost-exact-recovery objective permits $np\gg1$.  By contrast,
\cref{thm:gmbm-exact} achieves exact recovery at the connectivity scale
$np\ge C_D\log n$ and replaces the ambient requirement
$d\gg\log^{16}n$ by the explicit signal conditions in \eqref{eq:GMM}.
Exact recovery by a purely spectral algorithm remains open.
 
\subsection{Proof overview}\label{sec:proof-overview}

The proofs follow three common stages.  First, a population decomposition
identifies the geometric signal and bounds the remaining kernel operator.
Second, decoupling and matrix concentration control the sampled remainder.
Third, the resulting operator estimates are converted into embedding recovery,
synchronization, or label recovery.

\paragraph{Population decomposition}
For the spherical model, the centered edge kernel is
\[
        h(u,v)=\one_{\{\langle u,v\rangle\ge\tau\}}-p.
\]
Rotational invariance makes $h$ canonical: its conditional mean vanishes when
either endpoint is fixed.  The Funk--Hecke formula diagonalizes the associated
integral operator on spherical harmonics.  Its degree-one component is
$d\lambda_1\langle u,v\rangle$, with $\lambda_1\asymp p\tau$; this component
carries the latent embedding.  The full nonconstant operator norm is
$O(p\tau)$ and determines the geometric term in the matrix-concentration
bound.  After removing degree one, the higher-harmonic norm improves to
$O(p\log(1/p)/d)$, which is the population error relevant for recovery.  These
estimates are proved in Appendix~\ref{sec:spherical-harmonics}.

The Gaussian kernel has an additional one-endpoint effect because the
conditional edge probability depends on the norm of the fixed vector.  Its
Hoeffding decomposition \cite{Hoeffding} separates this radial first
projection from a canonical remainder.  Double-centering removes the
resulting rank-two row-and-column term, up to a smaller diagonal correction.
The degree-one Gaussian chaos of the canonical remainder is the bilinear
signal $\beta\langle x,y\rangle$; removing it leaves the higher-order terms
controlled in Appendix~\ref{sec:gaussian-kernel-estimates}.

\paragraph{Decoupled matrix concentration}
Consider a canonical kernel $h$ with population-operator norm $\Lambda$.
Decoupling replaces its dependent symmetric kernel matrix by
\[
        G_{ij}=h(U_i,V_j),
\]
where $V_1,\ldots,V_n$ form an independent copy of the latent sample.
Conditional on $U=(U_1,\ldots,U_n)$, the columns $g_U(V_j)$ are independent,
and
\[
        GG^\top
        =
        \sum_{j=1}^n g_U(V_j)g_U(V_j)^\top
\]
is a sum of independent positive-semidefinite matrices.  Writing $G^0$ for
the zero-diagonal version, the proof works directly with $G^0$ by deleting
the $j$th coordinate from the $j$th column;
this does not increase either the column norm or the conditional covariance
scale.  A Hilbert-space covariance estimate controls its conditional mean
through the squared
population operator, while a kernel-section estimate controls the largest
column.  Matrix Chernoff \cite{Tropp} then gives
\[
        \|G^0\|_{\op}
        \lesssim n\Lambda+\sqrt{np\log n}.
\]
The distinction from a direct rectangular matrix Bernstein argument
\cite{KaushikRombergMuthukumar} is important.  Bernstein places a
$\sqrt{\log n}$ factor on the conditional covariance contribution and yields
$n\Lambda\sqrt{\log n}$.  Passing to the positive matrix $GG^\top$ keeps
$n\Lambda$ in the mean term and confines the logarithmic loss to the column
bound.  A Banach-space decoupling inequality \cite[Chapter~3]{deLaPenaGine}
then transfers the estimate back to the original symmetric matrix.  This route
avoids the trace-moment expansions used in
\cite{LiuMohantySchrammYangExpansion,LiSchramm}.  Its only
model-specific inputs are the removal of any one-endpoint projection and
bounds on the population operator and kernel sections.  The remaining
decoupling and matrix-concentration argument is therefore a common template
for other random graph models with latent vectors.  The core decoupled
argument is developed in
\cref{sec:matrix-tools} and reused for the Gaussian model in
\cref{sec:gaussian-models}.

\paragraph{From concentration to embedding recovery}
In the spherical model, the empirical degree-one component is
$d\lambda_1XX^\top$, where the rows of $X$ are the latent directions.  The
higher harmonics and sparse fluctuation form a perturbation satisfying
\[
        \|E\|_{\op}
        \lesssim \sqrt{np\log n}+\frac{np\log(1/p)}{d},
\]
whereas the signal eigengap is of order
$np\sqrt{\log(1/p)/d}$.  Frobenius-norm versions of the Davis--Kahan and
Procrustes bounds convert their ratio into the vector approximation error; a
deterministic Gram-matrix inequality then gives the Gram-matrix approximation
error.

For Gaussian vectors, double-centering first removes the radial Hoeffding
projection.  After the bilinear signal is separated, the higher-order
canonical remainder contributes
\[
        O\!\left(\sqrt{np\log n}
        +\frac{np\log^2(1/p)}{d}\right).
\]
The corresponding signal eigengap is
again of order $np\sqrt{\log(1/p)/d}$.  A comparison of these two scales,
together with control of the centered sample covariance, gives the Gaussian
recovery theorem.  The two perturbation arguments are carried out in
\cref{sec:embedding-proof,sec:gaussian-models}.

\paragraph{Applications}
Synchronization requires both adjacency and Laplacian control.  For each
fixed vertex, conditioning on its latent vector makes its incident edge
indicators independent.  Applying scalar concentration vertex by vertex and
then taking a union bound controls all degrees.  Combined with adjacency
concentration, this verifies the deterministic expansion criterion of
\citet[Theorem~1.10]{AbdallaEtAlExpander} and proves
\cref{thm:global-synchronization}.

In the Gaussian mixture block model, the difference between the within- and
between-community edge probabilities produces a rank-one population signal
aligned with the labels.  
Matrix and row-sum concentration provide the
estimates needed for a signed-Laplacian SDP certificate, which proves exact
recovery.  At large separation, isolated vertices and exchangeability give
the impossibility result.  Both arguments appear in \cref{sec:gmbm}.

\subsection{Numerical illustrations}
\label{sec:numerical-illustrations}

We include two simulation experiments to illustrate the
finite-sample behavior of the estimators in
\cref{thm:spherical-inner-product,thm:raw-gaussian-inner-product} and the
recovery transition in
\cref{thm:gmbm-exact,thm:gmbm-large-separation}.  

\begin{figure}[t]
    \centering
    \includegraphics[width=\textwidth]{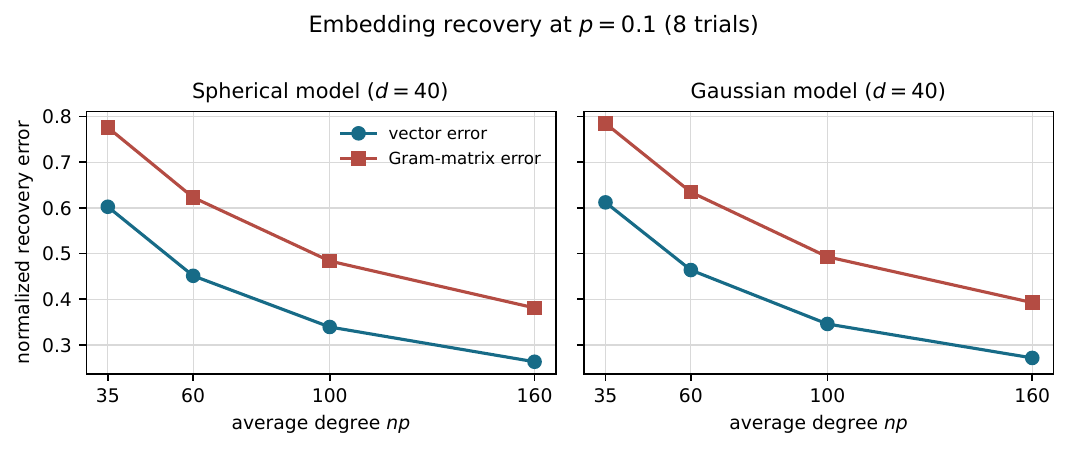}
    \caption{Embedding recovery for the spherical and
    Gaussian vector models at $p=0.1$.  We use $d=40$ in both models and
    $n\in\{350,600,1000,1600\}$.  The two curves are the normalized vector and
    Gram-matrix errors on the left-hand sides of
    \eqref{eq:spherical-vector-recovery}--\eqref{eq:spherical-gram-recovery}
    and \eqref{eq:gaussian-vector-recovery}--\eqref{eq:gaussian-gram-recovery},
    respectively.  Each point is the average over eight independent graphs.}
    \label{fig:embedding-recovery-simulations}
\end{figure}

In \cref{fig:embedding-recovery-simulations}, we apply the spectral
estimators defined before
\cref{thm:spherical-inner-product,thm:raw-gaussian-inner-product}.  Both the
vector approximation error and the relative Gram-matrix approximation error
decrease as the average degree grows.  The Gaussian errors are larger at the
same average degree, consistent with the additional radial fluctuation that
must be removed by double-centering.

\begin{figure}[t]
    \centering
    \includegraphics[width=\textwidth]{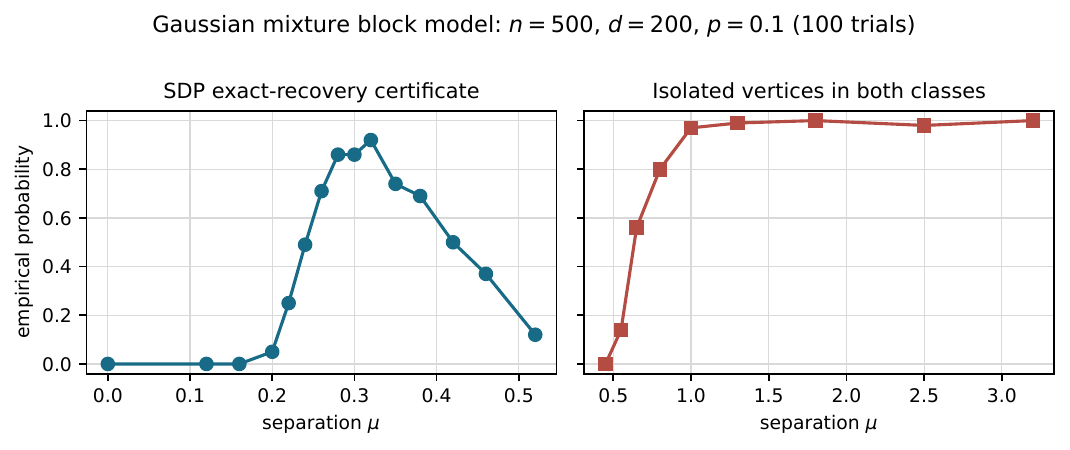}
    \caption{Recovery and its large-separation obstruction in
    the Gaussian mixture block model with $n=500$, $d=200$, and $p=0.1$.
    The left panel reports the frequency with which the signed-Laplacian dual
    certificate is strictly positive on $\one^\perp$, certifying that the SDP
    in \eqref{eq:gmbm-sdp} uniquely recovers all labels.  The right panel
    reports the frequency with which each label class contains an isolated
    vertex.  Each point is the empirical frequency over $100$ independent
    graphs.}
    \label{fig:gmbm-recovery-simulations}
\end{figure}

In the left panel of
\cref{fig:gmbm-recovery-simulations}, we conjugate $B_\mu$ by the true labels
and check the deterministic condition in
\cref{lem:signed-laplacian-certificate}.  The certificate succeeds with high
probability over a moderate range of $\mu$, illustrating
\cref{thm:gmbm-exact}.  As $\mu$ grows further, its success probability
decreases, while the probability of isolated vertices in both classes rises
to one.  This displays the nonmonotone effect of separation behind
\cref{thm:gmbm-large-separation}.

\FloatBarrier

\paragraph{Organization and notation}
The remaining sections contain the complete proofs.
\Cref{sec:matrix-tools,sec:embedding-proof} prove spherical matrix
concentration and embedding recovery.  \Cref{sec:gaussian-models} proves the
corresponding Gaussian vector results, \cref{sec:synchronization-proof} proves
the synchronization theorem, and \cref{sec:gmbm} proves the Gaussian mixture
recovery and impossibility theorems.  The auxiliary estimates for the three
models are collected in Appendices~\ref{sec:spherical-harmonics},
\ref{sec:gaussian-kernel-estimates}, and
\ref{sec:gmbm-auxiliary-estimates}, respectively.  The proof
of the dimension-free covariance estimate is deferred to
Appendix~\ref{sec:hilbert-covariance-proof}.

Throughout the paper, constants may change from line to line.  Constants with
subscript $D$ depend only on the requested polynomial tail exponent $D$, and
absolute constants are denoted by $c,C,c_0,C_0$.  We write
$\|\cdot\|_{\op}$ for the operator norm, $\|\cdot\|_{\HS}$ for the
Hilbert--Schmidt norm, $\|\cdot\|_{\mathrm F}$ for the matrix Frobenius
norm, and
$L^2_0(\sigma)$ for the subspace of mean-zero functions on the sphere.

\section{A decoupling--Chernoff framework for spectral concentration}
\label{sec:matrix-tools}

This section isolates the common matrix-concentration argument
used throughout the paper.  The model-specific analysis consists of removing
any one-endpoint conditional mean and controlling the kernel sections and the
population integral operator.  Once these inputs are available, the proof has
four steps:
\begin{enumerate}
\item replace the symmetric kernel matrix by a rectangular
matrix built from two independent latent samples;
\item control the conditional covariance of its columns by a
dimension-free Hilbert-space estimate;
\item apply matrix Chernoff to the resulting sum of independent
positive-semidefinite rank-one matrices; and
\item decouple back to the original symmetric matrix.
\end{enumerate}
The core estimate is proved in
\cref{prop:decoupled-G}.  The proof of
\cref{thm:main-concentration} then reduces to checking its hypotheses for a
spherical cap.  The same argument is reused for the higher-harmonic and
Gaussian kernels in \cref{sec:embedding-proof,sec:gaussian-models}.

Throughout the proof sections, we use the shorthand
\[L=1+\log(1/p).\]

\subsection{Two concentration inputs}
\label{sec:concentration-tools}

The next two lemmas are the probabilistic inputs used after decoupling.  The
first is a matrix Chernoff inequality for sums of positive-semidefinite
matrices.  The second is a dimension-free covariance estimate for random
vectors in a Hilbert space, tailored to the kernel sections $h(u,\cdot)$.  Its
proof is a one-sided positive-semidefinite specialization of the
dimension-free matrix Laplace method of Hsu et al.~\citeyearpar{HsuKakadeZhang}; see
also Oliveira~\citeyearpar{Oliveira} for closely related Hilbert-space rank-one
covariance concentration.

For clarity, if $\phi\in\cH$, then
$\phi\otimes\phi$ denotes the rank-one operator
$x\mapsto\langle\phi,x\rangle\phi$.  Thus
$\sum_i\phi_i\otimes\phi_i$ is the unnormalized sample covariance operator.
The point of the second lemma is that its bound depends on the population
operator norm and the section bound, but not on the dimension of $\cH$.

\begin{lemma}[Matrix Chernoff, Theorem~5.1 in \citet{Tropp}]\label{lem:matrix-chernoff}
Let $Y_1,\ldots,Y_N$ be independent positive-semidefinite $m\times m$
matrices.  Suppose, for some $b>0$, that
\[
        0\preceq Y_j\preceq b I_m
        \qquad\text{almost surely}
\]
and put
\[
        \mu=\left\|\sum_{j=1}^N\E Y_j\right\|_{\op}.
\]
Then, for every $s\ge0$,
\[
        \Pp\left\{
        \left\|\sum_{j=1}^NY_j\right\|_{\op}
        >
        C\bigl(\mu+b(s+\log m)\bigr)
        \right\}
        \le e^{-s}.
\]
\end{lemma}

\begin{lemma}[Hilbert covariance Chernoff]\label{lem:hilbert-cov}
Let $n\ge2$, and let $\phi_1,\ldots,\phi_n$ be i.i.d.\ random vectors in a
separable Hilbert space $\cH$.  Assume, for some $p>0$, that
\[
        \|\phi_i\|_{\cH}^2\le p
        \qquad\text{almost surely}.
\]
Let
\[
        \cC=\E(\phi_1\otimes\phi_1),
        \qquad
        \kappa=\|\cC\|_{\op}.
\]
Then, for every $D>0$, with probability at least $1-n^{-D}$,
\[
        \left\|
        \sum_{i=1}^n \phi_i\otimes\phi_i
        \right\|_{\op}
        \le
        C_D(n\kappa+p\log n).
\]
\end{lemma}

Its proof is deferred to
Appendix~\ref{sec:hilbert-covariance-proof}; this keeps the main argument
focused on how the concentration inputs fit together.

\subsection{The decoupled rectangular matrix}
\label{sec:decoupled-matrix}

For a symmetric canonical kernel $k$, let $q$ bound the squared
$L^2$ norm of every section $k(u,\cdot)$, and let $\Lambda$ be the operator
norm of its population integral operator.  If $G^0$ denotes the
zero-diagonal decoupled matrix and a typical column has squared Euclidean norm
$O(nq)$, the calculation below gives
\begin{equation}
        \|G^0\|_{\op}
        \lesssim
        \sqrt{nq\log n}+n\Lambda.
\label{eq:abstract-decoupled-output}
\end{equation}
The two terms have different origins: the first is the
fluctuation of an individual sampled column, while the second is the
population covariance contribution.  Thus applying the method to a new
latent-vector model amounts to the following checks: remove any one-endpoint
conditional mean, bound the section norm, bound the population operator, and
verify the column tail.

We now carry out these steps for the spherical cap kernel.
Write
\[
        h(u,v)=\one_{\{\langle u,v\rangle\ge\tau\}}-p,
        \qquad u,v\in\Sd.
\]
Rotational invariance gives
\begin{equation}
        \int_{\Sd}h(u,v)\,d\sigma(v)=0
        \qquad\text{for every }u\in\Sd,
\label{eq:h-canonical-main}
\end{equation}
so $h$ is canonical.  Let $T_h:L^2(\sigma)\to L^2(\sigma)$ be its integral
operator and set
\[
        (T_hf)(u)=\int_{\Sd}h(u,v)f(v)\,d\sigma(v),
        \qquad
        \Lambda=\|T_h\|_{\op}.
\]

We next prove \eqref{eq:abstract-decoupled-output} with $q=p$ for this
kernel.

Let $V_1,\ldots,V_n$ be an independent copy of $U_1,\ldots,U_n$ and define
\begin{equation}
        G_{ij}:=h(U_i,V_j),
        \qquad 1\le i,j\le n.
\label{eq:G-dec-def}
\end{equation}
Because the graph has no self-loops, the matrix required by
decoupling is
\[
        G^0_{ij}:=h(U_i,V_j)\one_{\{i\ne j\}}.
\]
For fixed $U=(U_1,\ldots,U_n)$ and $v\in\Sd$, set
\[
        g_U(v)=\bigl(h(U_1,v),\ldots,h(U_n,v)\bigr)^\top.
\]
Then
\[
        G=\sum_{j=1}^n g_U(V_j)e_j^\top.
\]
If $D_j=I_n-e_je_j^\top$, then the $j$th column of $G^0$
is $D_jg_U(V_j)$.  We work with these columns directly, which avoids a
separate diagonal correction at the end of the proof.
Conditioned on $U$, the columns are independent and centered by
\eqref{eq:h-canonical-main}.

All ``good'' events in the conditional-covariance and
column-bound arguments below depend only on the first
sample $U$.  Once such a realization is fixed, the remaining probability
estimates concern only the independent columns indexed by the $V_j$'s.  This
separation is what allows an ordinary matrix concentration inequality to be
used despite the dependence in the original graph.

\subsubsection{Conditional covariance}
\label{sec:conditional-covariance}

Conditioned on $U$, a full column $g_U(V)$ has
second-moment matrix $\Sigma_U$.  The $j$th zero-diagonal column has
second-moment matrix $D_j\Sigma_UD_j$, so the matrix Chernoff mean will be
controlled once $\|\Sigma_U\|_{\op}$ is bounded.

\begin{samepage}
This matrix is the Gram matrix of the kernel sections,
whereas their covariance operator on $L^2(\sigma)$ is $T_h^2$.  More
precisely, for $u\in\Sd$ write
$\phi_u=h(u,\cdot)\in L^2(\sigma)$.  If $V\sim\sigma$ is
independent of $U$, define
$\Sigma_U:=\E_V[g_U(V)g_U(V)^\top\mid U]$.  Then, for every $i,k\in[n]$,
\begin{align*}
        (\Sigma_U)_{ik}
        &=e_i^\top\E_V\bigl[g_U(V)g_U(V)^\top\mid U\bigr]e_k \\
        &=\E_V\bigl[h(U_i,V)h(U_k,V)\mid U\bigr] \\
        &=\int_{\Sd}h(U_i,v)h(U_k,v)\,d\sigma(v) \\
        &=\langle h(U_i,\cdot),h(U_k,\cdot)\rangle_{L^2(\sigma)}
          =\langle \phi_{U_i},\phi_{U_k}\rangle_{L^2(\sigma)}.
\end{align*}
\end{samepage}
The nonzero eigenvalues of $\Sigma_U$ agree with those of
\begin{equation}
        \cS_U:=\sum_{i=1}^n \phi_{U_i}\otimes\phi_{U_i}.
\label{eq:cap-empirical-section-covariance}
\end{equation}
Moreover,
\begin{equation}
        \|\phi_u\|_2^2
        =
        \int h(u,v)^2\,d\sigma(v)
        =
        p(1-p)\le p,
\label{eq:cap-section-norm}
\end{equation}
and
\[
        \bigl[\E_U(\phi_U\otimes\phi_U)f\bigr](x)
        =\int_{\Sd}h(u,x)\!\left(\int_{\Sd}h(u,y)f(y)\,d\sigma(y)\right)
          d\sigma(u)
        =(T_h^2f)(x),
\]
where the last equality uses the symmetry
$h(u,x)=h(x,u)$.  Hence
\begin{equation}
        \E_U(\phi_U\otimes\phi_U)=T_h^2,
        \qquad
        \|T_h^2\|_{\op}=\Lambda^2.
\label{eq:cap-section-covariance}
\end{equation}
Thus \cref{lem:hilbert-cov} controls $\|\Sigma_U\|_{\op}$.

\begin{lemma}[Good covariance event]\label{lem:good-cov}
For every $D>0$, with probability at least $1-n^{-D-4}$ over
$U_1,\ldots,U_n$,
\[
        \|\Sigma_U\|_{\op}
        \le
        C_D(n\Lambda^2+p\log n).
\]
\end{lemma}

\begin{proof}
Apply \cref{lem:hilbert-cov} with tail exponent $D+4$ to
$\phi_{U_i}=h(U_i,\cdot)$.  The section and population covariance bounds are
\eqref{eq:cap-section-norm} and
\eqref{eq:cap-section-covariance}.  Since $\Sigma_U$ and the operator in
\eqref{eq:cap-empirical-section-covariance} have the same
nonzero eigenvalues, the claimed bound follows from \cref{lem:hilbert-cov}.
\end{proof}

\subsubsection{Uniform column bound}
\label{sec:column-bound}

The second concentration input is an upper tail bound for the column norm.  For
fixed $v$, this norm is controlled by a binomial count of the sample points
falling in the cap centered at $v$.

For fixed $v$, the variables
\[
        \one_{\{\langle U_i,v\rangle\ge\tau\}},
        \qquad i=1,\ldots,n,
\]
are independent Bernoulli$(p)$ random variables.  Since
\[
        h(U_i,v)^2\le
        2\one_{\{\langle U_i,v\rangle\ge\tau\}}+2p^2,
\]
we have
\[
        \|g_U(v)\|_2^2\le 2N_v+2np^2,
\]
where $N_v\sim\mathrm{Bin}(n,p)$ under the randomness of $U$ for fixed $v$.
Thus, when $np\ge C_D\log n$, we set $R=C_D\sqrt{np}$.  For $C_D$
sufficiently large, the standard binomial Chernoff bound
\cite[Section~2.3]{VershyninHDP} gives
\begin{equation}
        \Pp_{U,V}\{\|g_U(V)\|_2>R\}\le n^{-2D-7}.
\label{eq:column-joint-tail}
\end{equation}
For fixed $U$, define
\[
        q_U:=\Pp_V\{\|g_U(V)\|_2>R\}.
\]

\begin{lemma}[Conditional column-tail bound]\label{lem:column-tail}
For every $D>0$ there is a constant $C_D>0$ such that, if
$np\ge C_D\log n$ and $R=C_D\sqrt{np}$ in the definition of $q_U$, then with
probability at least $1-n^{-D-4}$ over $U$,
\[
        q_U\le n^{-D-3}.
\]
\end{lemma}

\begin{proof}
By \eqref{eq:column-joint-tail},
$\E_U q_U\le n^{-2D-7}$.  Markov's inequality gives
\[
        \Pp_U\{q_U>n^{-D-3}\}
        \le n^{D+3}\E_U q_U
        \le n^{-D-4}.
\]
\end{proof}

\subsubsection{Matrix Chernoff}

\begin{proposition}[Decoupled matrix bound]
\label{prop:decoupled-G}
For every $D>0$ there is a constant $C_D>0$ such that, if
$np\ge C_D\log n$, then
\[
        \Pp\left\{
        \|G^0\|_{\op}
        >
        C_D\left(\sqrt{np\log n}+n\Lambda\right)
        \right\}
        \le n^{-D-1}.
\]
\end{proposition}

\begin{proof}
\proofstep{Truncation and the conditional mean.}
Fix a realization of $U$ satisfying the events in
\cref{lem:column-tail,lem:good-cov}, and truncate each column
at the deterministic level
$R=C_D\sqrt{np}$ required by matrix Chernoff.  Namely, set
\[
        \widetilde g_j(V):=
        D_jg_U(V)\one_{\{\|g_U(V)\|_2\le R\}}.
\]
Let $\widetilde G^0$ be the matrix with columns
$\widetilde g_1(V_1),\ldots,\widetilde g_n(V_n)$.  Then
\[
        \widetilde G^0(\widetilde G^0)^\top
        =
        \sum_{j=1}^n
        \widetilde g_j(V_j)\widetilde g_j(V_j)^\top.
\]
The summands
\[
        Y_j:=\widetilde g_j(V_j)\widetilde g_j(V_j)^\top
\]
are conditionally independent positive-semidefinite matrices and satisfy
\[
        0\preceq Y_j\preceq R^2I_n.
\]
Moreover,
\begin{equation}
        \sum_{j=1}^n\E[Y_j\mid U]
        \preceq
        \sum_{j=1}^nD_j\Sigma_UD_j
        =(n-2)\Sigma_U+\diag(\Sigma_U).
\label{eq:zero-diagonal-covariance-sum}
\end{equation}
The identity follows by expanding
$D_j\Sigma_UD_j$ and summing over $j$.  Since $\Sigma_U\succeq0$,
$\diag(\Sigma_U)\preceq\|\Sigma_U\|_{\op}I_n$; hence the last display is
bounded above by $(n-1)\|\Sigma_U\|_{\op}I_n$.
On the good covariance event of \cref{lem:good-cov},
\begin{equation}
        \mu_U:=
        \left\|\sum_{j=1}^n\E[Y_j\mid U]\right\|_{\op}
        \le
        (n-1)\|\Sigma_U\|_{\op}
        \le
        C_D(n^2\Lambda^2+np\log n).
\label{eq:decoupled-mu-bound}
\end{equation}

\proofstep{Conditional matrix Chernoff.}
Applying \cref{lem:matrix-chernoff}, conditionally on $U$, to
$\widetilde G^0(\widetilde G^0)^\top$ with $b=R^2$ and
$s=(D+3)\log n$,
we obtain
\[
        b(s+\log n)=R^2(D+4)\log n\le C_Dnp\log n.
\]
Combining \cref{lem:matrix-chernoff} with
\eqref{eq:decoupled-mu-bound} shows that, with conditional
probability at least $1-n^{-D-3}$,
\[
        \|\widetilde G^0(\widetilde G^0)^\top\|_{\op}
        \le
        C_D(n^2\Lambda^2+np\log n).
\]
Taking square roots gives
\[
        \|\widetilde G^0\|_{\op}
        \le
        C_D\left(n\Lambda+\sqrt{np\log n}\right).
\]

\proofstep{Removal of the truncation.}
Finally, on the event
$\max_{1\le j\le n}\|g_U(V_j)\|_2\le R$, which has conditional probability at
least $1-nq_U\ge1-n^{-D-2}$ by
\cref{lem:column-tail}, the original matrix satisfies
\[
        G^0=\widetilde G^0.
\]
Combining this conditional probability with the events in
\cref{lem:column-tail,lem:good-cov} and then integrating over
$U$ proves the claim.
\end{proof}

\subsection{Decoupling back to the symmetric matrix}
\label{sec:decoupling-symmetric}

Having bounded the independent-column surrogate, we transfer the estimate back
to the original symmetric matrix.  The canonical property
\eqref{eq:h-canonical-main} is the condition that makes the standard order-two
decoupling inequality applicable.

\begin{lemma}[Operator-norm decoupling]\label{lem:decoupling}
Let $U_1,\ldots,U_n$ and $V_1,\ldots,V_n$ be independent samples from a
probability space.  Let $k$ be a symmetric canonical kernel:
\[
        \E[k(u,U)]=0
        \qquad\text{for every fixed }u.
\]
Define
\[
        M_{ij}=k(U_i,U_j)\one_{i\ne j},
        \qquad
        G^0_{ij}=k(U_i,V_j)\one_{i\ne j}.
\]
There is a universal constant $C$ such that, for every $t>0$,
\[
        \Pp\{\|M\|_{\op}>t\}
        \le
        C\,\Pp\{\|G^0\|_{\op}>t/C\}.
\]
\end{lemma}

\begin{proof}
This is the standard tail decoupling inequality for Banach-space-valued
canonical U-statistics of order two; see de la Pe\~na and Gin\'e
\citeyearpar[Chapter~3]{deLaPenaGine}.  We apply that result in the Banach space of
$n\times n$ matrices equipped with the operator norm, to the matrix-valued
kernel whose $(i,j)$ entry is $k(U_i,U_j)\one_{i\ne j}$.  Its
decoupled version is exactly $G^0$.
\end{proof}

\begin{proof}[Proof of \cref{thm:main-concentration}]
Let $V_1,\ldots,V_n$ be the independent copy introduced in
\eqref{eq:G-dec-def}.  Let $C_0\ge1$ be the universal constant in
\cref{lem:decoupling}.  Apply \cref{prop:decoupled-G} with tail parameter
$D+2+\log_2 C_0$.  The resulting failure probability is
$n^{-D-3-\log_2 C_0}$, which is at most
$C_0^{-1}n^{-D-3}$ because $n\ge2$.  After relabeling the constant in the
operator-norm threshold, we therefore obtain
\begin{equation}
        \Pp\left\{
        \|G^0\|_{\op}>
        C_D\left(\sqrt{np\log n}+n\Lambda\right)
        \right\}
        \le C_0^{-1}n^{-D-3}.
\label{eq:cap-zero-diagonal-tail}
\end{equation}

Define
\[
        M_{ij}=h(U_i,U_j)\one_{\{i\ne j\}}.
\]
For $i\ne j$, the definitions of $A$ and $h$ give
$M_{ij}=A_{ij}-p$, while $M_{ii}=0$.  Since
$\E A=p(J-I_n)$, it follows entry by entry that
\begin{equation}
        M=A-\E A.
\label{eq:cap-centered-matrix-identification}
\end{equation}

Set
\[
        t=C_0C_D\left(\sqrt{np\log n}+n\Lambda\right).
\]
Applying \cref{lem:decoupling}, followed by
\eqref{eq:cap-zero-diagonal-tail}, yields
\[
\begin{aligned}
 \Pp\left\{\|A-\E A\|_{\op}>t\right\}
 &=\Pp\left\{\|M\|_{\op}>t\right\}\\
 &\le C_0\,\Pp\left\{\|G^0\|_{\op}>t/C_0\right\}\\
 &\le n^{-D-3}
 \le n^{-D},
\end{aligned}
\]
where the first equality uses
\eqref{eq:cap-centered-matrix-identification}.
Absorbing $C_0$ into $C_D$, we have proved that, with probability at least
$1-n^{-D}$,
\begin{equation}
        \|A-\E A\|_{\op}
        \le
        C_D\left(\sqrt{np\log n}+n\Lambda\right).
\label{eq:cap-preliminary-concentration}
\end{equation}

By \cref{prop:cap-operator}, $\Lambda\le Cp\tau$.  Hence
$n\Lambda\le Cnp\tau$, and substituting this bound into
\eqref{eq:cap-preliminary-concentration} gives
\eqref{eq:main-ptau}.
\end{proof}

\section{Proof of spherical embedding recovery}\label{sec:embedding-proof}

This section proves \cref{thm:spherical-inner-product}.  We separate the
degree-one harmonic, which carries the latent embedding, from the higher
harmonics and bound the empirical higher-harmonic matrix.  The two
deterministic lemmas below then convert the resulting operator-norm error into
vector approximation and, subsequently, Gram-matrix approximation.

\begin{lemma}[Frobenius eigenspace perturbation]
\label{lem:frobenius-perturbation}
Let $S=Y\Gamma Y^\top$ be an $n\times n$ positive semidefinite matrix of
rank $d$, where $Y^\top Y=I_d$ and every eigenvalue of $\Gamma$ is at least
$\gamma>0$.  Let $E$ be symmetric, and let $\widehat Y$ have
orthonormal columns spanning the top $d$ eigenspace of $S+E$.  If
$\|E\|_{\op}\le\gamma/4$, then
\[
        \min_{O\in O(d)}
        \|\widehat YO-Y\|_{\mathrm F}
        \le
        C\sqrt d\,\frac{\|E\|_{\op}}{\gamma}.
\]
\end{lemma}

\begin{proof}
The $d$ nonzero eigenvalues of $S$ are at least $\gamma$, while all other
eigenvalues vanish.  Weyl's inequality therefore leaves a gap of at least
$\gamma/2$ between the perturbed signal cluster and the remaining
eigenvalues.  The Frobenius form of the Davis--Kahan theorem
\cite{DavisKahan,BhatiaMatrixAnalysis} gives
\begin{equation}
        \|\sin\Theta(\widehat Y,Y)\|_{\mathrm F}
        \le
        \frac{2\sqrt d\,\|E\|_{\op}}{\gamma}.
\label{eq:frobenius-davis-kahan}
\end{equation}
For the orthogonal Procrustes alignment \cite[Section~2.2]{CapeTangPriebe},
\begin{equation}
        \min_{O\in O(d)}\|\widehat YO-Y\|_{\mathrm F}
        \le \sqrt2\,
        \|\sin\Theta(\widehat Y,Y)\|_{\mathrm F}.
\label{eq:frobenius-procrustes}
\end{equation}
Combining \eqref{eq:frobenius-davis-kahan} and
\eqref{eq:frobenius-procrustes} proves the claim.
\end{proof}

\begin{lemma}[Gram matrix perturbation]
\label{lem:gram-perturbation}
For $P,Q\in\R^{n\times d}$ and $O\in O(d)$,
\[
        \|PP^\top-QQ^\top\|_{\mathrm F}
        \le
        \bigl(\|P\|_{\op}+\|Q\|_{\op}\bigr)
        \|PO-Q\|_{\mathrm F}.
\]
\end{lemma}

\begin{proof}
Since $(PO)(PO)^\top=PP^\top$, write
\[
 PP^\top-QQ^\top
 =(PO-Q)(PO)^\top+Q(PO-Q)^\top.
\]
The claim follows from
$\|MN^\top\|_{\mathrm F}\le\|M\|_{\mathrm F}\|N\|_{\op}$.
\end{proof}

With these deterministic perturbation tools in place, the only remaining
probabilistic task is to control the empirical higher-harmonic matrix.  The
next proposition repeats the decoupling argument with the degree-one signal
removed.

Fix $u\in\Sd$, let $V\sim\sigma$, and define the degree-one
Funk--Hecke eigenvalue by
\begin{equation}
        \lambda_1
        :=
        \E_V\!\left[h(u,V)\langle u,V\rangle\right]
        =
        \E_V\!\left[
        \one_{\{\langle u,V\rangle\ge\tau\}}\langle u,V\rangle
        \right].
\label{eq:lambda1-main-definition}
\end{equation}
The value in \eqref{eq:lambda1-main-definition} does not depend on $u$ by
rotational invariance; the second equality uses
$\E_V\langle u,V\rangle=0$.  The degree-one reproducing kernel is
$K_1(u,v)=d\langle u,v\rangle$, so
\cref{lem:first-harmonic} identifies the degree-one component of $h$ as
\[
        h_1(u,v)=\lambda_1K_1(u,v)
        =d\lambda_1\langle u,v\rangle.
\]
We therefore define
\[
        h_\perp(u,v)
        =
        \one_{\{\langle u,v\rangle\ge\tau\}}-p
        -d\lambda_1\langle u,v\rangle
\]
to be the spherical threshold kernel with its degree-one harmonic removed.
Let $T_{h_\perp}$ be its integral operator and set
\[
        \Lambda_{\ge2}=\|T_{h_\perp}\|_{\op}.
\]

\begin{proposition}[Higher-harmonic concentration]
\label{prop:higher-harmonic}
Under the assumptions of \cref{thm:spherical-inner-product}, with probability
at least $1-n^{-D-2}$,
\[
        \left\|
        \bigl(h_\perp(U_i,U_j)\one_{\{i\ne j\}}\bigr)_{i,j=1}^n
        \right\|_{\op}
        \le C_D\left(\sqrt{np\log n}+n\Lambda_{\ge2}\right).
\]
\end{proposition}

\begin{proof}
\proofstep{Step 1: Section norm and conditional covariance.}
The kernel $h_\perp$ is symmetric and canonical, and its integral-operator
norm is $\Lambda_{\ge2}$.  Its sections satisfy
\[
\begin{aligned}
 \int h_\perp(u,v)^2\,d\sigma(v)
 &\le
 2\int h(u,v)^2\,d\sigma(v)
 +2d^2\lambda_1^2\int\langle u,v\rangle^2\,d\sigma(v)\\
 &\le 2p+2d\lambda_1^2\le Cp.
\end{aligned}
\]
Here $d\lambda_1^2\le Cp^2L\le Cp$, where
$L=1+\log(1/p)$, by \cref{lem:first-harmonic}.

We now verify the two inputs to the decoupled argument.  Let
$V_1,\ldots,V_n$ be an independent spherical sample.  Applying
\cref{lem:hilbert-cov} to the sections
$\phi_u=h_\perp(u,\cdot)$ gives, with probability at least
$1-n^{-D-4}$ over $U=(U_1,\ldots,U_n)$,
\begin{equation}
 \left\|\sum_{i=1}^n\phi_{U_i}\otimes\phi_{U_i}\right\|_{\op}
 \le C_D\bigl(n\Lambda_{\ge2}^2+p\log n\bigr).
\label{eq:higher-harmonic-covariance-event}
\end{equation}
Here the covariance operator is $T_{h_\perp}^2$, whose norm is
$\Lambda_{\ge2}^2$.

\proofstep{Step 2: Uniform column bound.}
For the uniform column bound, fix $v\in\mathbb S^{d-1}$.  Then
\[
 \sum_{i=1}^n h_\perp(U_i,v)^2
 \le
 2\sum_{i=1}^n h(U_i,v)^2
 +2d^2\lambda_1^2\sum_{i=1}^n\langle U_i,v\rangle^2.
\]
For fixed $v$, the first sum is at most a constant times a
$\operatorname{Bin}(n,p)$ variable plus $np^2$.  Hence the binomial Chernoff
bound \cite[Section~2.3]{VershyninHDP} gives
\begin{equation}
 \Pp_{U,V}\left\{
   \sum_{i=1}^nh(U_i,V)^2>C_Dnp
 \right\}\le n^{-2D-10}.
\label{eq:higher-harmonic-column-tail}
\end{equation}
If $q_U$ denotes the conditional probability of this event given $U$,
Markov's inequality applied to
\eqref{eq:higher-harmonic-column-tail} shows that
$q_U\le n^{-D-5}$ outside an event of
probability $n^{-D-5}$.  Independently, the standard sample-covariance bound
for the isotropic vectors $\sqrt d\,U_i$
\cite[Theorem~4.6.1]{VershyninHDP} gives
\begin{equation}
        \left\|\sum_{i=1}^nU_iU_i^\top\right\|_{\op}
        \le C\frac nd
\label{eq:higher-harmonic-sample-covariance}
\end{equation}
with probability at least $1-n^{-D-5}$.  This uses $d\le c_Dn$, which
follows from the upper bound on $d$ in
\cref{thm:spherical-inner-product}.  On
\eqref{eq:higher-harmonic-sample-covariance}, uniformly in $v$,
\[
 2d^2\lambda_1^2\sum_{i=1}^n\langle U_i,v\rangle^2
 \le Cnd\lambda_1^2
 \le Cnp;
\]
the last inequality follows from $d\lambda_1^2\le Cp^2L$ and $pL\le C$.
Consequently, conditional on a realization of $U$ satisfying
\eqref{eq:higher-harmonic-sample-covariance} and
$q_U\le n^{-D-5}$ as obtained from
\eqref{eq:higher-harmonic-column-tail},
all $n$ decoupled columns have squared norm at most $C_Dnp$ except on an
event of conditional probability at most $nq_U\le n^{-D-4}$.

\proofstep{Step 3: Chernoff concentration and decoupling.}
We now repeat the positive-semidefinite Chernoff calculation
in \cref{prop:decoupled-G}, directly for the zero-diagonal columns.  The
conditional mean is controlled by
\eqref{eq:higher-harmonic-covariance-event}, and Step~2 gives the squared
column bound $C_Dnp$.  Consequently, the zero-diagonal decoupled matrix has
norm at most
\begin{equation}
        C_D\left(\sqrt{np\log n}+n\Lambda_{\ge2}\right)
\label{eq:higher-harmonic-decoupled-bound}
\end{equation}
with the required conditional probability.  Applying the
operator-norm decoupling inequality in \cref{lem:decoupling} to
\eqref{eq:higher-harmonic-decoupled-bound} gives the claimed symmetric-matrix
bound.
\end{proof}

\begin{proof}[Proof of \cref{thm:spherical-inner-product}]
\proofstep{Step 1: Sample covariance and signal decomposition.}
Put $L=1+\log(1/p)$ and
\[
        M=\frac dnX^\top X.
\]
The standard covariance bound for independent isotropic subgaussian vectors
\cite[Theorem~4.6.1]{VershyninHDP} gives, with probability at least
$1-n^{-D-2}$,
\begin{equation}
        \|M-I_d\|_{\op}
        \le
        \delta_n:=C_D\sqrt{\frac{d+\log n}{n}}.
\label{eq:spherical-frob-covariance}
\end{equation}
The upper bound on $d$ in
\cref{thm:spherical-inner-product} implies $d\le c_Dn$, so
the constants can
be chosen such that $\delta_n\le1/2$.
On this event $M$ is positive definite, so $X$ has full column rank.  We may
therefore define
\[
        Y=X(X^\top X)^{-1/2}
        =\sqrt{\frac dn}\,XM^{-1/2}.
\]

Let
\[
        E_{ij}=h_\perp(U_i,U_j)\one_{\{i\ne j\}}.
\]
The degree-one decomposition is the exact identity
\begin{equation}
        B+d\lambda_1I_n=d\lambda_1XX^\top+E.
\label{eq:spherical-frob-decomposition}
\end{equation}
By \cref{prop:higher-harmonic,lem:higher-harmonics,lem:first-harmonic},
\[
 \|E\|_{\op}
 \le C_D\left(\sqrt{np\log n}+np\frac Ld\right),
 \qquad
 \lambda_1\asymp p\sqrt{\frac Ld}.
\]
The nonzero eigenvalues of $d\lambda_1XX^\top$ are the eigenvalues of
$n\lambda_1M$ and hence are at least $n\lambda_1/2$.  Moreover,
\begin{equation}
 \rho_n:=\frac{\|E\|_{\op}}{n\lambda_1}
 \le C_D\left(
        \sqrt{\frac{d\log n}{npL}}+\sqrt{\frac Ld}
 \right).
\label{eq:spherical-frob-ratio}
\end{equation}

\proofstep{Step 2: Eigenspace perturbation.}
The lower and upper bounds on $d$ make the two terms on the right smaller
than a sufficiently small absolute constant.  Adding
$d\lambda_1I_n$ shifts every eigenvalue by the same amount, so it does not
change the top $d$ eigenspace.  We may therefore apply
\cref{lem:frobenius-perturbation} to
\eqref{eq:spherical-frob-decomposition}.  It yields an $O\in O(d)$ such
that
\begin{equation}
        \|\widehat YO-Y\|_{\mathrm F}
        \le C_D\sqrt d\,\rho_n.
\label{eq:spherical-eigenspace-bound}
\end{equation}
After multiplying by $\sqrt{n/d}$ and dividing by $\sqrt n$,
\[
 \frac1{\sqrt n}
 \left\|\widehat UO-XM^{-1/2}\right\|_{\mathrm F}
 \le C_D\rho_n.
\]
Since $\|X\|_{\mathrm F}=\sqrt n$ and
\eqref{eq:spherical-frob-covariance} implies
$\|M^{-1/2}-I_d\|_{\op}\le C\delta_n$,
\[
 \frac1{\sqrt n}
 \|XM^{-1/2}-X\|_{\mathrm F}
 \le C\delta_n.
\]
Under the assumptions of
\cref{thm:spherical-inner-product}, this covariance error is
already contained
in the sparse perturbation term.  Indeed, $pL\le C$ and $d\ge C_DL$ give
\[
        \frac{d+\log n}{n}
        \le C_D\frac{d\log n}{npL}.
\]
The triangle inequality and \eqref{eq:spherical-frob-ratio} therefore prove
the vector bound \eqref{eq:spherical-vector-recovery}.

\proofstep{Step 3: Gram matrix recovery.}
On the covariance event
\eqref{eq:spherical-frob-covariance},
\[
        \|\widehat U\|_{\op}=\sqrt{\frac nd},
        \qquad
        \|X\|_{\op}\le C\sqrt{\frac nd}.
\]
Applying \cref{lem:gram-perturbation} with the alignment in
\eqref{eq:spherical-eigenspace-bound} gives
\[
 \frac{\sqrt d}{n}\|\widehat K-XX^\top\|_{\mathrm F}
 \le \frac C{\sqrt n}\|\widehat UO-X\|_{\mathrm F}.
\]
The vector bound \eqref{eq:spherical-vector-recovery},
together with a union bound over the covariance event
\eqref{eq:spherical-frob-covariance} and the event in
\cref{prop:higher-harmonic} proves
\eqref{eq:spherical-gram-recovery} and completes the proof.
\end{proof}

\section{Proofs for the Gaussian vector model}\label{sec:gaussian-models}

This section proves
\cref{thm:raw-gaussian-concentration,thm:raw-gaussian-inner-product}.  The
Gaussian Hoeffding decomposition \cite{Hoeffding} separates a radial
one-endpoint projection, the bilinear signal
$\beta\langle x,y\rangle$, and a canonical remainder.  Double-centering
removes the rank-two radial term, decoupling--Chernoff controls the remainder,
and the Frobenius perturbation argument from \cref{sec:embedding-proof}
recovers the bilinear signal.  Throughout this section, set
$L=1+\log(1/p)$.

Let $G,G'\stackrel{\mathrm{i.i.d.}}{\sim}N(0,I_d)$ and set
\[
        S=\frac{\langle G,G'\rangle}{\sqrt d}.
\]
Write $\gamma_d$ for the standard Gaussian probability measure on $\R^d$.
Let $u=u_{d,p}$ be the upper $p$-quantile of $S$.  For
$x,y\in\R^d$, put
\[
        q_{\rm G}(x,y)
        =
        \one_{\{\langle x,y\rangle/\sqrt d\ge u\}}-p.
\]
The first Hoeffding projection is
\begin{equation}
        a(x)
        =
        \E_Y q_{\rm G}(x,Y)
        =
        \overline\Phi\left(\frac{u\sqrt d}{\|x\|}\right)-p,
\label{eq:raw-a-def}
\end{equation}
where $Y\sim N(0,I_d)$ and $\overline\Phi(t)=\Pp\{N(0,1)\ge t\}$.
Thus $a(x)$ records the part of the centered edge indicator that can be
predicted from one endpoint alone.  Subtracting this contribution from both
endpoints gives the canonical part:
\begin{equation}
        r(x,y)=q_{\rm G}(x,y)-a(x)-a(y).
\label{eq:raw-r-def}
\end{equation}
Then $r$ is canonical: $\E_Y r(x,Y)=0$ and $\E_G r(G,y)=0$.

We also isolate the degree-one Gaussian signal.  Define
\begin{equation}
        m_1=\E\left[S\one_{\{S\ge u\}}\right],
        \qquad
        \beta=\frac{m_1}{\sqrt d}.
\label{eq:raw-beta-def}
\end{equation}
The degree-one component of $q_{\rm G}$, equivalently of $r$, is the
projection onto the span of the bilinear coordinate kernels $x_k y_k$,
$1\le k\le d$.  Indeed, by rotational symmetry,
\[
        \E\bigl[q_{\rm G}(G,G')G_aG'_b\bigr]
        =
        \frac{\delta_{ab}}{d}
        \E\bigl[q_{\rm G}(G,G')\langle G,G'\rangle\bigr]
        =
        \delta_{ab}\beta.
\]
Thus this projection is $\beta\langle x,y\rangle$.  Let
\[
        r_\perp(x,y):=r(x,y)-\beta\langle x,y\rangle.
\]

The Gaussian population estimates used below are collected in
\cref{lem:raw-gaussian-coeff}; their proof is deferred to
Appendix~\ref{sec:gaussian-kernel-estimates}.

\begin{lemma}[First projection bound]\label{lem:raw-first-projection}
Let $a_G=(a(G_1),\ldots,a(G_n))^\top$.  Under the assumptions of
\cref{thm:raw-gaussian-concentration}, with probability at least
$1-n^{-D}$,
\[
        \|a_G\|_2
        \le
        C_D pL\sqrt{\frac nd},
        \qquad
        \|a_G\|_\infty
        \le
        C_D pL\sqrt{\frac{\log n}{d}}+n^{-D-2}.
\]
\end{lemma}

\begin{proof}
Let
\[
        f(s)=\overline\Phi(u/s),\qquad
        s_i=\frac{\|G_i\|}{\sqrt d}.
\]

Conditioning on $G$ in the definition of the threshold gives
$p=\E f(\|G\|/\sqrt d)$.  Hence
\[
        a(G_i)=f(s_i)-p
        =
        \bigl(f(s_i)-f(1)\bigr)-\eta,
        \qquad
        \eta:=\E\bigl[f(\|G\|/\sqrt d)-f(1)\bigr].
\]
Thus the size of $a(G_i)$ is controlled by the radius fluctuation
$s_i-1$.

Let
\[
        t_n=C_D\sqrt{\frac{\log n}{d}},
        \qquad
        \mathcal E_{\rm rad}:=\{\max_i |s_i-1|\le t_n\}.
\]
By Gaussian norm concentration \cite[Theorem~3.1.1]{VershyninHDP},
$\Pp(\mathcal E_{\rm rad})\ge1-n^{-D-2}$ after increasing $C_D$.
After increasing the constant in the lower bound on $d$, the assumption
$d\ge C_D L^2\log n$ makes $t_n\le c/L$.  For every $s\in[1-c/L,1+c/L]$,
\[
        |f'(s)|
        =
        \frac{u}{s^2}\varphi(u/s)
        \le CpL
\]
by \cref{lem:raw-gaussian-coeff} and the Gaussian tail comparison
\cite[Proposition~2.1.2]{VershyninHDP}.  In particular, $f$ is
$CpL$-Lipschitz on $[1-c/L,1+c/L]$.

We also use a quadratic, rather than a coordinatewise, radius bound.  Since
the Euclidean norm is one-Lipschitz and
$|\E\|G_i\|/\sqrt d-1|\le C/d$, Gaussian concentration
\cite[Theorem~5.2.2]{VershyninHDP} shows that
$\sqrt d(s_i-1)$ has uniformly bounded subgaussian norm.  Consequently,
$d(s_i-1)^2$ has uniformly bounded subexponential norm
\cite[Lemma~2.7.7]{VershyninHDP}.  Scalar Bernstein's inequality
\cite[Theorem~2.8.1]{VershyninHDP} and $n\ge C_D\log n$ therefore give
\begin{equation}
        \sum_{i=1}^n(s_i-1)^2\le C_D\frac nd
\label{eq:raw-radius-quadratic-sum}
\end{equation}
with probability at least $1-n^{-D-2}$.  On the intersection of this event
with $\mathcal E_{\rm rad}$,
\begin{equation}
        \max_i |f(s_i)-f(1)|
        \le
        C_DpL\sqrt{\frac{\log n}{d}},
        \qquad
        \left(\sum_i |f(s_i)-f(1)|^2\right)^{1/2}
        \le
        C_DpL\sqrt{\frac nd}.
\label{eq:raw-first-projection-fluctuation}
\end{equation}
 
The centering term satisfies
\begin{equation}
        |\eta|
        \le
        \E|f(\|G\|/\sqrt d)-f(1)|
        \le
        CpL\,\E\left|\frac{\|G\|}{\sqrt d}-1\right|
        +
        \Pp\left\{\left|\frac{\|G\|}{\sqrt d}-1\right|>\frac cL\right\}
        \le
        C\frac{pL}{\sqrt d},
\label{eq:raw-first-projection-centering}
\end{equation}
where the last step uses
$\E|\|G\|/\sqrt d-1|\le C/\sqrt d$ and the same Gaussian norm
concentration; the tail term is absorbed under
$d\ge C_D L^2\log n$, after increasing $C_D$.  Combining
\eqref{eq:raw-first-projection-fluctuation} and
\eqref{eq:raw-first-projection-centering} gives the stated $\ell_2$ and
$\ell_\infty$ bounds on the
intersection of the two events.  A union bound completes the proof.
\end{proof}

The next proposition is purely algebraic.  It identifies the one-endpoint
radius fluctuation in the adjacency matrix and explains why
double-centering is useful: the two rank-one first-projection terms
disappear, leaving only the canonical matrix and a diagonal correction.

\begin{proposition}[Gaussian vector Hoeffding reduction]\label{prop:raw-gaussian-reduction}
Let
\[
        R_{ij}=r(G_i,G_j)\one_{\{i\ne j\}},
        \qquad
        a_G=(a(G_1),\ldots,a(G_n))^\top.
\]
Then
\begin{equation}
        A^{\rm G}-\E A^{\rm G}
        =
        R+a_G\one^\top+\one a_G^\top-2\diag(a_G).
\label{eq:raw-exact-decomp}
\end{equation}
Consequently,
\begin{equation}
        \Pi_n(A^{\rm G}-\E A^{\rm G})\Pi_n
        =
        \Pi_nR\Pi_n-2\Pi_n\diag(a_G)\Pi_n.
\label{eq:raw-centered-decomp}
\end{equation}
\end{proposition}

\begin{proof}
For $i\ne j$,
\[
        A^{\rm G}_{ij}-p
        =
        r(G_i,G_j)+a(G_i)+a(G_j).
\]
This gives \eqref{eq:raw-exact-decomp}.  Multiplying on both sides by
$\Pi_n$ and using $\Pi_n\one=0$ gives \eqref{eq:raw-centered-decomp}.
\end{proof}

It remains to control the canonical matrix $R$.
The following proposition applies the same decoupling-and-Chernoff framework
as \cref{prop:decoupled-G}.  The only new step
is to restrict to typical radii and recenter the kernel under the resulting
truncated Gaussian law.

\begin{proposition}[Canonical Gaussian matrix bound]\label{prop:raw-canonical-bound}
Let $k$ denote either $r$ or $r_\perp$, and let
$\Lambda_k=\|T_k\|_{\op}$.  Under the assumptions of
\cref{thm:raw-gaussian-concentration}, with the additional condition
$d\le c n$ when $k=r_\perp$,
\[
        \left\|
        \bigl(k(G_i,G_j)\one_{\{i\ne j\}}\bigr)_{i,j=1}^n
        \right\|_{\op}
        \le
        C_D\left(\sqrt{np\log n}+n\Lambda_k\right)
\]
with probability at least $1-n^{-D}$.
\end{proposition}

\begin{proof}
\proofstep{Step 1: Truncation and recentering.}
We spell out the truncation because conditioning changes the canonical
measure.  Let
\[
        \mathcal T=\left\{x:
        \left|\frac{\|x\|}{\sqrt d}-1\right|\le\frac cL\right\},
        \qquad
        \varepsilon=\gamma_d(\mathcal T^c),
        \qquad
        \nu=\gamma_d(\,\cdot\mid\mathcal T).
\]
Gaussian norm concentration \cite[Theorem~3.1.1]{VershyninHDP} and
$d\ge C_DL^2\log n$ give
$\varepsilon\le n^{-2D-20}$.  For $x\in\mathcal T$, define
\[
        m(x)=\int k(x,y)\,d\nu(y),
        \qquad
        m_0=\int m(x)\,d\nu(x),
\]
and
\[
        k_{\mathcal T}(x,y)=k(x,y)-m(x)-m(y)+m_0.
\]
Then $k_{\mathcal T}$ is canonical under $\nu$.  Since $k$ is canonical
under $\gamma_d$ and \cref{lem:raw-gaussian-coeff} bounds its squared section norm
by $Cp$ uniformly for $x\in\mathcal T$, Cauchy--Schwarz gives
\[
        \sup_{x\in\mathcal T}|m(x)|+|m_0|\le C\sqrt{p\varepsilon}.
\]
Moreover, if $Q_\nu$ denotes the orthogonal projection onto the mean-zero
subspace of $L^2(\nu)$, then
\[
        T_{k_{\mathcal T},\nu}
        =Q_\nu T_{k,\nu}Q_\nu,
        \qquad
        \|T_{k_{\mathcal T},\nu}\|_{\op}
        \le\frac{\Lambda_k}{1-\varepsilon}.
\]
The last inequality follows by identifying $L^2(\nu)$ with the subspace of
$L^2(\gamma_d)$ supported on $\mathcal T$ and observing that restriction and
renormalization multiply the operator norm by at most
$(1-\varepsilon)^{-1}$.

\proofstep{Step 2: Conditional covariance and uniform column bound.}
We now repeat the proof of \cref{prop:decoupled-G} under $\nu$, using the
zero-diagonal decoupled matrix required by \cref{lem:decoupling}.  The section
bound in \cref{lem:raw-gaussian-coeff} gives
\[
        \sup_{x\in\mathcal T}
        \int k_{\mathcal T}(x,y)^2\,d\nu(y)\le Cp.
\]
Consequently, \cref{lem:hilbert-cov} bounds the conditional covariance of
the decoupled columns by
\[
        C_D\bigl(n\Lambda_k^2+p\log n\bigr).
\]
If $\Sigma$ is the covariance of a full column, the identity
\eqref{eq:zero-diagonal-covariance-sum} shows that deleting the diagonal does
not enlarge the Chernoff mean scale.
For the column bound, fix a typical second endpoint $y$.  Under $\nu$, the
number of threshold indicators in that column is binomial with success
probability at most
$(1-\varepsilon)^{-1}\overline\Phi(u\sqrt d/\|y\|)\le Cp$.  The binomial
Chernoff and Markov argument of \cref{lem:column-tail} therefore bounds all
sampled threshold columns by $C_Dnp$ with the required conditional
probability.  The first Hoeffding projections are $O(p)$ on $\mathcal T$ and
contribute at most $Cnp^2$.  If $k=r_\perp$, the additional linear column
satisfies
\[
        \sum_{i=1}^n\beta^2\langle G_i,y\rangle^2
        \le C\beta^2 n\|y\|^2
        \le Cnp^2L
        \le Cnp
\]
on the standard Gaussian sample-covariance event
\cite[Theorem~4.6.1]{VershyninHDP}; conditioning all sample points to lie in
$\mathcal T$ changes its failure probability by at most
$(1-\varepsilon)^{-n}$.  Here we used $d\le cn$ and
$\beta^2\le Cp^2L/d$.  Thus the squared column bound is $C_Dnp$.

\proofstep{Step 3: Decoupling and removal of the truncation.}
Applying \cref{lem:matrix-chernoff} conditionally therefore
gives
\begin{equation}
        \|G^0_{\mathcal T}\|_{\op}
        \le C_D\left(\sqrt{np\log n}+n\Lambda_k\right).
\label{eq:raw-truncated-decoupled-bound}
\end{equation}
The decoupling inequality in \cref{lem:decoupling} transfers
\eqref{eq:raw-truncated-decoupled-bound} to the zero-diagonal
matrix generated by
$k_{\mathcal T}$.  On the event that all original sample points lie in
$\mathcal T$, the matrices generated by $k$ and $k_{\mathcal T}$ differ by
the off-diagonal part of $m\one^\top+\one m^\top-m_0J$.  Its operator norm is
at most $Cn\sqrt{p\varepsilon}$, including the diagonal correction, and is
smaller than $n^{-D-2}$.  Finally,
$\Pp\{\exists i:G_i\notin\mathcal T\}\le n\varepsilon$.  Combining the
events and relabeling $D$ proves the proposition.
\end{proof}

\begin{proof}[Proof of \cref{thm:raw-gaussian-concentration}]

The condition $np\ge C_D\log n$ implies
$L\le C\log n$.  Hence $d\ge C_DL^2\log n$ implies the
$d\ge CL^3$ hypothesis of \cref{lem:raw-gaussian-coeff}.
Apply \cref{prop:raw-canonical-bound} to $r$ with tail exponent $D+2$ and
use \cref{lem:raw-gaussian-coeff}.  This gives
\begin{equation}
        \|R\|_{\op}
        \le
        C_D\left(
        \sqrt{np\log n}
        +np\sqrt{\frac Ld}
        +np\frac{L^2}d
        \right)
\label{eq:raw-R-bound}
\end{equation}
with probability at least $1-n^{-D-2}$.  On the event from
\cref{lem:raw-first-projection}, also taken with tail exponent $D+2$, the
Hoeffding reduction \cref{prop:raw-gaussian-reduction} gives
\[
        \|\Pi_n(A^{\rm G}-\E A^{\rm G})\Pi_n\|_{\op}
        \le
        \|R\|_{\op}+2\|a_G\|_\infty.
\]
The $\|a_G\|_\infty$ term is absorbed by $\sqrt{np\log n}$ because
$pL$ is bounded for $0<p\le p_0$ and $np\ge C_D\log n$.  This proves
\eqref{eq:raw-centered-main}, since $d\ge C_DL^2\log n$ and
$L\le C\log n$ imply
\[
        \frac{L^2}{d}\le C\sqrt{\frac Ld}.
\]

For the full-matrix estimate, \eqref{eq:raw-exact-decomp} gives
\[
        \|a_G\one^\top+\one a_G^\top-2\diag(a_G)\|_{\op}
        \le
        2\sqrt n\,\|a_G\|_2+2\|a_G\|_\infty,
\]
which we combine with the bound \eqref{eq:raw-R-bound} on
$\|R\|_{\op}$.  By
\cref{lem:raw-first-projection},
\[
        2\sqrt n\,\|a_G\|_2
        \le
        C_D\frac{npL}{\sqrt d},
\]
and the $\|a_G\|_\infty$ term is again absorbed by the sparse term.  The
two geometric terms in $\|R\|_{\op}$ are at most $CnpL/\sqrt d$ because
$L\ge1$ and $d\ge L^2$.  A union bound over the events in
\cref{prop:raw-canonical-bound,lem:raw-first-projection}
proves \cref{eq:raw-centered-main,eq:raw-full-main} with
probability at least $1-n^{-D}$.
\end{proof}

For the recovery argument, let $Z$ have rows $G_i^\top/\sqrt d$, put
$Z_c=\Pi_nZ$, and define
\[
        E_{\rm G}:=
        \Pi_n(A^{\rm G}-\E A^{\rm G})\Pi_n+d\beta\Pi_n
        -d\beta Z_cZ_c^\top.
\]
The next lemma is the only additional concentration input needed for vector
approximation.

\begin{lemma}[Gaussian higher-chaos residual]
\label{lem:raw-recovery-residual}
Under the assumptions of \cref{thm:raw-gaussian-inner-product}, with
probability at least $1-n^{-D}$,
\begin{equation}
        \|E_{\rm G}\|_{\op}
        \le
        C_D\left(
        \sqrt{np\log n}+np\frac{L^2}{d}
        \right).
\label{eq:gaussian-frob-residual}
\end{equation}
\end{lemma}

\begin{proof}
Let $R_\perp$ be the zero-diagonal empirical matrix generated by
$r_\perp$, and write
\[
        D_Z=\diag(\|Z_1\|_2^2,\ldots,\|Z_n\|_2^2),
        \qquad D_a=\diag(a_G).
\]
The Hoeffding decomposition and the removal of the degree-one Gaussian
chaos give the exact identity
\begin{equation}
        E_{\rm G}
        =
        \Pi_nR_\perp\Pi_n
        -d\beta\Pi_n(D_Z-I_n)\Pi_n
        -2\Pi_nD_a\Pi_n.
\label{eq:gaussian-frob-residual-decomposition}
\end{equation}
Indeed, the off-diagonal matrix generated by
$\beta\langle G_i,G_j\rangle$ is
$d\beta(ZZ^\top-D_Z)$, while
$\Pi_nZZ^\top\Pi_n=Z_cZ_c^\top$.

By \cref{prop:raw-canonical-bound,lem:raw-gaussian-coeff},
\begin{equation}
        \|R_\perp\|_{\op}
        \le
        C_D\left(\sqrt{np\log n}+np\frac{L^2}{d}\right).
\label{eq:gaussian-recovery-rperp-bound}
\end{equation}
Gaussian norm concentration \cite[Theorem~3.1.1]{VershyninHDP} and a union
bound give
\begin{equation}
        \max_i\bigl|\|Z_i\|_2^2-1\bigr|
        \le C_D\sqrt{\frac{\log n}{d}},
\label{eq:gaussian-recovery-norm-bound}
\end{equation}
where the smaller term $\log n/d$ is absorbed because
$d\ge C_DL^2\log n$.  Since
$d\beta\asymp p\sqrt{dL}$,
\begin{equation}
        d\beta\|D_Z-I_n\|_{\op}
        \le C_Dp\sqrt{L\log n}.
\label{eq:gaussian-recovery-DZ-bound}
\end{equation}
Likewise, \cref{lem:raw-first-projection} gives
\begin{equation}
        \|D_a\|_{\op}
        \le C_DpL\sqrt{\frac{\log n}{d}}+n^{-D-2}.
\label{eq:gaussian-recovery-Da-bound}
\end{equation}
Both diagonal terms are absorbed by $\sqrt{np\log n}$ under
$np\ge C_D\log n$ and $pL\le C$.  Inserting
\eqref{eq:gaussian-recovery-rperp-bound},
\eqref{eq:gaussian-recovery-DZ-bound}, and
\eqref{eq:gaussian-recovery-Da-bound} into
\eqref{eq:gaussian-frob-residual-decomposition} proves the lemma.
\end{proof}

\begin{proof}[Proof of \cref{thm:raw-gaussian-inner-product}]
\proofstep{Step 1: Sample covariance and signal eigengap.}
Put $L=1+\log(1/p)$ and
\[
        M=\frac dn Z_c^\top Z_c.
\]
By orthogonal invariance,
$G^\top\Pi_nG$ has the same law as the Gram matrix of $n-1$ independent
$N(0,I_d)$ vectors.  The Wishart concentration bound
\cite[Theorem~4.6.1]{VershyninHDP} therefore gives, with probability at least
$1-n^{-D-2}$,
\begin{equation}
        \|M-I_d\|_{\op}
        \le
        \delta_n:=C_D\sqrt{\frac{d+\log n}{n}}.
\label{eq:gaussian-frob-covariance}
\end{equation}
The upper bound on $d$ in
\cref{thm:raw-gaussian-inner-product} implies $d\le c_Dn$,
and hence the constants may
be chosen such that $\delta_n\le1/2$.
On this event $M$ is positive definite, so $Z_c$ has full column rank.  We
may therefore define
\begin{equation}
        Y^{\rm G}
        :=Z_c(Z_c^\top Z_c)^{-1/2}
        =\sqrt{\frac dn}\,Z_cM^{-1/2}.
\label{eq:gaussian-population-orthonormalization}
\end{equation}

The nonzero eigenvalues of the signal matrix
$d\beta Z_cZ_c^\top$ are the eigenvalues of $n\beta M$ and therefore are at
least $n\beta/2$.  By
\cref{lem:raw-recovery-residual,lem:raw-gaussian-coeff},
\begin{equation}
\begin{aligned}
        \rho_{\rm G}
        &:=
        \frac{\|E_{\rm G}\|_{\op}}{n\beta}\le
        C_D\left(
        \sqrt{\frac{d\log n}{npL}}
        +\frac{L^{3/2}}{\sqrt d}
        \right).
\end{aligned}
\label{eq:gaussian-frob-ratio}
\end{equation}
The lower and upper dimension bounds, with the constants chosen sufficiently
large and small, respectively, give $\rho_{\rm G}\le1/8$.

\proofstep{Step 2: Eigenspace and vector recovery.}
The matrix
\[
        \Pi_n(A^{\rm G}-\E A^{\rm G})\Pi_n+d\beta\Pi_n
        =
        d\beta Z_cZ_c^\top+E_{\rm G}
\]
preserves $\one^\perp$.  On that subspace, adding $d\beta\Pi_n$ shifts every
eigenvalue by the same amount.  The signal matrix has its $d$ nonzero
eigenvalues in $[n\beta/2,3n\beta/2]$.  Since
$\|E_{\rm G}\|_{\op}\le n\beta/8$, Weyl's inequality places the smallest
perturbed signal eigenvalue above $3n\beta/8$ and every remaining eigenvalue
of the shifted matrix below $n\beta/8$.  After the shift is removed, the
signal eigenvalues are therefore at least $(3n/8-d)\beta$.  The upper bound on
$d$ in \cref{thm:raw-gaussian-inner-product} allows the
constant $c_D$ to be chosen so that $d\le n/8$; hence these eigenvalues remain
positive.  Removing the shift preserves the ordering within $\one^\perp$,
whereas the eigenvalue on $\operatorname{span}\{\one\}$ remains zero.  Thus
the top $d$ eigenspace in \cref{thm:raw-gaussian-inner-product} is precisely
the signal cluster to which
\cref{lem:frobenius-perturbation} applies.  We obtain an $O\in O(d)$ such
that
\[
        \|\widehat Y^{\rm G}O-Y^{\rm G}\|_{\mathrm F}
        \le C_D\sqrt d\,\rho_{\rm G}.
\]
Multiplying by $\sqrt{n/d}$ and dividing by $\sqrt n$ gives
\begin{equation}
 \frac1{\sqrt n}
 \left\|\widehat ZO-Z_cM^{-1/2}\right\|_{\mathrm F}
 \le C_D\rho_{\rm G}.
\label{eq:gaussian-aligned-signal-bound}
\end{equation}
On the covariance event
\eqref{eq:gaussian-frob-covariance},
$\|Z_c\|_{\mathrm F}\le C\sqrt n$ and
$\|M^{-1/2}-I_d\|_{\op}\le C\delta_n$.  Consequently,
\begin{equation}
 \frac1{\sqrt n}\|Z_cM^{-1/2}-Z_c\|_{\mathrm F}
 \le C\delta_n.
\label{eq:gaussian-covariance-centering-error}
\end{equation}
Finally, if $\bar Z=n^{-1}\sum_iZ_i$, then
$Z-Z_c=\one\bar Z^\top$.  Since
$\bar Z\sim N(0,(nd)^{-1}I_d)$, Gaussian norm concentration
\cite[Theorem~3.1.1]{VershyninHDP} gives
\begin{equation}
        \|\bar Z\|_2\le C_D/\sqrt n\le C_D\delta_n.
\label{eq:gaussian-sample-mean-error}
\end{equation}
Combining \eqref{eq:gaussian-aligned-signal-bound},
\eqref{eq:gaussian-covariance-centering-error}, and
\eqref{eq:gaussian-sample-mean-error} with
\eqref{eq:gaussian-frob-ratio} proves
\eqref{eq:gaussian-vector-recovery}, because
$pL\le C$ and $d\ge C_DL$ imply
\[
        \delta_n^2= C_D^2\frac{d+\log n}{n}
        \le C_D\frac{d\log n}{npL}.
\]

\proofstep{Step 3: Gram matrix recovery.}
Standard Gaussian singular-value bounds \cite[Theorem~4.6.1]{VershyninHDP}
and $d\le c_Dn$ give
\[
        \|\widehat Z\|_{\op}=\sqrt{\frac nd},
        \qquad
        \|Z\|_{\op}\le C\sqrt{\frac nd}.
\]
Thus \cref{lem:gram-perturbation}, with the alignment in
\eqref{eq:gaussian-aligned-signal-bound}, yields
\[
 \frac{\sqrt d}{n}\|\widehat K^{\rm G}-ZZ^\top\|_{\mathrm F}
 \le \frac C{\sqrt n}\|\widehat ZO-Z\|_{\mathrm F}.
\]
The vector bound \eqref{eq:gaussian-vector-recovery},
together with a union bound over the events in
\eqref{eq:gaussian-frob-covariance},
\cref{lem:raw-recovery-residual}, and
\eqref{eq:gaussian-sample-mean-error} proves
\eqref{eq:gaussian-gram-recovery} and completes the proof.
\end{proof}

\section{Proof of the synchronization theorem}
\label{sec:synchronization-proof}

This section proves \cref{thm:global-synchronization}.  The argument has two
inputs.  We first record the deterministic synchronization criterion in the
form needed here.  We then derive the centered Laplacian bounds from the
adjacency estimates by controlling every degree.

\begin{lemma}[Deterministic synchronization criterion, Theorem~1.10 in \citet{AbdallaEtAlExpander}]
\label{lem:deterministic-synchronization}
Let $M$ be the adjacency matrix of a graph on $n$ vertices, let
$\mathcal L_M=\diag(M\one)-M$, and let $q>0$.  Suppose that, for some
$0<\delta\le1/5$,
\[
        \|M-qJ\|_{\op}\le\delta nq,
        \qquad
        \|\mathcal L_M+qJ-nqI_n\|_{\op}\le\delta nq.
\]
If
\begin{equation}
 \frac{32\delta(1+7\delta)}{(1-\delta)^2}
 \log\!\left(\frac{1+2\delta}{2\delta}\right)<1,
\label{eq:synchronization-delta-condition}
\end{equation}
then the graph is globally synchronizing.
\end{lemma}

For a symmetric zero-diagonal matrix $M$, write
$d_i(M)=\sum_{j\ne i}M_{ij}$ for its $i$th row sum.  The next lemma is the
only additional probabilistic input needed for the Laplacian.

\begin{lemma}[Uniform degree concentration]
\label{lem:synchronization-degrees}
For every $D>0$ there are constants $C_D>0$ and $p_0>0$ such that, if
$0<p\le p_0$ and $np\ge C_D\log n$, then the following statements hold.
\begin{enumerate}[label=\textup{(\roman*)},leftmargin=2.2em]
\item For the spherical threshold graph, with probability at least $1-n^{-D}$,
\[
        \max_{i\in[n]}|d_i(A)-(n-1)p|
        \le C_D\sqrt{np\log n}.
\]
\item For the Gaussian vector graph, if
$d\ge C_D(1+\log(1/p))^2\log n$, then, with probability at least $1-n^{-D}$,
\[
 \max_{i\in[n]}|d_i(A^{\rm G})-(n-1)p|
 \le C_D\left(
        \sqrt{np\log n}
        +np(1+\log(1/p))\sqrt{\frac{\log n}{d}}
        \right).
\]
\end{enumerate}
\end{lemma}

\begin{proof}
For either model, let $q_i$ be the edge probability conditional on the latent
vector at vertex $i$.  Conditional on that vector,
$d_i\sim\operatorname{Bin}(n-1,q_i)$.  Hence the scalar Bernstein bound
\cite[Section~2.8]{VershyninHDP} and $np\ge C_D\log n$ imply
\begin{equation}
 \Pp\left\{
   |d_i-(n-1)q_i|>K_D\sqrt{np\log n},\ q_i\le2p
 \right\}
 \le2n^{-D-2}.
\label{eq:degree-common-binomial-bound}
\end{equation}
The estimate \eqref{eq:degree-common-binomial-bound} is
applied separately to each vertex and then union-bounded;
no independence among the degrees is needed.

\proofstep{Spherical graph.}
Rotational invariance gives, for every $i\ne j$,
\[
 \Pp\{\langle U_i,U_j\rangle\ge\tau\mid U_i\}=p.
\]
Thus $q_i=p$, and \eqref{eq:degree-common-binomial-bound} followed by a union
bound proves part~\textup{(i)}.

\proofstep{Gaussian vector graph.}
Put $L=1+\log(1/p)$ and define
\[
        q_i
        :=\Pp\{\langle G_i,G\rangle/\sqrt d\ge u\mid G_i\}
        =\overline\Phi\!\left(\frac{u\sqrt d}{\|G_i\|}\right),
\]
where $G\sim N(0,I_d)$ is independent of $G_i$ and
$\overline\Phi(x)=\Pp\{N(0,1)\ge x\}$.  Thus
$q_i-p=a(G_i)$ in the notation of \cref{lem:raw-first-projection}.
Apply \cref{lem:raw-first-projection} with tail exponent
$D+3$.  With probability at least
$1-n^{-D-3}$,
\begin{equation}
        \max_i|q_i-p|
        \le K_DpL\sqrt{\frac{\log n}{d}}+n^{-D-5}.
\label{eq:gaussian-degree-probability-shift}
\end{equation}
Choose the constant in the dimension assumption so that
$C_D\ge4K_D^2$.  Then
$K_DL\sqrt{\log n/d}\le1/2$, so the first term is at most $p/2$.
The second is also at most $p/2$ because
$p\ge C_D(\log n)/n$.  Hence $q_i\le2p$ for every $i$ on this event.

Using \eqref{eq:degree-common-binomial-bound}, a union bound, and
\eqref{eq:gaussian-degree-probability-shift}, we obtain
\[
\begin{aligned}
 \max_i|d_i(A^{\rm G})-(n-1)p|
 &\le
 \max_i|d_i(A^{\rm G})-(n-1)q_i|+(n-1)\max_i|q_i-p|\\
 &\le C_D\left(
        \sqrt{np\log n}
        +npL\sqrt{\frac{\log n}{d}}
        \right).
\end{aligned}
\]
This proves part~\textup{(ii)}.
\end{proof}

\begin{proof}[Proof of \cref{thm:global-synchronization}]
Choose a universal constant $\delta_0\in(0,1/5)$ so small that
\begin{equation}
 \frac{32\delta_0(1+7\delta_0)}{(1-\delta_0)^2}
 \log\!\left(\frac{1+2\delta_0}{2\delta_0}\right)<1.
\label{eq:fixed-synchronization-delta}
\end{equation}
Such a choice is possible because $\delta\log(1/\delta)\to0$ as
$\delta\downarrow0$.  We prove that the two operator-norm deviations in
\cref{lem:deterministic-synchronization} are at most $\delta_0np$.

Both models satisfy $\E M=p(J-I_n)$.  If
$\mathcal D_M=\diag(M\one)$ and $\mathcal L_M=\mathcal D_M-M$, then
\begin{equation}
\begin{aligned}
 M-pJ&=(M-\E M)-pI_n,\\
 \mathcal L_M+pJ-npI_n
 &=\mathcal D_M-(n-1)pI_n-(M-\E M).
\end{aligned}
\label{eq:common-sync-identities}
\end{equation}
We write $\mathcal L_A$ and $\mathcal L_{\rm G}$ for the two resulting
Laplacians.

\proofstep{Spherical graph.}
Apply \cref{thm:main-concentration,lem:synchronization-degrees} with tail
exponent $D+3$.  Substituting $M=A$ in
\eqref{eq:common-sync-identities} gives, simultaneously,
\begin{equation}
\begin{aligned}
 \|A-pJ\|_{\op}
 &\le K_D\bigl(\sqrt{np\log n}+np\tau\bigr),\\
 \|\mathcal L_A+pJ-npI_n\|_{\op}
 &\le K_D\bigl(\sqrt{np\log n}+np\tau\bigr).
\end{aligned}
\label{eq:spherical-sync-two-deviations}
\end{equation}
Here the harmless term $pI_n$ is absorbed by $\sqrt{np\log n}$.

Put again $L=1+\log(1/p)$.  By \cref{lem:higher-harmonics}, the condition
$d\ge C_DL$ gives $\tau\le C\sqrt{L/d}$.  Dividing
\eqref{eq:spherical-sync-two-deviations} by $np$ therefore gives
\begin{equation}
 \max\left\{
   \frac{\|A-pJ\|_{\op}}{np},
   \frac{\|\mathcal L_A+pJ-npI_n\|_{\op}}{np}
 \right\}
 \le K_D\left(
        \sqrt{\frac{\log n}{np}}+\sqrt{\frac Ld}
        \right).
\label{eq:spherical-sync-normalized}
\end{equation}
By increasing the constants in the assumptions
$np\ge C_D\log n$ and $d\ge C_DL$, the right-hand side of
\eqref{eq:spherical-sync-normalized} is at most
$\delta_0$.  The conclusion follows from
\cref{lem:deterministic-synchronization} and
\eqref{eq:fixed-synchronization-delta}.

\proofstep{Gaussian vector graph.}
Apply the full-matrix estimate
\eqref{eq:raw-full-main} and
\cref{lem:synchronization-degrees}, again with tail exponent $D+3$.
Using \eqref{eq:common-sync-identities} with $M=A^{\rm G}$ and $L\ge1$, we
obtain
\begin{equation}
\begin{aligned}
 &\max\left\{
   \frac{\|A^{\rm G}-pJ\|_{\op}}{np},
   \frac{\|\mathcal L_{\rm G}+pJ-npI_n\|_{\op}}{np}
 \right\}\\
 &\qquad\le K_D\left(
        \sqrt{\frac{\log n}{np}}
        +L\sqrt{\frac{\log n}{d}}
        \right).
\end{aligned}
\label{eq:gaussian-sync-normalized}
\end{equation}
The assumptions $np\ge C_D\log n$ and $d\ge C_DL^2\log n$, with constants
chosen sufficiently large, make the right-hand side of
\eqref{eq:gaussian-sync-normalized} at most $\delta_0$.
Another application of \cref{lem:deterministic-synchronization} completes the
Gaussian proof.  A union bound over the events in
\cref{thm:main-concentration,thm:raw-gaussian-concentration,lem:synchronization-degrees}
gives probability
at least $1-n^{-D}$ in both models.
\end{proof}

\section{Proofs for the Gaussian mixture block model}\label{sec:gmbm}

This section proves
\cref{thm:gmbm-exact,thm:gmbm-large-separation}.  We use the model and
notation from \cref{sec:main-gmbm}.  Throughout this section,
$L=1+\log(1/p)$.

\subsection{Exact recovery at moderate separation}
\label{sec:gmbm-exact-proof}

For two independent vertices, define the conditional edge probabilities and
their half-difference by
\[
\begin{aligned}
        p_+(\mu)&:=\Pp\{\sqrt d\,\langle Z_1,Z_2\rangle\ge\tau
        \mid \xi_1\xi_2=1\},\\
        p_-(\mu)&:=\Pp\{\sqrt d\,\langle Z_1,Z_2\rangle\ge\tau
        \mid \xi_1\xi_2=-1\},
        \qquad
        \Delta_\mu:=\frac{p_+(\mu)-p_-(\mu)}2.
\end{aligned}
\]
The argument has four components.  First, we estimate
the two conditional tails and hence the label signal.  Second, we bound the
population operator after removing this signal and the first Hoeffding
projections.  Third, decoupling and scalar Bernstein inequalities
\cite[Section~2.8]{VershyninHDP} transfer
these population estimates to the sampled matrix.  Finally, a
signed-Laplacian dual certificate converts the operator and signed-row bounds
into exact recovery.

For later use, let $Y=(\xi,G)$ and $Y'=(\xi',G')$ be independent, write
$Z(Y)=G/\sqrt d+\mu\xi e_1$, and define
\[
\begin{aligned}
        h_\mu(Y,Y')
        &:=\one_{\{\sqrt d\langle Z(Y),Z(Y')\rangle\ge\tau\}}
        -p-\Delta_\mu\xi\xi',\\
        a_\mu(Y)&:=\E_{Y'}h_\mu(Y,Y'),\\
        r_\mu^{\rm lab}(Y,Y')
        &:=h_\mu(Y,Y')-a_\mu(Y)-a_\mu(Y').
\end{aligned}
\]
Since $\E h_\mu(Y,Y')=p-p-\Delta_\mu\E[\xi\xi']=0$, we have
$\E a_\mu(Y)=0$, and the kernel $r_\mu^{\rm lab}$ is canonical.  We write
$T_{r_\mu^{\rm lab}}$ for its integral operator and
$\Lambda_\mu^{\rm lab}=\|T_{r_\mu^{\rm lab}}\|_{\op}$.

For the empirical argument, set $\xi_c=\Pi_n\xi$ and define
\[
        E_\mu^{\rm lab}
        :=\Pi_n\bigl(A^\mu-\E[A^\mu\mid\xi]\bigr)\Pi_n
        -\Delta_\mu\Pi_n.
\]
The conditional edge probability is $p_+(\mu)$ when
$\xi_i\xi_j=1$ and $p_-(\mu)$ when $\xi_i\xi_j=-1$.  Hence, for $i\ne j$,
$\E[A^\mu_{ij}\mid\xi]=p+\Delta_\mu\xi_i\xi_j$.  Therefore
\[
        \E[A^\mu\mid\xi]-p(J-I_n)
        =\Delta_\mu(\xi\xi^\top-I_n).
\]
Applying $\Pi_n$ on both sides and using
$\Pi_n(\xi\xi^\top-I_n)\Pi_n=\xi_c\xi_c^\top-\Pi_n$ gives
\begin{equation}
        B_\mu=\Delta_\mu\xi_c\xi_c^\top+E_\mu^{\rm lab}.
\label{eq:gmbm-label-projection-main}
\end{equation}

The two parts of the next lemma use different concentration mechanisms.
Decoupling and matrix Chernoff control the operator norm, while conditional
scalar Bernstein controls the signed sum in each row.  The latter is the
quantity needed by the SDP dual certificate and is stronger than what an
operator-norm estimate alone would provide.

\begin{lemma}[Label-residual concentration]\label{lem:gmbm-label-residual}
Under the assumptions of \cref{thm:gmbm-exact}, with probability at least
$1-n^{-D}$,
\begin{equation}
        \|E_\mu^{\rm lab}\|_{\op}
        \le
        C_D\left(\sqrt{np\log n}
        +np\mu\sqrt L+np\sqrt{\frac Ld}\right)
\label{eq:gmbm-residual-operator}
\end{equation}
and
\begin{equation}
        \|E_\mu^{\rm lab}\xi_c\|_\infty
        \le
        C_D\left(\sqrt{np\log n}
        +np\mu\sqrt{L\log n}\right).
\label{eq:gmbm-residual-row}
\end{equation}
\end{lemma}

\begin{proof}
\proofstep{Step 1: Empirical Hoeffding decomposition.}
We use the kernels $h_\mu,a_\mu,r_\mu^{\rm lab}$ defined at the beginning of
this section.  By construction,
\[
        \E_{Y'} r_\mu^{\rm lab}(Y,Y')
        =
        \E_Y r_\mu^{\rm lab}(Y,Y')=0,
\]
so $r_\mu^{\rm lab}$ is canonical, and its integral-operator norm is
$\Lambda_\mu^{\rm lab}$.

We first relate this population decomposition to the zero-diagonal empirical
matrix.  Let
\[
        R_{\mu,ij}^{\rm lab}
        =
        r_\mu^{\rm lab}(Y_i,Y_j)\one_{\{i\ne j\}},
        \qquad
        a_i=a_\mu(Y_i).
\]
For $i\ne j$,
\[
        h_\mu(Y_i,Y_j)
        =
        r_\mu^{\rm lab}(Y_i,Y_j)+a_i+a_j.
\]
Thus the off-diagonal empirical matrix generated by $h_\mu$ is
\[
        R_\mu^{\rm lab}
        +a\one^\top+\one a^\top-2\diag(a).
\]
Since $\Pi_n\one=0$, the definition of $E_\mu^{\rm lab}$ gives the exact
identity
\begin{equation}
        E_\mu^{\rm lab}
        =
        \Pi_n R_\mu^{\rm lab}\Pi_n
        -\Delta_\mu\Pi_n
        -2\Pi_n\diag(a)\Pi_n.
\label{eq:gmbm-empirical-residual-decomposition}
\end{equation}
Both correction terms have operator norm $O(1)$ and send $\xi_c$ to a vector
with $\ell_\infty$ norm $O(1)$, because $|a_i|\le C$,
$|\Delta_\mu|\le1$, and $\|\xi_c\|_\infty\le2$.  Since
$np\ge C_D\log n$, these deterministic terms are absorbed into
$C_D\sqrt{np\log n}$ after increasing $C_D$.

\proofstep{Step 2: Operator-norm concentration.}
We next bound $\|R_\mu^{\rm lab}\|_{\op}$.

\smallskip\noindent\emph{Typical sections.}
Let $m_\mu=\mu^2\sqrt d$, and let
$\mathcal G$ be the set of points $Y=(\xi,G)$ satisfying
\[
        \left|\frac{\|G\|}{\sqrt d}-1\right|
        \le C_D\sqrt{\frac{\log n}{d}}\le \frac cL,
        \qquad
        |G_1|\le C_D\sqrt{\log n}.
\]
Chi-square and Gaussian tails
\cite[Theorem~3.1.1 and Proposition~2.1.2]{VershyninHDP}, together with the
deterministic implications of
\cref{thm:gmbm-exact}, give
$\Pp\{Y\notin\mathcal G\}\le n^{-2D-12}$.  Hence all points in the original
sample, and all points on both sides of a decoupled sample, belong to
$\mathcal G$ with probability at least $1-n^{-2D-10}$.  Denote this event by
$\mathcal E_{\rm sec}$.

The assumptions of \cref{thm:gmbm-exact} also imply
$m_\mu\le c/\sqrt L$ and $\mu|G_1|\le c/\sqrt L$ on $\mathcal G$.  For any
$y=(\xi,g)\in\mathcal G$, conditioning on $Y=y$ makes the edge
statistic a mixture of two Gaussians.  Their common variance is
\[
 \left\|\frac{g}{\sqrt d}+\mu\xi e_1\right\|_2^2
 =1+O(L^{-1}),
\]
and their shifts have absolute value
$|\mu g_1+m_\mu\xi|=O(L^{-1/2})$.  Since
$\tau\asymp\sqrt L$, Mills' ratio \cite{GordonMillsRatio} gives the uniform
bound
\[
 \Pp\{\sqrt d\,\langle Z(y),Z(Y')\rangle\ge\tau\}\le Cp.
\]
The elementary inequality $(\one_E-p)^2\le2\one_E+2p^2$, the bound
$|\Delta_\mu|\le Cp$, and conditional Jensen therefore imply
\begin{equation}
        \int r_\mu^{\rm lab}(y,y')^2\,d\Pp_Y(y')\le Cp
        \qquad\text{for every }y\in\mathcal G.
\label{eq:gmbm-typical-section-bound}
\end{equation}

\smallskip\noindent\emph{Recentering after truncation.}
We also check explicitly that truncation does not destroy canonicality.  Let
$\varepsilon=\Pp\{Y\notin\mathcal G\}$ and
$\nu=\Pp_Y(\,\cdot\mid\mathcal G)$.  For $y\in\mathcal G$, set
\[
\begin{aligned}
        m_{\mathcal G}(y)
        &=\int r_\mu^{\rm lab}(y,y')\,d\nu(y'),\\
        m_{\mathcal G,0}
        &=\int m_{\mathcal G}(y)\,d\nu(y),\\
        r_{\mu,\mathcal G}^{\rm lab}(y,y')
        &=r_\mu^{\rm lab}(y,y')-m_{\mathcal G}(y)
          -m_{\mathcal G}(y')+m_{\mathcal G,0}.
\end{aligned}
\]
The last kernel is canonical under $\nu$.  Since
$r_\mu^{\rm lab}$ is canonical under the original law and its sections on
$\mathcal G$ have squared $L^2$ norm at most $Cp$, Cauchy--Schwarz gives
\begin{equation}
        \sup_{y\in\mathcal G}|m_{\mathcal G}(y)|
        +|m_{\mathcal G,0}|
        \le C\sqrt{p\varepsilon}.
\label{eq:gmbm-truncation-centering-size}
\end{equation}
If $Q_\nu$ is the orthogonal projection onto the mean-zero subspace of
$L^2(\nu)$, then
\[
        T_{r_{\mu,\mathcal G}^{\rm lab},\nu}
        =Q_\nu T_{r_\mu^{\rm lab},\nu}Q_\nu,
        \qquad
        \|T_{r_{\mu,\mathcal G}^{\rm lab},\nu}\|_{\op}
        \le\frac{\Lambda_\mu^{\rm lab}}{1-\varepsilon}.
\]
Indeed, extension by zero identifies $L^2(\nu)$ with functions supported on
$\mathcal G$ in the original $L^2$ space, and the change of normalization
costs the factor $(1-\varepsilon)^{-1}$.  The conditional covariance
calculation in \cref{lem:good-cov}, applied to the recentered canonical
kernel, is therefore bounded by
\begin{equation}
        C_D\bigl(n(\Lambda_\mu^{\rm lab})^2+p\log n\bigr).
\label{eq:gmbm-truncated-covariance-bound}
\end{equation}

\smallskip\noindent\emph{Conditional matrix concentration.}
The recentered kernel is uniformly bounded, and its squared section norm is
$O(p)$.  For each fixed second endpoint, scalar Bernstein therefore bounds
the squared column norm by $C_Dnp$.  The conditional-probability and Markov
argument of \cref{lem:column-tail} makes the squared-column bound $C_Dnp$
simultaneous for all
sampled decoupled columns.  The covariance estimate
\eqref{eq:gmbm-truncated-covariance-bound}, the matrix
Chernoff inequality in \cref{lem:matrix-chernoff}, and the
decoupling inequality in \cref{lem:decoupling} now give
\begin{equation}
        \|R_{\mu,\mathcal G}^{\rm lab}\|_{\op}
        \le
        C_D\left(\sqrt{np\log n}+n\Lambda_\mu^{\rm lab}\right)
\label{eq:gmbm-truncated-operator-bound}
\end{equation}
with probability at least $1-n^{-D-1}$, where
$R_{\mu,\mathcal G}^{\rm lab}$ is the zero-diagonal empirical matrix of the
recentered kernel.  On $\mathcal E_{\rm sec}$, the recentering correction
satisfies
\begin{equation}
        \|R_\mu^{\rm lab}-R_{\mu,\mathcal G}^{\rm lab}\|_{\op}
        \le Cn\sqrt{p\varepsilon}=o(n^{-D-2}).
\label{eq:gmbm-recentering-correction}
\end{equation}
Combining \eqref{eq:gmbm-truncated-operator-bound} with
\eqref{eq:gmbm-recentering-correction} gives the same bound
for $R_\mu^{\rm lab}$.  By
\cref{lem:gmbm-gap-explicit}, this is at most
\[
        C_D\left(\sqrt{np\log n}
        +np\mu\sqrt L+np\sqrt{\frac Ld}\right).
\]
Together with
\eqref{eq:gmbm-empirical-residual-decomposition}, this proves
\eqref{eq:gmbm-residual-operator}.

\proofstep{Step 3: Signed row sums.}
The operator estimate \eqref{eq:gmbm-residual-operator}
controls the spectrum, but the SDP dual
certificate also requires coordinatewise control in the label direction.
Since $\xi_c=\xi-\bar\xi\one$,
\[
        R_\mu^{\rm lab}\xi_c
        =
        R_\mu^{\rm lab}\xi-\bar\xi\,R_\mu^{\rm lab}\one.
\]
We control the two terms on the right separately.  First fix $i$ and
condition on $Y_i$.  The variables
\[
        X_{ij}=r_\mu^{\rm lab}(Y_i,Y_j)\xi_j,\qquad j\ne i,
\]
are independent, bounded by an absolute constant, and have conditional
variance at most $Cp$.  The variance bound follows from the section estimate
\eqref{eq:gmbm-typical-section-bound}:
\begin{equation}
        \sum_{j\ne i}\Var(X_{ij}\mid Y_i)
        \le
        \sum_{j\ne i}\E[X_{ij}^2\mid Y_i]
        \le Cnp.
\label{eq:gmbm-signed-row-variance}
\end{equation}
Their conditional mean is
$T_{r_\mu^{\rm lab}}\xi(Y_i)$.  The pointwise estimate
\eqref{eq:gmbm-pointwise-label-action} gives on
$\mathcal E_{\rm sec}$
\[
        \left|
        T_{r_\mu^{\rm lab}}\xi(Y_i)
        \right|
        \le C_Dp\mu\sqrt{L\log n}.
\]
Therefore
\[
        \left|
        \sum_{j\ne i}\E[X_{ij}\mid Y_i]
        \right|
        \le
        C_Dnp\mu\sqrt{L\log n}.
\]
Scalar Bernstein's inequality, conditional on $Y_i$, gives
\[
        \Pp\left\{
        \left|
        \sum_{j\ne i}\bigl(X_{ij}-\E[X_{ij}\mid Y_i]\bigr)
        \right|
        >
        C_D\sqrt{np\log n}
        \,\middle|\,Y_i
        \right\}
        \le n^{-D-3},
\]
where the linear Bernstein term is absorbed because $np\ge C_D\log n$.
A union bound over $i$ gives
\begin{equation}
        \max_i |(R_\mu^{\rm lab}\xi)_i|
        \le
        C_D\left(\sqrt{np\log n}
        +np\mu\sqrt{L\log n}\right)
\label{eq:gmbm-row-xi-bound}
\end{equation}
with probability at least $1-n^{-D-2}$.  The same argument with
$X_{ij}=r_\mu^{\rm lab}(Y_i,Y_j)$ gives
\begin{equation}
        \max_i |(R_\mu^{\rm lab}\one)_i|
        \le
        C_D\sqrt{np\log n},
\label{eq:gmbm-row-one-bound}
\end{equation}
with probability at least $1-n^{-D-2}$, because $r_\mu^{\rm lab}$ is
canonical and hence $T_{r_\mu^{\rm lab}}1=0$.  Hoeffding's inequality
\cite{HoeffdingInequality} for the
Rademacher labels gives
$|\bar\xi|\le C_D\sqrt{\log n/n}$ with probability at least $1-n^{-D-2}$.
Since this factor is at most one,
\eqref{eq:gmbm-row-xi-bound} and
\eqref{eq:gmbm-row-one-bound} imply
\begin{equation}
        \max_i
        \left|
        \sum_{j\ne i}
        r_\mu^{\rm lab}(Y_i,Y_j)(\xi_j-\bar\xi)
        \right|
        \le
        C_D\left(\sqrt{np\log n}
        +np\mu\sqrt{L\log n}\right).
\label{eq:gmbm-row-centered-label-bound}
\end{equation}
The estimate \eqref{eq:gmbm-row-centered-label-bound} is the
required row bound for
$R_\mu^{\rm lab}\xi_c$.  The left
multiplication by $\Pi_n$ can only subtract the average of this vector, so it
increases the $\ell_\infty$ norm by at most a factor of two.  Adding the already
controlled deterministic diagonal and $-\Delta_\mu\Pi_n$ terms in
\eqref{eq:gmbm-empirical-residual-decomposition} proves
\eqref{eq:gmbm-residual-row}.
\end{proof}

We use the following deterministic specialization of the standard SDP
dual-certificate argument; see \citet[Lemma~3]{HajekWuXuSDP}.

\begin{lemma}[Signed-Laplacian certificate]
\label{lem:signed-laplacian-certificate}
Let $C$ be a symmetric $n\times n$ matrix and define
\[
        S_C=\diag(C\one)-C.
\]
If $S_C\succeq0$ and $\ker(S_C)=\operatorname{span}\{\one\}$, then $J$ is
the unique optimizer of
\[
        \max\{\langle C,X\rangle:X\succeq0,\ \diag(X)=\one\}.
\]
In particular, these conditions hold if
\[
        C=\Delta J+H,\qquad \Delta>0,\qquad
        \|H\|_{\op}+\|H\one\|_\infty<n\Delta.
\]
\end{lemma}

\begin{proof}[Proof of \cref{thm:gmbm-exact}]

\proofstep{Step 1: Conjugation by the true labels.}
Let $D_\xi=\diag(\xi_1,\ldots,\xi_n)$.  For each feasible matrix $X$ in
\eqref{eq:gmbm-sdp}, set $Y=D_\xi XD_\xi$.  Since
$D_\xi^\top=D_\xi=D_\xi^{-1}$, the map $X\mapsto Y$ is a bijection of the
elliptope: $X\succeq0$ and $X_{ii}=1$ for every $i$ if and only if
$Y\succeq0$ and $Y_{ii}=1$ for every $i$.  Moreover, with
\[
        C_\mu:=D_\xi B_\mu D_\xi,
\]
cyclicity of the trace gives
\[
        \langle B_\mu,X\rangle
        =\langle D_\xi B_\mu D_\xi,Y\rangle
        =\langle C_\mu,Y\rangle.
\]
Thus the conjugated SDP is
\begin{equation}
        \max\left\{
        \langle C_\mu,Y\rangle:
        Y\succeq0,\ Y_{ii}=1\ \text{for every }i
        \right\}.
\label{eq:gmbm-conjugated-sdp}
\end{equation}
Finally,
$D_\xi(\xi\xi^\top)D_\xi=J$.  Hence $\xi\xi^\top$ is the unique optimizer
of \eqref{eq:gmbm-sdp} if and only if $J$ is the unique optimizer of
\eqref{eq:gmbm-conjugated-sdp}.  We prove the latter assertion using
\cref{lem:signed-laplacian-certificate}.
By \eqref{eq:gmbm-label-projection-main},
\[
        B_\mu
        =
        \Delta_\mu\xi_c\xi_c^\top+E_\mu^{\rm lab},
        \qquad
        \xi_c=\xi-\bar\xi\one,
        \qquad
        \bar\xi=\frac1n\sum_{i=1}^n\xi_i.
\]
Since
\[
        D_\xi\xi_c
        =
        \one-\bar\xi\xi,
\]
we may write
\begin{equation}
        C_\mu
        =
        \Delta_\mu qq^\top+R_\mu,
        \qquad
        q=\one-\bar\xi\xi,
        \qquad
        R_\mu=D_\xi E_\mu^{\rm lab}D_\xi.
\label{eq:gmbm-sdp-conjugated}
\end{equation}
Define the perturbation matrix
\[
        H_\mu
        :=C_\mu-\Delta_\mu J
        =\Delta_\mu(qq^\top-J)+R_\mu.
\]
Thus $C_\mu=\Delta_\mu J+H_\mu$, in the form required by
\cref{lem:signed-laplacian-certificate}.  The reference matrix
$C_0=\Delta_\mu J$ has signed Laplacian
\[
        S_{C_0}=n\Delta_\mu I_n-\Delta_\mu J.
\]
It vanishes on $\one$ and equals $n\Delta_\mu I_n$ on $\one^\perp$.

\proofstep{Step 2: Perturbation of the population certificate.}
We first control the label imbalance.  Hoeffding's inequality
\cite{HoeffdingInequality} gives
\begin{equation}
        |\bar\xi|
        \le
        C_D\sqrt{\frac{\log n}{n}}
\label{eq:gmbm-sdp-label-balance}
\end{equation}
with probability at least $1-n^{-D-2}$.  The population part of $H_\mu$ is
\[
        \Delta_\mu(qq^\top-J)
        =
        -\Delta_\mu\bar\xi
        (\one\xi^\top+\xi\one^\top)
        +\Delta_\mu\bar\xi^2\xi\xi^\top.
\]
Also $q^\top\one=n(1-\bar\xi^2)$.  These two identities give directly
\begin{equation}
\begin{aligned}
 &\|\Delta_\mu(qq^\top-J)\|_{\op}
 +\|\Delta_\mu(qq^\top-J)\one\|_\infty
 \le
 4n\Delta_\mu\bigl(|\bar\xi|+\bar\xi^2\bigr).
\end{aligned}
\label{eq:gmbm-sdp-pop-size}
\end{equation}

For the empirical part, conjugation preserves operator norm.  Moreover,
$E_\mu^{\rm lab}\one=0$, and hence
\[
 \|R_\mu\|_{\op}=\|E_\mu^{\rm lab}\|_{\op},
 \qquad
 \|R_\mu\one\|_\infty=\|E_\mu^{\rm lab}\xi_c\|_\infty.
\]
The two estimates in \cref{lem:gmbm-label-residual} now give
\begin{equation}
\begin{aligned}
 \|R_\mu\|_{\op}+\|R_\mu\one\|_\infty
 \le C_D\left(
        \sqrt{np\log n}
        +np\mu\sqrt{L\log n}
        +np\sqrt{\frac Ld}
 \right)
\label{eq:gmbm-sdp-random-size}
\end{aligned}
\end{equation}
with probability at least $1-n^{-D-1}$.

\proofstep{Step 3: Comparison with the label signal.}
By \cref{lem:gmbm-gap-explicit},
$n\Delta_\mu\ge cnp\mu^2\sqrt{dL}$.  Dividing the three terms in
\eqref{eq:gmbm-sdp-random-size} by this signal scale gives
\[
 \frac{\sqrt{\log n/(np)}}{\mu^2\sqrt{dL}},
 \qquad
 \frac{\sqrt{\log n}}{\mu\sqrt d},
 \qquad
 \frac1{\mu^2d}.
\]
The first is small by
$\mu^2\sqrt{dL}\ge C_D\sqrt{\log n/(np)}$, the second by
$\mu^2d\ge C_D\log n$, and the third by the same inequality.  After
increasing $C_D$, we obtain
\begin{equation}
        \|R_\mu\|_{\op}+\|R_\mu\one\|_\infty
        \le \frac18 n\Delta_\mu.
\label{eq:gmbm-sdp-random-small}
\end{equation}
Likewise, \eqref{eq:gmbm-sdp-label-balance} and
\eqref{eq:gmbm-sdp-pop-size}
imply, for all sufficiently large $n$,
\begin{equation}
 \|\Delta_\mu(qq^\top-J)\|_{\op}
 +\|\Delta_\mu(qq^\top-J)\one\|_\infty
 \le\frac18 n\Delta_\mu.
\label{eq:gmbm-sdp-pop-small}
\end{equation}
Combining \eqref{eq:gmbm-sdp-random-small} and
\eqref{eq:gmbm-sdp-pop-small} gives
\begin{equation}
        \|H_\mu\|_{\op}+\|H_\mu\one\|_\infty
        \le\frac14 n\Delta_\mu.
\label{eq:gmbm-sdp-total-small}
\end{equation}

\proofstep{Step 4: Apply the certificate.}
By \eqref{eq:gmbm-sdp-total-small}, the decomposition
$C_\mu=\Delta_\mu J+H_\mu$ satisfies the strict perturbation condition in
\cref{lem:signed-laplacian-certificate}.  Hence $J$ is the unique optimizer
of the conjugated SDP, and conjugating back gives
$\widehat X=\xi\xi^\top$.  The label-balance and residual events hold
simultaneously with probability at least $1-n^{-D}$.
\end{proof}

\subsection{Large-separation impossibility}
\label{sec:gmbm-impossibility-proof}

We now prove \cref{thm:gmbm-large-separation}.  The argument first calibrates
the threshold from below.  It then uses one extreme latent coordinate in
each label class to construct two isolated vertices.  A final symmetry
argument converts isolation into an information-theoretic obstruction.

\begin{proof}[Proof of \cref{thm:gmbm-large-separation}]
Let
\[
        H_i=\xi_iG_{i1},
        \qquad
        R_{ij}=\frac{\langle G_i,G_j\rangle}{\sqrt d},
        \qquad
        m_\mu=\mu^2\sqrt d.
\]
The sign change in $H_i$ absorbs the unknown label, so
$H_1,\ldots,H_n$ are again independent $N(0,1)$ variables.  The quantity
$m_\mu$ is the deterministic same-label shift, while $R_{ij}$ contains the
remaining Gaussian inner product.  With this notation the label dependence
is isolated in one sign $\xi_i\xi_j$, and, for $i\ne j$,
\begin{equation}
 \sqrt d\,\langle Z_i,Z_j\rangle
 =\xi_i\xi_j\bigl(m_\mu+\mu(H_i+H_j)\bigr)+R_{ij}.
\label{eq:gmbm-large-separation-statistic}
\end{equation}
Since $np\ge C\log n$, we have $p\ge C\log n/n$ and hence
$L\le C\log n$ for all sufficiently large $n$.  The assumption
$\mu^2d\ge C\log n$ therefore gives
\begin{equation}
        \mu\sqrt d\ge C\sqrt{\log n}\ge C\sqrt L,
\label{eq:gmbm-large-mu-d}
\end{equation}
after adjusting the absolute constant.

\proofstep{Step 1: Lower bound for the threshold.}
Choose $p_0<1/8$, and let $u_p>0$ be determined by
\[
        \Pp\{N(0,2)\ge u_p\}=4p.
\]
The classical Gaussian Mills bounds \cite{GordonMillsRatio} state that, for
$x>0$,
\begin{equation}
 \frac{x}{1+x^2}\frac{e^{-x^2/2}}{\sqrt{2\pi}}
 \le \overline\Phi(x)
 \le \frac1x\frac{e^{-x^2/2}}{\sqrt{2\pi}}.
\label{eq:gmbm-large-mills}
\end{equation}
After decreasing $p_0$ if necessary,
\eqref{eq:gmbm-large-mills} implies
\begin{equation}
        c\sqrt L\le u_p\le C\sqrt L.
\label{eq:gmbm-large-up}
\end{equation}

We first record a tail bound for the residual inner product.  For independent
standard Gaussian vectors $G_1,G_2\in\R^d$ and $|\lambda|<\sqrt d$,
\[
 \E\exp\left\{\lambda\frac{\langle G_1,G_2\rangle}{\sqrt d}\right\}
 =\left(1-\frac{\lambda^2}{d}\right)^{-d/2}.
\]
Indeed, conditioning on $G_1$ gives
\[
 \E\left[
   \exp\left\{\frac{\lambda^2\|G_1\|_2^2}{2d}\right\}
 \right]
 =\left(1-\frac{\lambda^2}{d}\right)^{-d/2}.
\]
For $|\lambda|\le\sqrt d/2$, the inequality
$-\log(1-\lambda^2/d)\le2\lambda^2/d$ bounds the moment generating function
by $e^{\lambda^2}$.  Chernoff's inequality
\cite[Section~2.3]{VershyninHDP}, first with
$\lambda=t/4$ when $t\le2\sqrt d$ and then with
$\lambda=\sqrt d/2$ when $t>2\sqrt d$, yields
\begin{equation}
        \Pp\{|R_{12}|>t\}
        \le2\exp\{-c\min(t^2,t\sqrt d)\},
        \qquad t\ge0.
\label{eq:gmbm-product-tail}
\end{equation}

Conditional on $\xi_i\xi_j=1$, the statistic in
\eqref{eq:gmbm-large-separation-statistic} is
\[
        m_\mu+\mu(H_i+H_j)+R_{ij},
        \qquad H_i+H_j\sim N(0,2).
\]
Since the unconditional edge probability is
$p=(p_+(\mu)+p_-(\mu))/2$, we have $p_+(\mu)\le2p$.  Suppose, toward a
contradiction, that
$\tau<m_\mu+\mu u_p/2$.  Then
\begin{equation}
\begin{aligned}
 p_+(\mu)
 &\ge
 \Pp\{H_i+H_j\ge u_p,\ R_{ij}\ge-\mu u_p/2
       \mid \xi_i\xi_j=1\}\\
 &\ge 4p-\Pp\{R_{ij}< -\mu u_p/2\}.
\end{aligned}
\label{eq:gmbm-pplus-contradiction}
\end{equation}
This is a union-bound estimate and does not require independence between
$H_i+H_j$ and $R_{ij}$.
By \eqref{eq:gmbm-large-up}, the lower bound on $\mu\sqrt L$, and
\eqref{eq:gmbm-large-mu-d},
\[
 \min\{(\mu u_p)^2,\mu u_p\sqrt d\}
 \ge c\min\{\mu^2L,\mu\sqrt{dL}\}
 \ge C L.
\]
Thus \eqref{eq:gmbm-product-tail} gives
$\Pp\{R_{ij}< -\mu u_p/2\}\le p$ after increasing $C$.  Substitution into
\eqref{eq:gmbm-pplus-contradiction} would imply
$p_+(\mu)\ge3p$, a contradiction.  We conclude that
\begin{equation}
        \tau\ge m_\mu+\frac12\mu u_p.
\label{eq:gmbm-threshold-lower-large-mu}
\end{equation}

\proofstep{Step 2: Extreme coordinates and residual inner products.}
Hoeffding's inequality \cite{HoeffdingInequality} gives
\begin{equation}
 \Pp\left\{\frac n3\le |\{i:\xi_i=\sigma\}|\le\frac{2n}3
       \text{ for both }\sigma\in\{\pm1\}\right\}=1-o(1).
\label{eq:gmbm-large-class-balance}
\end{equation}
On this event, choose
\[
        i_\sigma\in\argmin_{\{i:\xi_i=\sigma\}}H_i,
        \qquad
        M_\sigma=\max_{\{i:\xi_i=\sigma\}}H_i,
        \qquad \sigma\in\{\pm1\}.
\]
Conditional on the labels, the variables in either class are independent
$N(0,1)$.  Set $a_n=\sqrt{2\log n}$.  The Mills bounds
\eqref{eq:gmbm-large-mills} imply, uniformly for $n/3\le q\le2n/3$,
\begin{equation}
\begin{aligned}
 q\,\overline\Phi\left(a_n+\frac{2\log\log n}{a_n}\right)&=o(1),\\
 q\,\overline\Phi\left(a_n-\frac{2\log\log n}{a_n}\right)&\longrightarrow\infty.
\end{aligned}
\label{eq:gmbm-extreme-tail-relations}
\end{equation}
The first relation in \eqref{eq:gmbm-extreme-tail-relations}
and a union bound control the maximum in each class.  For
the minimum, symmetry and $(1-x)^q\le e^{-qx}$ show that the probability
that every observation exceeds
$-a_n+2\log\log n/a_n$ tends to zero.  Hence, with probability $1-o(1)$,
simultaneously for $\sigma\in\{\pm1\}$,
\[
        M_\sigma\le a_n+\frac{2\log\log n}{a_n},
        \qquad
        H_{i_\sigma}\le-a_n+\frac{2\log\log n}{a_n}.
\]
Since \eqref{eq:gmbm-large-up} bounds $u_p$ below by a
positive absolute constant, it follows
that, for all sufficiently large $n$,
\begin{equation}
        M_\sigma+H_{i_\sigma}
        \le\frac{4\log\log n}{a_n}
        \le\frac14u_p,
        \qquad \sigma\in\{\pm1\}.
\label{eq:gmbm-min-plus-max}
\end{equation}
A Gaussian tail bound \cite[Proposition~2.1.2]{VershyninHDP} and a union bound
also give, for an absolute constant $K>0$,
\begin{equation}
        \max_{1\le i\le n}|H_i|\le K\sqrt{\log n}
\label{eq:gmbm-all-h-small}
\end{equation}
with probability $1-o(1)$.

It remains to control all residual inner products incident to the two
selected vertices.  Write $G_i=(G_{i1},W_i)$ with
$W_i\in\R^{d-1}$, with the evident interpretation for $d=1$.  The indices
$i_+$ and $i_-$ depend only on the labels and first coordinates, and hence
are independent of $W_1,\ldots,W_n$.  Gaussian norm concentration
\cite[Theorem~3.1.1]{VershyninHDP} gives
\[
        \max_{\sigma\in\{\pm1\}}\|W_{i_\sigma}\|_2
        \le C(\sqrt d+\sqrt{\log n})
\]
with probability $1-o(1)$.  Conditional on a selected $W_{i_\sigma}$, each
$\langle W_{i_\sigma},W_j\rangle/\sqrt d$ is centered Gaussian with
variance $\|W_{i_\sigma}\|_2^2/d$.  The same conditioning applies to the
inner product between the two selected vectors because their remaining
coordinates are independent.  A Gaussian tail bound
\cite[Proposition~2.1.2]{VershyninHDP} and a union bound over
the two candidates and all $j$ therefore give
\[
 \max_{\sigma\in\{\pm1\}}\max_{j\ne i_\sigma}
 \frac{|\langle W_{i_\sigma},W_j\rangle|}{\sqrt d}
 \le K\left(\sqrt{\log n}+\frac{\log n}{\sqrt d}\right)
\]
with probability $1-o(1)$.  On the event
\eqref{eq:gmbm-all-h-small}, the omitted first-coordinate product is at most
$K^2\log n/\sqrt d$.  Increasing $K$ and writing
\[
        r_n:=\sqrt{\log n}+\frac{\log n}{\sqrt d},
\]
we obtain
\begin{equation}
 \max_{\sigma\in\{\pm1\}}\max_{j\ne i_\sigma}|R_{i_\sigma j}|
 \le Kr_n
\label{eq:gmbm-candidate-r-small}
\end{equation}
with probability $1-o(1)$.

We finally verify the scale comparisons needed in the next step.  By
\eqref{eq:gmbm-large-up} and the lower bound on $\mu\sqrt L$,
\begin{equation}
        \mu u_p\ge8Kr_n
\label{eq:gmbm-mu-u-large}
\end{equation}
after increasing $C$.  Moreover,
\[
 m_\mu=\mu(\mu\sqrt d)\ge C\mu\sqrt{\log n}
\]
by \eqref{eq:gmbm-large-mu-d}, while
\[
 m_\mu
 =\frac{(\mu\sqrt d)(\mu\sqrt L)}{\sqrt L}
 \ge Cr_n.
\]
Increasing $C$ once more gives
\begin{equation}
        m_\mu\ge4K\mu\sqrt{\log n}+4Kr_n.
\label{eq:gmbm-m-large}
\end{equation}

\proofstep{Step 3: Isolated vertices and posterior symmetry.}
Fix $\sigma\in\{\pm1\}$.  We show that $i_\sigma$ has no neighbors.  If
$j\ne i_\sigma$ has label $\sigma$, then
\eqref{eq:gmbm-large-separation-statistic},
\eqref{eq:gmbm-min-plus-max}, and
\eqref{eq:gmbm-candidate-r-small} give
\[
 \sqrt d\,\langle Z_{i_\sigma},Z_j\rangle
 \le m_\mu+\frac14\mu u_p+Kr_n
 <m_\mu+\frac12\mu u_p
 \le\tau,
\]
where the strict inequality uses \eqref{eq:gmbm-mu-u-large}, and the last
inequality is \eqref{eq:gmbm-threshold-lower-large-mu}.  If $j$ has the
opposite label, then \eqref{eq:gmbm-all-h-small},
\eqref{eq:gmbm-candidate-r-small}, and \eqref{eq:gmbm-m-large} imply
\[
 \sqrt d\,\langle Z_{i_\sigma},Z_j\rangle
 \le -m_\mu+2K\mu\sqrt{\log n}+Kr_n
 <0<m_\mu+\frac12\mu u_p\le\tau.
\]
Thus $i_\sigma$ is isolated.  Applying the argument to both labels and
intersecting the events in
\eqref{eq:gmbm-large-class-balance},
\eqref{eq:gmbm-min-plus-max},
\eqref{eq:gmbm-all-h-small}, and
\eqref{eq:gmbm-candidate-r-small} proves
\[
 \Pp\{A^\mu\text{ has an isolated vertex in each label class}\}
 \longrightarrow1.
\]

It remains only to formalize the exchangeability argument.  Fix an adjacency
matrix $a$ with $\Pp\{A^\mu=a\}>0$, and let $\mathcal I(a)$ be its isolated
vertices.  Every
permutation $\pi$ supported on $\mathcal I(a)$ leaves $a$ unchanged.  Since
the joint law of $(A^\mu,\xi)$ is vertex-exchangeable,
\[
 \Pp\{\xi=x\mid A^\mu=a\}
 =\Pp\{\xi=\pi x\mid A^\mu=a\}
 \qquad\text{for every }x\in\{\pm1\}^n.
\]
Equivalently, conditional on $A^\mu=a$, and after further conditioning on the
labels outside $\mathcal I(a)$ and the number of positive labels in
$\mathcal I(a)$, all assignments of those labels to the isolated vertices are
equally likely.

Let $\mathcal E$ be the event that both signs occur among the isolated
vertices.  On $\mathcal E$, there are at least two such assignments.  If at
least one vertex is nonisolated, at most one of them can equal a fixed output
up to a global sign, because the labels outside $\mathcal I(a)$ are fixed.
If every vertex is isolated, at most two assignments can equal that output
up to sign, whereas the number of assignments is at least $n$, which is at
least $4$ for all sufficiently large $n$.  Thus, after also conditioning on
the independent random seed of a possibly randomized estimator
$\widetilde\xi$,
\[
 \Pp\{\widetilde\xi\in\{\xi,-\xi\}\mid A^\mu=a\}
 \le 1-\frac12\Pp\{\mathcal E\mid A^\mu=a\}.
\]
Averaging over $A^\mu$ and using $\Pp(\mathcal E)=1-o(1)$ yields, uniformly
over all such estimators,
\[
        \Pp\{\widetilde\xi\in\{\xi,-\xi\}\}
        \le\frac12+o(1).
\]
This completes the proof.
\end{proof}

\acks{During the preparation of this manuscript, the authors used OpenAI's
GPT-5.6 to assist with language editing, improving exposition, and checking
routine calculations.  Y.Z. was partially supported by the Simons Grant
MPS-TSM-00013944.}

\appendix

\section{Spherical harmonic estimates}\label{sec:spherical-harmonics}

This appendix supplies the harmonic estimates used in the proofs of
\cref{thm:main-concentration,thm:spherical-inner-product}.  We recall the
Funk--Hecke diagonalization and then bound the first and higher coefficients
of a sparse spherical cap.

\subsection{Funk--Hecke diagonalization}\label{sec:funk-hecke}

We use standard notation for spherical harmonics and Gegenbauer polynomials;
see Dai and Xu \citeyearpar[Chapters~1--2]{DaiXu}.  For $\ell\ge0$, a
\emph{spherical harmonic of degree $\ell$} is the restriction to $\Sd$ of a
homogeneous harmonic polynomial of degree $\ell$: the polynomial is
homogeneous in the sense that $Q(rx)=r^\ell Q(x)$ and harmonic in the sense
that its Euclidean Laplacian vanishes.  Thus the space of degree-$\ell$
spherical harmonics is
\[
 \cH_\ell^d
 =
 \left\{
 \left.Q\right|_{\Sd}:
 Q\in\R[x_1,\ldots,x_d],\quad
 Q(rx)=r^\ell Q(x),\quad \Delta_{\R^d}Q=0
 \right\}.
\]
Throughout this appendix, $d\ge3$.  Set $\alpha=(d-2)/2$.  The Gegenbauer
polynomials $C_\ell^{(\alpha)}$ are defined by the generating function
\[
 (1-2tz+z^2)^{-\alpha}
 =\sum_{\ell=0}^{\infty}C_\ell^{(\alpha)}(t)z^\ell,
 \qquad |z|<1.
\]
Since
$C_\ell^{(\alpha)}(1)=\binom{\ell+d-3}{\ell}$, the normalized Gegenbauer
polynomial is
\[
        P_\ell^{(d)}(t)
        =\frac{C_\ell^{(\alpha)}(t)}{C_\ell^{(\alpha)}(1)},
        \qquad P_\ell^{(d)}(1)=1.
\]
Equivalently, $P_\ell^{(d)}$ is the unique degree-$\ell$ polynomial satisfying
\begin{equation}
 (1-t^2)(P_\ell^{(d)})''(t)
 -(d-1)t(P_\ell^{(d)})'(t)
 +\ell(\ell+d-2)P_\ell^{(d)}(t)=0,
 \qquad P_\ell^{(d)}(1)=1.
\label{eq:zonal-polynomial-ode}
\end{equation}
It is the normalized zonal spherical harmonic of degree $\ell$.  Indeed, if
$\{Y_{\ell,m}\}_{m=1}^{\dim(\cH_\ell^d)}$ is any orthonormal basis of
$\cH_\ell^d$ in $L^2(\sigma)$, the addition formula gives
\[
 \sum_{m=1}^{\dim(\cH_\ell^d)}Y_{\ell,m}(u)Y_{\ell,m}(v)
 =\dim(\cH_\ell^d)P_\ell^{(d)}(\langle u,v\rangle),
 \qquad u,v\in\Sd.
\]
Thus a rotationally invariant kernel, which depends on $u$ and $v$ only
through $\langle u,v\rangle$, acts diagonally on the spaces
$\cH_\ell^d$.
We write
\begin{equation}
        w_d(t)=c_d(1-t^2)^{(d-3)/2}\one_{(-1,1)}(t),
        \qquad
        c_d=\frac{\Gamma(d/2)}{\sqrt{\pi}\Gamma((d-1)/2)},
\label{eq:wd-def}
\end{equation}
for the density of one coordinate of a uniform point on $\Sd$.  Recall the
cap kernel $h$ and its operator $T_h$ from \cref{sec:matrix-tools}.  By the
Funk--Hecke formula \cite[Theorem~1.2.9]{DaiXu}, $T_h$ acts on
$\cH_\ell^d$ by the scalar
\[
        \lambda_\ell=\int_\tau^1P_\ell^{(d)}(t)w_d(t)\,dt,
        \qquad \ell\ge1.
\]
The degree-zero coefficient vanishes by \eqref{eq:h-canonical-main}.  Since
$L^2(\Sd)=\bigoplus_{\ell\ge0}\cH_\ell^d$ orthogonally
\cite[Chapter~1]{DaiXu},
\[
        \|T_h\|_{\op}=\sup_{\ell\ge1}|\lambda_\ell|,
        \qquad
        \lambda_1=\int_\tau^1t\,w_d(t)\,dt,
        \qquad
        \Lambda_{\ge2}=\sup_{\ell\ge2}|\lambda_\ell|.
\]

\subsection{The sparse cap operator}\label{sec:sparse-cap-operator}

The first estimate controls every nonconstant coefficient at the scale
$p\tau$.  We then identify the degree-one signal and obtain the sharper
higher-degree bound needed for recovery.

\begin{lemma}[Spherical Mills estimate]\label{lem:spherical-mills}
There are absolute constants $c,C>0$ and $p_0>0$ such that, for
$0<p\le p_0$ and $d\ge3$,
\[
        c\,\frac{c_d}{d\tau}(1-\tau^2)^{(d-1)/2}
        \le
        p
        \le
        C\,\frac{c_d}{d\tau}(1-\tau^2)^{(d-1)/2}.
\]
In particular,
\[
        \frac{c_d}{d}(1-\tau^2)^{(d-1)/2}
        \le C p\tau.
\]
\end{lemma}

\begin{proof}
We first locate the cap threshold uniformly in $d$.  There are absolute
constants $a,q_0>0$ such that
\[
        \Pp\{X_1\ge a/\sqrt d\}\ge q_0,
        \qquad X\sim\sigma,\quad d\ge3.
\]
Indeed, on $[a/\sqrt d,2a/\sqrt d]$ the density $w_d$ is bounded below by
$c\sqrt d$ when $a$ is sufficiently small; here we use
$c_d\asymp\sqrt d$.  After choosing $p_0<q_0$, the identity
$p=\Pp\{X_1\ge\tau\}$ therefore implies
$\tau\ge a/\sqrt d$.  In particular,
\begin{equation}
        \frac{d\tau}{1+\tau}\ge c.
\label{eq:spherical-mills-threshold-location}
\end{equation}

Put
\[
        \delta=\frac{c_0(1-\tau^2)}{d\tau},
\]
where $c_0>0$ is a sufficiently small absolute constant.  By
\eqref{eq:spherical-mills-threshold-location},
$\delta/(1-\tau)=c_0(1+\tau)/(d\tau)\le1$, so
$[\tau,\tau+\delta]\subset[\tau,1]$.  Moreover, for every point in this
interval,
\[
 \frac{t^2-\tau^2}{1-\tau^2}
 \le \frac{\delta(2\tau+\delta)}{1-\tau^2}
 \le \frac{Cc_0}{d}.
\]
Consequently,
\[
 \left(\frac{1-t^2}{1-\tau^2}\right)^{(d-3)/2}
 \ge c,
 \qquad \tau\le t\le\tau+\delta.
\]
Integrating over this interval gives
\[
        \int_\tau^1(1-t^2)^{(d-3)/2}\,dt
        \ge
        c\frac{(1-\tau^2)^{(d-1)/2}}{d\tau}.
\]
For the upper bound, use $t\ge\tau$ on $[\tau,1]$:
\[
\begin{aligned}
        \int_\tau^1(1-t^2)^{(d-3)/2}\,dt
        &\le
        \frac1\tau
        \int_\tau^1 t(1-t^2)^{(d-3)/2}\,dt        \\
        &=
        \frac{(1-\tau^2)^{(d-1)/2}}{(d-1)\tau}
        \le
        C\frac{(1-\tau^2)^{(d-1)/2}}{d\tau}.
\end{aligned}
\]
Multiplication by $c_d$ proves the two-sided estimate.
\end{proof}

\begin{proposition}[Sparse cap operator]\label{prop:cap-operator}
There are absolute constants $C>0$ and $p_0>0$ such that, for every
$0<p\le p_0$ and every
$d\ge3$,
\[
        \Lambda=\sup_{\ell\ge1}|\lambda_\ell|
        \le C p\tau.
\]
\end{proposition}

\begin{proof}
The differential equation \eqref{eq:zonal-polynomial-ode}
gives
\[
        (1-t^2)y''(t)-(d-1)t\,y'(t)+\ell(\ell+d-2)y(t)=0,
\]
for $y=P_\ell^{(d)}$.  Multiplying by $(1-t^2)^{(d-3)/2}$ gives
\[
\frac{d}{dt}
\left[
(1-t^2)^{(d-1)/2}(P_\ell^{(d)})'(t)
\right]
=
-\ell(\ell+d-2)(1-t^2)^{(d-3)/2}P_\ell^{(d)}(t).
\]
Integrating from $\tau$ to $1$, the upper boundary term vanishes because
$(1-t^2)^{(d-1)/2}=0$ at $t=1$.  Hence
\begin{equation}
        \lambda_\ell
        =
        \frac{c_d(1-\tau^2)^{(d-1)/2}}
        {\ell(\ell+d-2)}
        (P_\ell^{(d)})'(\tau).
\label{eq:lambda-derivative-identity}
\end{equation}
The normalized derivative identity
\cite[Appendix~B.2]{DaiXu} and the addition formula give
\[
        (P_\ell^{(d)})'(t)
        =
        \frac{\ell(\ell+d-2)}{d-1}P_{\ell-1}^{(d+2)}(t),
        \qquad
        |P_{\ell-1}^{(d+2)}(t)|\le1.
\]
Substitution into \eqref{eq:lambda-derivative-identity} yields
\[
        |\lambda_\ell|
        \le
        \frac{c_d}{d-1}(1-\tau^2)^{(d-1)/2}.
\]
By \Cref{lem:spherical-mills}, this implies $|\lambda_\ell|\le Cp\tau$.
\end{proof}

\begin{lemma}[First harmonic]\label{lem:first-harmonic}
For every $d\ge3$,
\begin{equation}
        \lambda_1
        =
        \int_\tau^1 t\,w_d(t)\,dt
        =
        \frac{c_d}{d-1}(1-\tau^2)^{(d-1)/2}.
\label{eq:lambda1-exact}
\end{equation}
There are absolute constants $c,C>0$ and $p_0>0$ such that, whenever
$0<p\le p_0$,
\begin{equation}
        c p\tau\le \lambda_1\le C p\tau.
\label{eq:lambda1-ptau}
\end{equation}
The degree-one component of $h$ is
\[
        h_1(u,v)=d\lambda_1\langle u,v\rangle.
\]
\end{lemma}

\begin{proof}
The identity follows from $P_1^{(d)}(t)=t$ and
\[
        \int t(1-t^2)^{(d-3)/2}\,dt
        =
        -\frac1{d-1}(1-t^2)^{(d-1)/2}.
\]
The comparison \eqref{eq:lambda1-ptau} follows from the exact formula and
\cref{lem:spherical-mills}.  Finally, the degree-one reproducing
kernel is $K_1(u,v)=d\langle u,v\rangle$; see \citet[Chapter~1]{DaiXu}.
\end{proof}

\begin{lemma}[Higher cap harmonics]\label{lem:higher-harmonics}
Let $L=1+\log(1/p)$.  There are absolute constants $p_0,c,C>0$ such that, if
$0<p\le p_0$ and $d\ge C L$, then
\[
        c\sqrt{\frac Ld}\le \tau\le C\sqrt{\frac Ld}
\]
and
\[
        \Lambda_{\ge2}
        =
        \sup_{\ell\ge2}|\lambda_\ell|
        \le
        C p\frac Ld.
\]
\end{lemma}

\begin{proof}
The two-sided estimate for $\tau$ follows from
\cref{lem:spherical-mills}.  Indeed, using $c_d\asymp\sqrt d$ and taking
logarithms in
\[
        p\asymp
        \frac{c_d}{d\tau}(1-\tau^2)^{(d-1)/2}
\]
gives
\[
        L\asymp 1+d\tau^2,
\]
after increasing the constant in the assumption $d\ge C L$ to absorb the
polynomial prefactor.  Hence $\tau\asymp\sqrt{L/d}$.

After increasing $C$, we may assume
$\tau\le c_0$ for a sufficiently small absolute constant
$c_0$.  We use the refined ultraspherical derivative
bound
\begin{equation}
        \sup_{\ell\ge2}
        \frac{|(P_\ell^{(d)})'(t)|}{\ell(\ell+d-2)}
        \le
        \frac{C}{d}\left(|t|+\frac1d\right),
        \qquad |t|\le c_0.
\label{eq:refined-ultraspherical-derivative}
\end{equation}
Indeed, the normalized derivative identity gives
\[
        (P_\ell^{(d)})'(t)
        =
        \frac{\ell(\ell+d-2)}{d-1}P_{\ell-1}^{(d+2)}(t).
\]
It remains to bound the normalized polynomial on the right uniformly in its
degree.  We use the Laplace integral representation; see
\citet[Chapter~VII]{Szego}.  If $q\ge3$ and $Z$ is the first coordinate of a
uniform point on $\mathbb S^{q-2}$, then
\[
        P_m^{(q)}(t)
        =\E\left(t+i\sqrt{1-t^2}\,Z\right)^m.
\]
For $m=1$, this equals $t$.  For $m\ge2$, the modulus inside the expectation
is at most one, and therefore
\begin{align}
 |P_m^{(q)}(t)|
 &\le \E\left|t+i\sqrt{1-t^2}\,Z\right|^m \\
 &\le \E\left|t+i\sqrt{1-t^2}\,Z\right|^2
 =t^2+\frac{1-t^2}{q-1}
 \le C\left(|t|+\frac1q\right)
 \label{eq:normalized-zonal-central-bound}
\end{align}
for $|t|\le c_0$.  Applying
\eqref{eq:normalized-zonal-central-bound} with $q=d+2$ and $m=\ell-1$
proves \eqref{eq:refined-ultraspherical-derivative}.

Combining \eqref{eq:lambda-derivative-identity} with
\eqref{eq:refined-ultraspherical-derivative} at $t=\tau$ gives, for every
$\ell\ge2$,
\[
        |\lambda_\ell|
        \le
        \frac{C c_d}{d}(1-\tau^2)^{(d-1)/2}
        \left(\tau+\frac1d\right).
\]
By \cref{lem:spherical-mills}, the prefactor
$c_d d^{-1}(1-\tau^2)^{(d-1)/2}$ is at most $C p\tau$.  Hence
\[
        |\lambda_\ell|
        \le
        C p\tau\left(\tau+\frac1d\right)
        \le
        C p\frac{L}{d},
\]
where the last step uses $\tau^2\asymp L/d$ and $L\ge1$.  Taking the supremum
over $\ell\ge2$ proves the claim.
\end{proof}

\section{Gaussian kernel estimates}
\label{sec:gaussian-kernel-estimates}

This appendix supplies the three population inputs used for the Gaussian
vector model: the threshold and degree-one coefficient, the canonical
operator norms, and the second moments of typical kernel sections.

\begin{lemma}[Gaussian vector threshold and coefficients]\label{lem:raw-gaussian-coeff}
Let $L=1+\log(1/p)$.  There are absolute constants $p_0,c,C>0$ such that, if
$0<p\le p_0$ and $d\ge C L^3$, then
\[
        c\sqrt L\le u\le C\sqrt L,
        \qquad
        c p\sqrt{\frac Ld}\le \beta\le C p\sqrt{\frac Ld}.
\]
Moreover, for the integral operators $T_r$ and $T_{r_\perp}$ on
$L^2(\gamma_d)$,
\begin{equation}
        \|T_r\|_{\op}
        \le
        C p\left(\sqrt{\frac Ld}+\frac{L^2}d\right),
        \qquad
        \|T_{r_\perp}\|_{\op}
        \le
        C p\frac{L^2}d.
\label{eq:raw-operator-bounds}
\end{equation}
Finally,
\[
        \sup_{\left|\|x\|/\sqrt d-1\right|\le c/L}
        \int r(x,y)^2\,d\gamma_d(y)\le Cp,
        \qquad
        \sup_{\left|\|x\|/\sqrt d-1\right|\le c/L}
        \int r_\perp(x,y)^2\,d\gamma_d(y)\le Cp.
\]
\end{lemma}

\begin{proof}
\proofstep{Step 1: Threshold and degree-one coefficient.}
Conditioning on $G$ gives
\[
        S\,\big|\,G
        \sim
        N\left(0,\frac{\|G\|^2}{d}\right).
\]
The chi-square concentration inequality
\cite[Theorem~3.1.1]{VershyninHDP} gives
$\Pp\{|\|G\|/\sqrt d-1|>t\}\le 2e^{-cdt^2}$ for $0<t<1$.  Combining this
with the Gaussian tail bounds \cite[Proposition~2.1.2]{VershyninHDP} gives a
quantitative comparison.  If $R=\|G\|/\sqrt d$ and
$|R-1|\le c/L$, then, uniformly for $t\asymp\sqrt L$,
\begin{equation}
        \overline\Phi(t/R)\asymp\overline\Phi(t).
\label{eq:gaussian-tail-comparison}
\end{equation}
The exceptional probability is at most $2e^{-cd/L^2}=o(p)$ because
$d\ge CL^3$.  Since
$\Pp\{S\ge t\}=\E\overline\Phi(t/R)$, evaluating
\eqref{eq:gaussian-tail-comparison} at two
sufficiently small and large constant multiples of $\sqrt L$ brackets the
upper $p$-quantile and proves $u\asymp\sqrt L$.

The same conditioning gives
\[
        \E\left[S\one_{\{S\ge u\}}\mid G\right]
        =
        \frac{\|G\|}{\sqrt d}
        \varphi\left(\frac{u\sqrt d}{\|G\|}\right),
\]
where $\varphi$ is the standard Gaussian density.  On
$|R-1|\le c/L$, the integrand is comparable to $\varphi(u)$; the exceptional
part is $o(p\sqrt L)$.  Mills' ratio \cite{GordonMillsRatio} gives
$\varphi(u)\asymp pu\asymp p\sqrt L$.  Thus
$m_1\asymp p\sqrt L$, which, by \eqref{eq:raw-beta-def},
proves the bounds on $\beta$ in \cref{lem:raw-gaussian-coeff}.

\proofstep{Step 2: Radial harmonic decomposition.}
For $G\sim\gamma_d$, write
\[
        R:=\frac{\|G\|}{\sqrt d},
        \qquad
        U:=\frac{G}{\|G\|}.
\]
Then $G=\sqrt d\,RU$, where $U\sim\sigma$ is independent of $R$; let
$\nu_d$ be the law of $R$.  The spherical-harmonic decomposition from
\cref{sec:funk-hecke} gives
\[
        L^2(\gamma_d)
        =\bigoplus_{\ell\ge0}L^2(\nu_d)\otimes\cH_\ell^d.
\]
The operator preserves each summand and acts there through an integral
operator on the radial factor.

With $P_0^{(d)}\equiv1$, set, for $\ell\ge0$,
\[
        F_\ell(q)
        =\int_{u/(\sqrt d q)}^1P_\ell^{(d)}(t)w_d(t)\,dt,
\]
with the integral interpreted as zero if its lower endpoint exceeds one.
The Funk--Hecke formula \cite[Theorem~1.2.9]{DaiXu} shows that the radial
kernel in degree $\ell$ is $K_\ell(r,s)=F_\ell(rs)$.  If
$Qf=f-\int f\,d\nu_d$, then the degree-zero block of $T_r$ is $QK_0Q$;
the positive-degree blocks are unchanged.

We record the derivative bounds needed below.  Let
$\mathcal I=[1-c/L,1+c/L]$.  Gaussian norm concentration
\cite[Theorem~3.1.1]{VershyninHDP} gives
\begin{equation}
        \Pp\{R\notin\mathcal I\}\le2e^{-cd/L^2},
        \qquad
        \|R-1\|_{L^k(\nu_d)}\le \frac{C_k}{\sqrt d},
        \quad 2\le k\le6.
\label{eq:gaussian-radial-moments}
\end{equation}
For $r,s\in\mathcal I$, the shifted threshold
$u/(rs)$ differs from $u$ by $O(L^{-1/2})$.  To justify the derivative
bounds below, put $z_q=u/(\sqrt d q)$.  Direct differentiation gives
\[
        F_0'(q)=\frac{u}{\sqrt d\,q^2}w_d(z_q),
        \qquad
        \frac{w_d'(z)}{w_d(z)}=-\frac{(d-3)z}{1-z^2}.
\]
For $q=rs$ with $r,s\in\mathcal I$, we have $dz_q/dq=-z_q/q$ and
$dz_q^2=O(L/d)$.  It follows by two further differentiations that
$|F_0^{(j)}(q)|\le C_jpL^j$ for $1\le j\le3$; the case $j=0$ follows from
the Gaussian tail comparison.  Similarly, the explicit formula
\[
        F_1(q)=\frac{c_d}{d-1}
        \left(1-\frac{u^2}{dq^2}\right)^{(d-1)/2},
\]
has logarithmic derivative
\[
        \frac{F_1'(q)}{F_1(q)}
        =\frac{(d-1)z_q^2}{q(1-z_q^2)}=O(L).
\]
Differentiating this identity once more gives
$|F_1''(q)|\le C L^2|F_1(q)|$.  Since Gaussian Mills bounds give
$F_1(q)\asymp p\sqrt{L/d}$ uniformly in this interval, we obtain
\begin{align}
        |F_0^{(j)}(q)|&\le C_jpL^j,
        &&0\le j\le3,\label{eq:gaussian-F0-derivatives}\\
        |F_1^{(j)}(q)|&\le C_jp\sqrt{\frac Ld}\,L^j,
        &&0\le j\le2,\label{eq:gaussian-F1-derivatives}
\end{align}
whenever $q=rs$ with $r,s\in\mathcal I$.  Uniformly on the same set,
\cref{lem:higher-harmonics}, applied at the shifted threshold, gives
\begin{equation}
        \sup_{\ell\ge2}|F_\ell(rs)|\le Cp\frac Ld.
\label{eq:gaussian-higher-radial-blocks}
\end{equation}
The complement of $\mathcal I^2$ contributes at most
$Ce^{-cd/(2L^2)}$ in Hilbert--Schmidt norm.  Since $d\ge CL^3$, increasing
$C$ makes this quantity at most $CpL/d$; thus it is absorbed in each of the
block estimates below.

\proofstep{Step 3: Bounds for the radial blocks.}
Put $\delta_r=r-1$ and $\delta_s=s-1$.  Since
$rs-1=\delta_r+\delta_s+\delta_r\delta_s$, Taylor's formula at $1$ gives
\[
 F_0(rs)=F_0(1)+F_0'(1)(rs-1)
 +\frac12F_0''(1)(rs-1)^2+\mathcal R_3(r,s).
\]
Applying $Q$ in both variables removes the constant and one-variable terms.
The surviving part of the linear term is
$F_0'(1)Q\delta_r\,Q\delta_s$.  Since $Q$ is an orthogonal projection,
\[
 \|Q_rQ_s[(rs-1)^2]\|_2
 \le \|(rs-1)^2\|_2
 =\|rs-1\|_4^2\le \frac Cd.
\]
Here $Q_r$ and $Q_s$ mean that $Q$ acts in the $r$ and $s$ variables,
respectively.
Thus the linear and quadratic contributions have respective norms at most
\[
 |F_0'(1)|\|Q\delta_r\|_2\|Q\delta_s\|_2
 \le \frac{CpL}{d},
 \qquad
 \frac12|F_0''(1)|\|Q_rQ_s[(rs-1)^2]\|_2
 \le \frac{CpL^2}{d}.
\]
The integral remainder satisfies
\[
        \|\mathcal R_3\|_2
        \le CpL^3\|rs-1\|_6^3
        \le Cp\frac{L^3}{d^{3/2}}.
\]
By \cref{eq:gaussian-radial-moments,eq:gaussian-F0-derivatives} and
$d\ge CL^2$, these contributions sum to at most $CpL^2/d$.  Hence
\begin{equation}
        \|QK_0Q\|_{\op}
        \le\|QK_0Q\|_{\HS}
        \le Cp\frac{L^2}{d}.
\label{eq:gaussian-radial-zero-block}
\end{equation}

For degree one,
\cref{eq:gaussian-F1-derivatives,eq:gaussian-radial-moments} yield
\[
        \|K_1-F_1(1)(\one\otimes\one)\|_{\HS}
        \le Cp\frac{L^{3/2}}d.
\]
Write $E=K_1-F_1(1)(\one\otimes\one)$.  Since
$\|r\|_2=\|\one\|_2=1$ and
$1-\langle r,\one\rangle^2=\Var(r)\le C/d$,
\[
 \left|\beta-F_1(1)\right|
 \le \|E\|_{\op}+C\frac{|F_1(1)|}{d}.
\]
The normalized radial function $r\mapsto r$ has $L^2(\nu_d)$ norm one, and
the coefficient of the Gaussian coordinate functions is
\[
        \beta=\langle r,K_1r\rangle_{L^2(\nu_d)}.
\]
Moreover,
\[
 \|r\otimes r-\one\otimes\one\|_{\op}
 \le \|(r-\one)\otimes r\|_{\op}
     +\|\one\otimes(r-\one)\|_{\op}
 \le \frac C{\sqrt d}.
\]
Combining the last three displays, using
$|F_1(1)|+|\beta|\le Cp\sqrt{L/d}$, gives
\begin{equation}
        \|K_1-\beta(r\otimes r)\|_{\op}
        \le Cp\frac{L^{3/2}}d.
\label{eq:gaussian-radial-one-residual}
\end{equation}
Finally, apply the elementary Schur bound to the restriction of $K_\ell$ to
$\mathcal I^2$ and add the Hilbert--Schmidt estimate for its complement from
the paragraph following \eqref{eq:gaussian-higher-radial-blocks}.  Since
$\nu_d$ is a probability measure, this gives
\begin{equation}
        \sup_{\ell\ge2}\|K_\ell\|_{\op}
        \le Cp\frac Ld.
\label{eq:gaussian-radial-higher-op}
\end{equation}
Taking the supremum over the mutually orthogonal blocks in
\cref{eq:gaussian-radial-zero-block,eq:gaussian-radial-one-residual,eq:gaussian-radial-higher-op}
proves
\[
        \|T_{r_\perp}\|_{\op}\le Cp\frac{L^2}{d}.
\]
Adding the rank-$d$ first-chaos operator, whose norm is
$\beta\asymp p\sqrt{L/d}$, proves the bound for $T_r$ in
\eqref{eq:raw-operator-bounds}.

\proofstep{Step 4: Section bounds.}
If
$|\|x\|/\sqrt d-1|\le c/L$, then
$u\sqrt d/\|x\|=u+O(1/\sqrt L)$ and the Gaussian tail comparison in
\eqref{eq:gaussian-tail-comparison} gives
\[
        \Pp\{\langle x,Y\rangle/\sqrt d\ge u\}\le Cp.
\]
The $L^2$ section bounds now follow from
$q_{\rm G}^2\le 2\one_{\{\langle x,Y\rangle/\sqrt d\ge u\}}+2p^2$, and the
definitions of $r$ and $r_\perp$.  Indeed, conditional Jensen gives
$a(x)^2\le\E_Yq_{\rm G}(x,Y)^2\le Cp$ and
$\E a(Y)^2\le\E q_{\rm G}(G,Y)^2\le Cp$, while
\[
        \beta^2\E\langle x,Y\rangle^2
        =\beta^2\|x\|^2
        \le Cp^2L\le Cp
\]
on the stated radial set.  Expanding the two definitions and increasing the
constant proves both claims.
\end{proof}

\section{Population estimates for the Gaussian mixture block model}
\label{sec:gmbm-auxiliary-estimates}

Throughout this appendix,
\[
        L=1+\log(1/p),\qquad m_\mu=\mu^2\sqrt d,
\]
and $\varphi$ and $\overline\Phi$ denote the standard Gaussian density and
upper-tail function.  The proofs use two direct reductions: conditioning on
one Gaussian vector turns the threshold statistic into a one-dimensional
Gaussian, while an angular harmonic decomposition reduces the residual
operator estimate to spherical-cap bounds and a scalar Taylor expansion.
We recall the population notation from
\cref{sec:gmbm-exact-proof}.  Let $Y=(\xi,G)$ and $Y'=(\xi',G')$ be
independent, where $\xi,\xi'$ are uniform on $\{\pm1\}$ and
$G,G'\sim N(0,I_d)$, and put $Z(Y)=G/\sqrt d+\mu\xi e_1$.  The conditional
edge probabilities and their half-difference are
\[
 p_\pm(\mu)=\Pp\{\sqrt d\langle Z(Y),Z(Y')\rangle\ge\tau
                   \mid \xi\xi'=\pm1\},
 \qquad
 \Delta_\mu=\frac{p_+(\mu)-p_-(\mu)}2.
\]
Finally, define
\[
\begin{aligned}
 h_\mu(Y,Y')
 &=\one_{\{\sqrt d\langle Z(Y),Z(Y')\rangle\ge\tau\}}
   -p-\Delta_\mu\xi\xi',\\
 a_\mu(Y)&=\E_{Y'}h_\mu(Y,Y'),\\
 r_\mu^{\rm lab}(Y,Y')&=h_\mu(Y,Y')-a_\mu(Y)-a_\mu(Y'),
 \qquad
 \Lambda_\mu^{\rm lab}=\|T_{r_\mu^{\rm lab}}\|_{\op}.
\end{aligned}
\]

\begin{lemma}[Population label signal and residual]\label{lem:gmbm-gap-explicit}
Under the hypotheses of \cref{thm:gmbm-exact},
\begin{equation}
        c\,p\mu^2\sqrt{dL}
        \le
        \Delta_\mu
        \le
        C\,p\mu^2\sqrt{dL},
        \qquad
        \Lambda_\mu^{\rm lab}
        \le C p\left(\mu\sqrt L+\sqrt{\frac Ld}\right).
\label{eq:gmbm-population-signal-operator}
\end{equation}
In addition, if $y=(\xi,g)$ satisfies
\[
        \left|\frac{\|g\|}{\sqrt d}-1\right|
        \le C_D\sqrt{\frac{\log n}{d}},
        \qquad |g_1|\le C_D\sqrt{\log n},
\]
then
\begin{equation}
        \left|T_{r_\mu^{\rm lab}}\xi(y)\right|
        \le C_Dp\mu\sqrt{L\log n}.
\label{eq:gmbm-pointwise-label-action}
\end{equation}
\end{lemma}

\paragraph{Consequences of the hypotheses.}
We first record three bounds that will be used throughout the appendix.
Dividing
$\mu^2d\ge C_D\log n$ by $\mu^2\sqrt{dL}\le c_D$ gives
$\sqrt{d/L}\ge(C_D/c_D)\log n$.  Hence
$d\ge C_D L(\log n)^2$ and, because $L\le C\log n$ under
$np\ge C_D\log n$, also $d\ge C_D L^2\log n$.  Finally,
$\mu^2\le c_D/\sqrt{dL}$ gives, after adjusting the constants,
\begin{equation}
 d\ge C_D L(\log n)^2,\qquad
 d\ge C_D L^2\log n,\qquad
 \mu^2L\log n\le c_D.
\label{eq:gmbm-exact-derived-conditions}
\end{equation}

\begin{lemma}[Tail calibration]\label{lem:gmbm-tail-calibration}
Put $m_\mu=\mu^2\sqrt d$, let
$H,H'\stackrel{\mathrm{i.i.d.}}{\sim}N(0,I_d)$, and define
\[
        T_\mu
        =
        \frac{\langle H,H'\rangle}{\sqrt d}
        +\mu(H_1+H'_1).
\]
Under the hypotheses of \cref{thm:gmbm-exact}, the conditional edge
probabilities satisfy
\begin{equation}
        p_+(\mu)=\Pp\{T_\mu\ge\tau-m_\mu\},
        \qquad
        p_-(\mu)=\Pp\{T_\mu\le-\tau-m_\mu\}.
\label{eq:gmbm-conditional-tail-laws}
\end{equation}
Let $f_\mu$ be the density of $T_\mu$.  Then
\[
        c\sqrt L\le \tau\le C\sqrt L.
\]
In addition,
\begin{equation}
        \overline\Phi(\tau)\asymp p,
        \qquad
        \varphi(\tau)\asymp p\sqrt L.
\label{eq:gmbm-threshold-gaussian-tail}
\end{equation}
Moreover, for every
$t\in[\tau-m_\mu,\tau+m_\mu]\cup[-\tau-m_\mu,-\tau+m_\mu]$,
\begin{equation}
        c\,p\sqrt L\le f_\mu(t)\le C\,p\sqrt L.
\label{eq:gmbm-density-window}
\end{equation}
Finally, uniformly for $t\in[\tau-m_\mu,\tau+m_\mu]$,
\begin{equation}
        \left|
        \Pp\{T_\mu\le -t\}-\Pp\{T_\mu\ge t\}
        \right|
        \le
        C p\mu\sqrt L.
\label{eq:gmbm-tail-asymmetry}
\end{equation}
\end{lemma}
\begin{proof}
Introduce the sign-adjusted Gaussian vectors
\[
        H=\xi G,\qquad H'=\xi'G'.
\]
They are independent standard Gaussian vectors and are independent of
$\xi,\xi'$.  If $\zeta=\xi\xi'$, then the threshold statistic has the exact
representation
\[
 \sqrt d\,\langle Z(Y),Z(Y')\rangle
 =
 \zeta\left\{
 m_\mu+\frac{\langle H,H'\rangle}{\sqrt d}
 +\mu(H_1+H'_1)
 \right\}
 =\zeta(m_\mu+T_\mu).
\]
Conditioning on $\zeta=\pm1$ proves
\eqref{eq:gmbm-conditional-tail-laws}.  Since the two signs are equally
likely,
\begin{equation}
        p=\frac12\Pp\{T_\mu\ge\tau-m_\mu\}
        +\frac12\Pp\{T_\mu\le-\tau-m_\mu\}.
\label{eq:gmbm-threshold-mixture}
\end{equation}

\paragraph{A conditional Gaussian representation.}
Condition on $H$ and set
\[
        a_H=\mu H_1,\qquad
        s_H=\left\|\frac H{\sqrt d}+\mu e_1\right\|_2.
\]
Then
\[
        T_\mu\mid H\sim N(a_H,s_H^2).
\]
Consequently,
\begin{equation}
\begin{aligned}
 f_\mu(t)
 &=\E\left[
       \frac1{s_H}\varphi\left(\frac{t-a_H}{s_H}\right)
      \right],\\
 \Pp\{T_\mu\ge t\}
 &=\E\overline\Phi\left(\frac{t-a_H}{s_H}\right),\qquad
 \Pp\{T_\mu\le-t\}
 =\E\overline\Phi\left(\frac{t+a_H}{s_H}\right).
\end{aligned}
\label{eq:gmbm-conditional-gaussian-formulas}
\end{equation}

Let $\mathcal E$ be the event
\[
 |H_1|\le\frac{c_0}{\mu\sqrt L},
 \qquad
 \left|\frac{\|H\|^2}{d}-1\right|\le\frac{c_0}{L}.
\]
The lower signal assumption ensures $\mu>0$.  Gaussian and chi-square
concentration \cite[Theorem~3.1.1]{VershyninHDP} give
\begin{equation}
 \Pp(\mathcal E^c)
 \le
 2\exp\left(-\frac c{\mu^2L}\right)
 +2\exp\left(-\frac{cd}{L^2}\right).
\label{eq:gmbm-conditional-bad-event}
\end{equation}
On $\mathcal E$,
\[
 |a_H|\le\frac{c_0}{\sqrt L},
 \qquad
 |s_H-1|
 \le C\left(
        \frac1L+\frac{\mu|H_1|}{\sqrt d}+\mu^2
       \right)
 \le\frac cL,
\]
where the last inequality follows from
\eqref{eq:gmbm-exact-derived-conditions}.  Thus, uniformly for
$c\sqrt L\le |t|\le C\sqrt L$ and $|v|\le c/\sqrt L$,
Gaussian Mills bounds \cite{GordonMillsRatio} imply on $\mathcal E$ that
\begin{equation}
 \frac1{s_H}\varphi\left(\frac{t+v-a_H}{s_H}\right)
 \asymp\varphi(t),
 \qquad
 \overline\Phi\left(\frac{t+v\pm a_H}{s_H}\right)
 \asymp\overline\Phi(t)
 \quad (t>0).
\label{eq:gmbm-local-conditional-comparison}
\end{equation}
Indeed, the logarithms of the corresponding density ratios are bounded by
$C(t^2|s_H-1|+|t|(|v|+|a_H|))=O(1)$.

For the density upper bound on $\mathcal E^c$, note that
$s_H=\|H+\mu\sqrt d\,e_1\|/\sqrt d$.  The Gaussian Laplace transform gives
\[
 \E e^{-u\|H+\mu\sqrt d\,e_1\|^2}
 =(1+2u)^{-d/2}
   \exp\left(-\frac{u\mu^2d}{1+2u}\right)
 \le (1+2u)^{-d/2}.
\]
Using $x^{-1}=\int_0^\infty e^{-ux}\,du$, we obtain
\[
        \E s_H^{-2}
        \le d\int_0^\infty(1+2u)^{-d/2}\,du
        =\frac d{d-2}\le C.
\]
Hence Cauchy--Schwarz bounds the exceptional contribution to the first line
of \eqref{eq:gmbm-conditional-gaussian-formulas} by
$C\Pp(\mathcal E^c)^{1/2}$.  To verify that this contribution is negligible
at the required scale, put
\[
        A=(\mu^2L)^{-1},\qquad B=d/L^2.
\]
The derived conditions imply
$A\ge C_D\log n\ge C_DL$ and $B\ge C_D\log n\ge C_DL$.
Moreover, $\mu^2d\ge C_D\log n$ gives
$A\le BL/(C_D\log n)$.  Since exponential decay dominates the resulting
polynomial factors, increasing $C_D$ in the hypotheses gives
\begin{equation}
        \Pp(\mathcal E^c)^{1/2}
        \le C\left(e^{-cA}+e^{-cB}\right)
        \le C_De^{-L}A^{-1/2}
        \le C_Dp\mu\sqrt L.
\label{eq:gmbm-conditional-bad-absorption}
\end{equation}
Here we used $e^{-L}=p/e$.  This proves the exceptional-set estimate without
losing any logarithmic factor.

We now bracket the solution of
\eqref{eq:gmbm-threshold-mixture}.  At a sufficiently small constant
multiple of $\sqrt L$, its right-hand side is larger than $p$, whereas at a
sufficiently large constant multiple it is smaller than $p$; this follows
from \eqref{eq:gmbm-local-conditional-comparison},
\eqref{eq:gmbm-conditional-bad-absorption}, and
$m_\mu\sqrt L\le c_D$.  Monotonicity in $\tau$ therefore gives
$\tau\asymp\sqrt L$.  Applying
\eqref{eq:gmbm-local-conditional-comparison} at this value of $\tau$ yields
\[
        \overline\Phi(\tau)\asymp p,
        \qquad
        \varphi(\tau)\asymp p\sqrt L,
\]
where the second relation is Mills' ratio.  The density formula in
\eqref{eq:gmbm-conditional-gaussian-formulas}, with
\eqref{eq:gmbm-conditional-bad-absorption}, now proves
\eqref{eq:gmbm-density-window} for both the positive and negative windows.

\paragraph{Comparison of the two tails.}
For $t\in[\tau-m_\mu,\tau+m_\mu]$, subtract the two tail formulas in
\eqref{eq:gmbm-conditional-gaussian-formulas}.  On $\mathcal E$, the mean
value theorem and \eqref{eq:gmbm-local-conditional-comparison} give
\[
 \left|
 \overline\Phi\left(\frac{t+a_H}{s_H}\right)
 -
 \overline\Phi\left(\frac{t-a_H}{s_H}\right)
 \right|
 \le Cp\sqrt L\,|a_H|.
\]
Taking expectations and using $\E|a_H|\le C\mu$, while controlling
$\mathcal E^c$ by \eqref{eq:gmbm-conditional-bad-absorption}, proves
\eqref{eq:gmbm-tail-asymmetry}.
\end{proof}

\begin{lemma}[Direct residual-operator bound]
\label{lem:gmbm-population-residual}
Under the hypotheses of \cref{thm:gmbm-exact},
\begin{equation}
        \Lambda_\mu^{\rm lab}
        \le
        Cp\left(\mu\sqrt L+\sqrt{\frac Ld}\right).
\label{eq:gmbm-direct-residual-operator}
\end{equation}
\end{lemma}

\begin{proof}
We bound the population operator directly, without interpolating in $\mu$.
Let $\mathsf P$ denote the orthogonal projection in
$L^2(\Pp_Y\otimes\Pp_Y)$ that removes the mean, the two one-variable
Hoeffding components, and the label kernel
$(Y,Y')\mapsto\xi\xi'$.  By definition, the integral kernel defining
$T_{r_\mu^{\rm lab}}$ is $\mathsf P K_\mu$, where $K_\mu$ is the edge
indicator.

\paragraph{Angular decomposition.}
Write
\[
        g=(x,r\omega),\qquad
        x=g_1,\quad r=\|g_{2:d}\|,\quad
        \omega\in\mathbb S^{d-2},
\]
and define $x',r',\omega'$ similarly.  Let $\eta$ be the joint law of
$(\xi,x,r)$.  For
\[
 A=\xi x'+\xi'x,\qquad
 \zeta=\xi\xi',\qquad
 W=\frac{xx'}{\sqrt d},
\]
the edge condition is
\[
 \frac{rr'}{\sqrt d}\langle\omega,\omega'\rangle
 +W+\mu A+m_\mu\zeta\ge\tau.
\]
For $\ell\ge0$, put
\[
 F_\ell(z)
 =
 \int_{z/\sqrt d}^1
 P_\ell^{(d-1)}(t)w_{d-1}(t)\,dt,
\]
with its natural constant extension outside $[-1,1]$: it is zero when
$z/\sqrt d\ge1$, while for $z/\sqrt d\le-1$ it equals the integral over
$[-1,1]$ (namely, one for $\ell=0$ and zero for $\ell\ge1$).  The
Funk--Hecke formula
\cite[Theorem~1.2.9]{DaiXu} shows that the degree-$\ell$ angular block has
radial-label kernel $F_\ell(z(Y,Y'))$, where
\begin{equation}
 z(Y,Y')
 =
 \frac{\tau-\mu A-m_\mu\zeta-W}
 {(r/\sqrt d)(r'/\sqrt d)}.
\label{eq:gmbm-direct-normalized-boundary}
\end{equation}
The mean, label, and one-variable terms all have angular degree zero.
Consequently, $\mathsf P$ changes only the block $\ell=0$.

Let $\mathcal G$ be the one-point event
\[
 |x|\le c_\star\left(\frac dL\right)^{1/4},
 \qquad
 \left|\frac r{\sqrt d}-1\right|\le\frac{c_\star}{L},
\]
where $c_\star>0$ is a sufficiently small absolute constant.
Gaussian and chi-square concentration give
\begin{equation}
 \varepsilon:=\Pp_\eta(\mathcal G^c)
 \le
 C\exp\left(-c\sqrt{\frac dL}\right)
 +C\exp\left(-\frac{cd}{L^2}\right).
\label{eq:gmbm-direct-typical-probability}
\end{equation}
The assumptions imply
\[
 d\ge C_DL^2\log n,\qquad
 m_\mu\sqrt L\le c_D,\qquad
 \mu L\le C.
\]
Hence, on $\mathcal G^2$,
\[
        z(Y,Y')=\tau+O(L^{-1/2}).
\]
Put $Q=\sqrt{d/L}$.  By
\eqref{eq:gmbm-exact-derived-conditions},
$Q\ge C_D\log n\ge C_DL$ and $d/L^2\ge C_D\log n\ge C_DL$.
Consequently, exponential decay dominates both $e^{-L}$ and the factor
$Q^{-1}$, and \eqref{eq:gmbm-direct-typical-probability} gives
\begin{equation}
        \sqrt\varepsilon
        \le
        Cp\sqrt{\frac Ld}
        \le
        Cp\left(\mu\sqrt L+\sqrt{\frac Ld}\right).
\label{eq:gmbm-direct-bad-absorption}
\end{equation}
Since every block kernel is bounded by one in absolute value, its restriction
to the complement of $\mathcal G^2$ has Hilbert--Schmidt norm at most
$C\sqrt\varepsilon$.  It remains to work on $\mathcal G^2$.

\paragraph{Positive angular degrees.}
For $z=\tau+O(L^{-1/2})$, the cap with threshold $z/\sqrt d$ has probability
comparable to $p$.  Indeed, the spherical Mills estimate
\cref{lem:spherical-mills}, applied in dimension $d-1$, gives
\[
 \Pp\left\{\langle\omega,\omega'\rangle\ge\frac z{\sqrt d}\right\}
 \asymp
 \frac1z\exp\left(-\frac{z^2}{2}+O\left(\frac{z^4}{d}\right)\right)
 \asymp \overline\Phi(\tau)\asymp p,
\]
where we used $z-\tau=O(L^{-1/2})$ and $L^2/d$ sufficiently small.
Applying the sparse cap estimate
\cref{prop:cap-operator} in dimension $d-1$ therefore gives
\[
        \sup_{\ell\ge1}|F_\ell(z)|
        \le Cp\sqrt{\frac Ld}.
\]
Schur's test on the probability space $(\eta,\Pp_\eta)$ therefore yields
\begin{equation}
        \sup_{\ell\ge1}
        \|T_\mu^{(\ell)}\|_{\op}
        \le Cp\sqrt{\frac Ld}+C\sqrt\varepsilon,
\label{eq:gmbm-direct-positive-blocks}
\end{equation}
where $T_\mu^{(\ell)}$ is the degree-$\ell$ block.

\paragraph{Angular degree zero.}
Write $F=F_0$ and let $\mathsf P_0$ be the restriction of $\mathsf P$ to the
degree-zero block.  The operator norm of this block is at most the $L^2$ norm
of its kernel, so it suffices to estimate $\|\mathsf P_0F(z)\|_2$.  If
\[
 \rho_d(z)=\frac1{\sqrt d}w_{d-1}\left(\frac z{\sqrt d}\right),
\]
then $F'(z)=-\rho_d(z)$.  The elementary logarithmic derivatives of
$\rho_d$ make the required estimates explicit:
\[
 \rho_d(z)\asymp\varphi(z),\qquad
 \frac{\rho_d'(z)}{\rho_d(z)}
 =-\frac{(d-4)z}{d-z^2},
 \qquad |z-\tau|\le \frac C{\sqrt L}.
\]
Indeed, $|z|\asymp\sqrt L$ and $L^2/d$ is sufficiently small under
\eqref{eq:gmbm-exact-derived-conditions}; differentiating the displayed
logarithmic derivative once more gives
$|\rho_d''(z)|\le C(1+z^2)\rho_d(z)$.  Together with
\cref{eq:gmbm-threshold-gaussian-tail}, these identities give
\begin{equation}
        |F'(z)|\le Cp\sqrt L,\qquad
        |F''(z)|\le CpL,\qquad
        |F'''(z)|\le CpL^{3/2}.
\label{eq:gmbm-direct-F-derivatives}
\end{equation}

Put
\[
        \delta=\frac r{\sqrt d}-1,\qquad
        \delta'=\frac {r'}{\sqrt d}-1,\qquad
        z_\zeta=\tau-m_\mu\zeta,\qquad
        u=z(Y,Y')-z_\zeta.
\]
Because $\zeta\in\{\pm1\}$, the kernel $F(z_\zeta)$ belongs to
$\operatorname{span}\{1,\zeta\}$ and is annihilated by $\mathsf P_0$.
Let $I=\one_{\mathcal G}(Y)\one_{\mathcal G}(Y')$.  We first use the
preceding cancellation globally and only then restrict the difference
$F(z)-F(z_\zeta)$.  Since this difference is bounded by two, its part on
$\{I=0\}$ has $L^2$ norm at most $C\sqrt\varepsilon$.  Taylor's formula at
$z_\zeta$ on $\{I=1\}$ therefore gives the following identity.  Both $z$
and $z_\zeta$ differ from $\tau$ by $O(L^{-1/2})$, so the entire Taylor
segment lies in the window of
\eqref{eq:gmbm-direct-F-derivatives}:
\begin{equation}
 \mathsf P_0F(z)
 =
 \mathsf P_0\left\{I\left(
 F'(z_\zeta)u+\frac12F''(z_\zeta)u^2+\mathcal R
 \right)\right\}+\mathcal R_{\rm bad}.
\label{eq:gmbm-direct-Taylor}
\end{equation}
Here
\[
        \|\mathcal R_{\rm bad}\|_2\le C\sqrt\varepsilon,
        \qquad
        |\mathcal R|\le CpL^{3/2}|u|^3
        \quad\text{on }\{I=1\}.
\]

We now bound the three terms in
\eqref{eq:gmbm-direct-Taylor}.  On $\mathcal G^2$, expansion of the
denominator in \eqref{eq:gmbm-direct-normalized-boundary} gives
\begin{equation}
 u=-\mu A-W-(\tau-m_\mu\zeta)(\delta+\delta')+E,
\label{eq:gmbm-direct-u-expansion}
\end{equation}
where
\[
 |E|
 \le C\left[
   (\mu|A|+|W|)(|\delta|+|\delta'|)
   +\sqrt L(\delta^2+(\delta')^2)
 \right].
\]
For every fixed $2\le k\le6$,
\[
 \|\delta\|_k+\|\delta'\|_k\le\frac{C_k}{\sqrt d},
 \qquad
 \|A\|_k\le C_k,\qquad
 \|W\|_k\le\frac{C_k}{\sqrt d}.
\]
These estimates remain valid after multiplication by $I$.  In particular,
\begin{equation}
        \|Iu\|_k
        \le C_k\left(\mu+\sqrt{\frac Ld}\right),
        \qquad 2\le k\le6.
\label{eq:gmbm-direct-u-moments}
\end{equation}
Here $m_\mu/\sqrt d=\mu^2\le C\mu$ under the derived conditions.

To use the cancellation in the linear term, write
\[
        F'(z_\zeta)=c_0+c_1\zeta.
\]
The derivative bounds in
\eqref{eq:gmbm-direct-F-derivatives} imply
\[
        |c_0|\le Cp\sqrt L,\qquad
        |c_1|\le Cp\,m_\mu L.
\]
Set
\[
 \alpha_0=c_0\tau-c_1m_\mu,
 \qquad
 \alpha_1=c_1\tau-c_0m_\mu,
\]
so that $F'(z_\zeta)z_\zeta=\alpha_0+\alpha_1\zeta$.
Without the factor $I$, the term
$\alpha_0(\delta+\delta')$ is a sum of two one-variable functions and is
annihilated by $\mathsf P_0$; the remaining radial term is
$\alpha_1\zeta(\delta+\delta')$.
This cancellation remains valid after truncation up to a negligible error.
Indeed,
\[
 I\delta
 =\delta\one_{\mathcal G}(Y)
 -\delta\one_{\mathcal G}(Y)\one_{\mathcal G^c}(Y'),
\]
and the first term is a function of $Y$ alone; the analogous identity holds
for $I\delta'$.  Hence the $L^2$ cost of inserting $I$ in each cancelled
radial term is at most $C\sqrt{\varepsilon/d}$.  Moreover,
\[
        |\alpha_0|\le CpL,
        \qquad
        |\alpha_1|\le Cp\,m_\mu L^{3/2}.
\]
Substituting
\eqref{eq:gmbm-direct-u-expansion}, applying Cauchy--Schwarz to all remaining
terms, and then using $m_\mu\sqrt L\le c_D$ and $\mu L\le C$, gives
\begin{align}
 \left\|
 \mathsf P_0\{I F'(z_\zeta)u\}
 \right\|_2
 &\le
 Cp\Bigg[
 \sqrt L\left(\mu+\frac1{\sqrt d}\right)
 +m_\mu L\left(\mu+\sqrt{\frac Ld}\right)
 \notag\\[-2pt]
 &\hspace{2.3cm}
 +\sqrt L\left(\frac\mu{\sqrt d}+\frac{\sqrt L}{d}\right)
 \Bigg] +CpL\sqrt{\frac\varepsilon d} \notag\\
 &\le
 Cp\left(\mu\sqrt L+\sqrt{\frac Ld}\right).
\label{eq:gmbm-direct-linear-bound}
\end{align}
For example, the only radial term not removed by $\mathsf P_0$ has size at
most
\[
        \frac{Cp\,m_\mu L^{3/2}}{\sqrt d}
        =Cp\,\mu^2L^{3/2}
        \le Cp\,\mu\sqrt L.
\]

For the quadratic term and the remainder, projection is an $L^2$
contraction, so \eqref{eq:gmbm-direct-F-derivatives} and
\eqref{eq:gmbm-direct-u-moments} yield
\begin{align}
 \left\|
 \mathsf P_0\{I F''(z_\zeta)u^2\}
 \right\|_2
 &\le
 CpL\left(\mu+\sqrt{\frac Ld}\right)^2,\label{eq:gmbm-direct-quadratic-bound}\\
 \|\mathsf P_0(I\mathcal R)\|_2
 &\le
 CpL^{3/2}\left(\mu+\sqrt{\frac Ld}\right)^3.
\label{eq:gmbm-direct-remainder-bound}
\end{align}
The assumptions $\mu\sqrt L\le c_D$ and $d\ge C_DL^3$ imply that both
right-hand sides are bounded by
\[
        Cp\left(\mu\sqrt L+\sqrt{\frac Ld}\right).
\]
Combining
\cref{eq:gmbm-direct-Taylor,eq:gmbm-direct-linear-bound,eq:gmbm-direct-quadratic-bound,eq:gmbm-direct-remainder-bound}
with \eqref{eq:gmbm-direct-bad-absorption} therefore bounds the degree-zero
block by the same quantity.

Finally, take the supremum over the orthogonal angular blocks and use
\cref{eq:gmbm-direct-bad-absorption,eq:gmbm-direct-positive-blocks}.  This
proves \eqref{eq:gmbm-direct-residual-operator}.
\end{proof}

\begin{proof}[Proof of \cref{lem:gmbm-gap-explicit}]
\smallskip\noindent\emph{Size of the label signal.}
By \cref{eq:gmbm-conditional-tail-laws,eq:gmbm-tail-asymmetry},
\[
\begin{aligned}
        p_+(\mu)-p_-(\mu)
        &=
        \int_{\tau-m_\mu}^{\tau+m_\mu}f_\mu(t)\,dt
        +O\!\left(p\mu\sqrt L\right).
\end{aligned}
\]
Moreover,
\[
 \frac{\mu}{m_\mu}
 =\frac1{\mu\sqrt d}
 \le \frac{C}{\sqrt{\log n}}
\]
by $\mu^2d\ge C_D\log n$.  Thus the error is smaller than the integral
after increasing $C_D$, and \eqref{eq:gmbm-density-window} gives
\[
        p_+(\mu)-p_-(\mu)
        \asymp
        m_\mu\,p\sqrt L
        =
        p\mu^2\sqrt{dL}.
\]
Since $\Delta_\mu=(p_+(\mu)-p_-(\mu))/2$,
this proves the first estimate in
\eqref{eq:gmbm-population-signal-operator}.

\smallskip\noindent\emph{The residual population operator.}
The second estimate in
\eqref{eq:gmbm-population-signal-operator} is precisely
\cref{lem:gmbm-population-residual}.

The remaining pointwise estimate is \cref{lem:gmbm-pointwise-action} below.
\end{proof}

\begin{lemma}[Pointwise label action]\label{lem:gmbm-pointwise-action}
Under the hypotheses of \cref{thm:gmbm-exact}, if $y=(\xi,g)$ satisfies
\[
        \left|\frac{\|g\|}{\sqrt d}-1\right|
        \le C_D\sqrt{\frac{\log n}{d}},
        \qquad |g_1|\le C_D\sqrt{\log n},
\]
then
\[
        \left|T_{r_\mu^{\rm lab}}\xi(y)\right|
        \le C_Dp\mu\sqrt{L\log n}.
\]
\end{lemma}

\begin{proof}
For $y=(\xi,g)$, set
\[
        s_y=\left\|\frac g{\sqrt d}+\mu\xi e_1\right\|_2,
        \qquad
        b_y=\mu g_1+m_\mu\xi,
\]
and
\[
        \Psi(s,b)=\frac12\left[
        \overline\Phi\left(\frac{\tau-b}{s}\right)
        -\overline\Phi\left(\frac{\tau+b}{s}\right)
        \right].
\]
Conditioning on the second label and Gaussian vector gives
\[
        \E_{Y'}\!\left[
        \one_{\{\sqrt d\langle Z(y),Z(Y')\rangle\ge\tau\}}\xi'
        \right]=\Psi(s_y,b_y).
\]
The transformation $(\xi,g)\mapsto(-\xi,-g)$ preserves the law and changes
the sign of $\xi$ while leaving $a_\mu(\xi,g)$ unchanged.  Hence
$\E[a_\mu(Y)\xi]=0$.  Moreover,
$\Delta_\mu=\E[\xi\Psi(s_Y,b_Y)]$, and therefore
\begin{equation}
        T_{r_\mu^{\rm lab}}\xi(y)
        =\Psi(s_y,b_y)-\Delta_\mu\xi.
\label{eq:gmbm-label-action-identity}
\end{equation}

We compare both terms with the deterministic quantity
$\xi\Psi(1,m_\mu)$.  Direct differentiation and
\eqref{eq:gmbm-threshold-gaussian-tail} show that, whenever
$|s-1|\le c/L$ and $|b|\le c/\sqrt L$,
\begin{equation}
        |\partial_b\Psi(s,b)|\le Cp\sqrt L,
        \qquad
        |\partial_s\Psi(s,b)|\le CpL^{3/2}|b|.
\label{eq:gmbm-Psi-derivatives}
\end{equation}
For the second inequality, observe that $\partial_s\Psi(s,0)=0$ and apply
the mean value theorem in $b$; the mixed derivative is a linear combination
of $\varphi^{(j)}((\tau\pm b)/s)$, $0\le j\le2$, and is bounded by
$CpL^{3/2}$ in the stated window.

The assumptions on $y$ imply
\begin{equation}
 |s_y-1|
 \le C_D\left(
   \sqrt{\frac{\log n}{d}}
   +\frac{\mu\sqrt{\log n}}{\sqrt d}
   +\mu^2
 \right),
 \qquad
 |b_y-m_\mu\xi|\le C_D\mu\sqrt{\log n}.
\label{eq:gmbm-fixed-sb-deviation}
\end{equation}
By \eqref{eq:gmbm-exact-derived-conditions}, the right-hand sides place
$(s_y,b_y)$ in the window of \eqref{eq:gmbm-Psi-derivatives}.  Comparing
first in $b$ and then in $s$ gives
\begin{equation}
 \left|\Psi(s_y,b_y)-\xi\Psi(1,m_\mu)\right|
 \le
 C_Dp\mu\sqrt{L\log n}
 +Cp\,m_\mu L^{3/2}|s_y-1|.
\label{eq:gmbm-fixed-Psi-comparison}
\end{equation}
The last term is also at most
$C_Dp\mu\sqrt{L\log n}$.  Indeed, substituting
\eqref{eq:gmbm-fixed-sb-deviation} and dividing by
$\mu\sqrt{L\log n}$ leaves the three factors
\[
        \mu L,\qquad \mu^2L,\qquad
        \frac{m_\mu\mu L}{\sqrt{\log n}},
\]
which are bounded under
\eqref{eq:gmbm-exact-derived-conditions} and
$m_\mu\sqrt L\le c_D$.

It remains to compare $\Delta_\mu$ with $\Psi(1,m_\mu)$.  For a random
$Y=(\xi,G)$, put $H=\xi G$.  Then
$s_Y=\|H/\sqrt d+\mu e_1\|_2$ and
$b_Y=\xi(m_\mu+\mu H_1)$.  Thus, on the event $\mathcal E$ from
\eqref{eq:gmbm-conditional-bad-event}, the same two comparisons give
\[
 \left|\Psi(s_Y,b_Y)-\xi\Psi(1,m_\mu)\right|
 \le
 Cp\mu\sqrt L\,|H_1|
 +Cp\,m_\mu L^{3/2}|s_Y-1|.
\]
Furthermore,
\[
        \E|s_Y-1|
        \le
        C\left(\frac1{\sqrt d}+\frac\mu{\sqrt d}+\mu^2\right).
\]
On the complementary event we use only $|\Psi|\le1$ and
\eqref{eq:gmbm-conditional-bad-absorption}.  Taking expectations therefore
gives
\begin{equation}
        \left|\Delta_\mu-\Psi(1,m_\mu)\right|
        \le Cp\mu\sqrt L.
\label{eq:gmbm-population-Psi-comparison}
\end{equation}
Combining
\cref{eq:gmbm-label-action-identity,eq:gmbm-fixed-Psi-comparison,eq:gmbm-population-Psi-comparison}
proves \eqref{eq:gmbm-pointwise-label-action}.
\end{proof}

\section{Proof of the Hilbert covariance estimate}
\label{sec:hilbert-covariance-proof}

\begin{proof}[Proof of \cref{lem:hilbert-cov}]
The proof combines the matrix Chernoff method of Tropp~\citeyearpar{Tropp} with the
dimension-free trace argument of \citet{HsuKakadeZhang}.
We first project onto an arbitrary finite-dimensional subspace.  Let
$\Pi$ be an orthogonal projection of finite rank and put
\[
        Y_i=\Pi(\phi_i\otimes\phi_i)\Pi.
\]
Then $0\preceq Y_i\preceq p\Pi$ and
\[
        \E Y_i=\Pi\cC\Pi=:\cC_\Pi.
\]
For $\theta>0$ and any $0\le y\le p$, convexity of the exponential on
$[0,p]$ gives
\[
        e^{\theta y}
        \le
        1+\frac{e^{\theta p}-1}{p}y.
\]
By the transfer rule for matrix functions \cite[Section~2.3]{Tropp},
\[
        e^{\theta Y_i}
        \preceq
        \Pi+\frac{e^{\theta p}-1}{p}Y_i
\]
on $\range(\Pi)$, and hence
\[
        \E e^{\theta Y_i}
        \preceq
        \Pi+\frac{e^{\theta p}-1}{p}\cC_\Pi
        \preceq
        \exp\left(\frac{e^{\theta p}-1}{p}\cC_\Pi\right).
\]
Let $S_\Pi=\sum_iY_i$.  The matrix Laplace-transform method gives
\[
        \E\tr e^{\theta S_\Pi}
        \le
        \tr\exp\left(
        \sum_{i=1}^n\log \E e^{\theta Y_i}
        \right)
        \le
        \tr\exp\left(
        n\frac{e^{\theta p}-1}{p}\cC_\Pi
        \right).
\]
The last step uses operator monotonicity of the logarithm and trace
monotonicity of the matrix exponential \cite[Sections~2.4--2.5]{Tropp}.
Because $S_\Pi\ge0$, the event
$\lambda_{\max}(S_\Pi)\ge t$ implies
$\tr(e^{\theta S_\Pi}-\Pi)\ge e^{\theta t}-1$.  Hence
\[
        \Pp\{\lambda_{\max}(S_\Pi)\ge t\}
        \le
        \inf_{\theta>0}
        \frac{
        \tr\!\left[
        \exp\!\left(
        n\frac{e^{\theta p}-1}{p}\cC_\Pi
        \right)-\Pi
        \right]}
        {e^{\theta t}-1}.
\]
The remaining trace is controlled by the effective mass
$\tr(\cC_\Pi)\le p$.  Set
$a_\theta=n(e^{\theta p}-1)/p$.  Since
$\cC_\Pi\preceq\kappa\Pi$,
\[
\begin{aligned}
        \tr(e^{a_\theta\cC_\Pi}-\Pi)
        &=
        \sum_s(e^{a_\theta\lambda_s(\cC_\Pi)}-1)  \\
        &\le
        a_\theta e^{a_\theta\kappa}\tr(\cC_\Pi)
        \le
        a_\theta p\,e^{a_\theta\kappa}.
\end{aligned}
\]

Choose $\theta=1/p$.  Then $a_\theta=(e-1)n/p$, and for all $t\ge p$,
\begin{equation}
        \Pp\{\lambda_{\max}(S_\Pi)\ge t\}
        \le
        \exp\left(
        C+\log n+C\frac{n\kappa}{p}-\frac{t}{Cp}
        \right).
\label{eq:hilbert-finite-rank-tail}
\end{equation}
Taking
\[
        t=C_D(n\kappa+p\log n)
\]
with $C_D$ large enough gives probability at most $n^{-D}$, uniformly in
$\Pi$.

To remove the projection, choose a countable orthonormal basis
$(e_m)_{m\ge1}$ of the separable space $\cH$, and let $\Pi_m$ be the
projection onto $\operatorname{span}\{e_1,\ldots,e_m\}$.  For
$S=\sum_{i=1}^n\phi_i\otimes\phi_i\succeq0$, the variational formula gives
\[
        \|\Pi_mS\Pi_m\|_{\op}\uparrow\|S\|_{\op}.
\]
Thus the events
$E_m=\{\|\Pi_mS\Pi_m\|_{\op}>t\}$ increase to
$\{\|S\|_{\op}>t\}$.  Since
\eqref{eq:hilbert-finite-rank-tail} holds uniformly in $m$,
continuity from below gives
\[
        \Pp\{\|S\|_{\op}>t\}
        =
        \lim_{m\to\infty}\Pp(E_m)
        \le n^{-D}.
\]
This completes the proof.
\end{proof}

\bibliography{ref}

\end{document}